\documentclass[11pt,twoside]{article}

\usepackage{xcolor}

\usepackage{fullpage}

\usepackage{epsf}
\usepackage{fancyhdr}
\usepackage{graphics}
\usepackage{graphicx} 
\usepackage{float} 
\usepackage{subfigure} 
\usepackage{psfrag}
\usepackage{comment}

\usepackage[linesnumbered,ruled]{algorithm2e}
\DontPrintSemicolon	

\usepackage{color}
\usepackage{amsthm}
\usepackage{amsfonts}
\usepackage{amsmath}
\usepackage{bm}
\usepackage{amssymb,bbm}
\usepackage[numbers]{natbib}
\usepackage{algorithmic}
\usepackage[usestackEOL]{stackengine}

\usepackage{url}
\usepackage[colorlinks=True,linkcolor=magenta,citecolor=blue,urlcolor=blue,pagebackref=true,backref=true]
{hyperref}
\renewcommand*{\backref}[1]{\ifx#1\relax \else Page #1 \fi}
\renewcommand*{\backrefalt}[4]{%
    \ifcase #1 \footnotesize{(Not cited.)}%
    \or        \footnotesize{(Cited on page~#2.)}%
    \else      \footnotesize{(Cited on pages~#2.)}%
    \fi}

\usepackage{nicefrac}

\usepackage{chngpage}

\usepackage{tabularx}%

\usepackage{enumitem}
\usepackage{booktabs}
\usepackage{pbox}

\usepackage{caption}

\usepackage{mathtools}

\usepackage{fullpage}
\allowdisplaybreaks

\newcommand{\brackets}[1]{\left[ #1 \right]}
\newcommand{\parenth}[1]{\left( #1 \right)}

\newcommand{\sech}{\text{sech}}

\newcommand{\Exs}{\ensuremath{{\mathbb{E}}}}

\newtheoremstyle{named}{}{}{\itshape}{}{\bfseries}{.}{.5em}{\thmnote{#3's }#1}
\theoremstyle{named}

\theoremstyle{plain}

\newtheorem{theorem}{Theorem}
\newtheorem{proposition}{Proposition}
\newtheorem{lemma}{Lemma}

\newtheorem{corollary}{Corollary}

\newlength{\widebarargwidth}
\newlength{\widebarargheight}
\newlength{\widebarargdepth}

\makeatletter
\long\def\@makecaption#1#2{
        \vskip 0.8ex
        \setbox\@tempboxa\hbox{\small {\bf #1:} #2}
        \parindent 1.5em  
        \dimen0=\hsize
        \advance\dimen0 by -3em
        \ifdim \wd\@tempboxa >\dimen0
                \hbox to \hsize{
                        \parindent 0em
                        \hfil
                        \parbox{\dimen0}{\def\baselinestretch{0.96}\small
                                {\bf #1.} #2
                                }
                        \hfil}
        \else \hbox to \hsize{\hfil \box\@tempboxa \hfil}
        \fi
        }
\makeatother

\long\def\comment#1{}
\definecolor{battleshipgrey}{rgb}{0.52, 0.52, 0.51}
\definecolor{darkgray}{rgb}{0.66, 0.66, 0.66}
\definecolor{darkgreen}{rgb}{0.0, 0.2, 0.13}
\definecolor{darkspringgreen}{rgb}{0.09, 0.45, 0.27}
\definecolor{dukeblue}{rgb}{0.0, 0.0, 0.61}
\definecolor{olivedrab7}{rgb}{0.24, 0.2, 0.12}
\definecolor{darkblue}{rgb}{0.0, 0.0, 0.55}
\definecolor{darkscarlet}{rgb}{0.34, 0.01, 0.1}
\definecolor{candyapplered}{rgb}{1.0, 0.03, 0.0}
\definecolor{ao(english)}{rgb}{0.0, 0.5, 0.0}
\definecolor{applegreen}{rgb}{0.55, 0.71, 0.0}

\SetKwInput{KwInput}{Input}                
\SetKwInput{KwOutput}{Output}              

\begin{document}
\begin{center}

{\bf{\LARGE{Towards Statistical and Computational Complexities of Polyak Step Size Gradient Descent}}}
  
\vspace*{.2in}
{\large{
\begin{tabular}{ccccc}
Tongzheng Ren$^{\star, \diamond, \ddag}$ & Fuheng Cui$^{\star,\flat}$ & Alexia Atsidakou$^{\star,\dagger}$ & Sujay Sanghavi$^{\dagger}$ & Nhat Ho$^{\flat, \ddag}$ \\
\end{tabular}
}}

\vspace*{.1in}

\begin{tabular}{c}
Department of Computer Science, University of Texas at Austin$^\diamond$, \\
Department of Statistics and Data Sciences, University of Texas at Austin$^\flat$ \\
Department of Electrical and Computer Engineering, University of Texas at Austin$^\dagger$, \\
\end{tabular}

\today

\vspace*{.2in}

\begin{abstract}
We study the statistical and computational complexities of the Polyak step size gradient descent algorithm under generalized smoothness and Łojasiewicz conditions of the population loss function, namely, the limit of the empirical loss function when the sample size goes to infinity, and the stability between the gradients of the empirical and population loss functions, namely, the polynomial growth on the concentration bound between the gradients of sample and population loss functions. We demonstrate that the Polyak step size gradient descent iterates reach a final statistical radius of convergence around the true parameter after logarithmic number of iterations in terms of the sample size. It is computationally cheaper than the polynomial number of iterations on the sample size of the fixed-step size gradient descent algorithm to reach the same final statistical radius when the population loss function is not locally strongly convex. Finally, we illustrate our general theory under three statistical examples: generalized linear model, mixture model, and mixed linear regression model.
\end{abstract}
\end{center}
\let\thefootnote\relax\footnotetext{$\star$ Alexia Atsidakou, Fuheng Cui and Tongzheng Ren contributed equally to this work. }
\let\thefootnote\relax\footnotetext{$\ddag$ Correspondence to: Tongzheng Ren (\href{mailto:tongzheng@utexas.edu}{tongzheng@utexas.edu}) and Nhat Ho (\href{mailto:minhnhat@utexas.edu}{minhnhat@utexas.edu}).}
\section{Introduction}
From its origin in mathematics, gradient descent algorithm~\cite{Polyak_Introduction, bubeck2015convex, Nesterov_Introduction} has played a central role in large-scale machine learning and data science applications. In general unconstrained settings, this algorithm can be used for finding optimal solutions of optimization problems of the following form:
\begin{align}
    \min_{\theta \in \mathbb{R}^{d}} f_{n} (\theta). \label{eq:unconstrained_opt}
\end{align}
Here, $n$ stands for the sample size of i.i.d. data $X_{1}, X_{2}, \ldots, X_{n}$ coming from an unknown distribution $P_{\theta^{*}}$ where $\theta^{*}$ is true but unknown parameter and $f_{n}$ is a given empirical loss function whose optimal solutions, denoted by $\widehat{\theta}_{n}$, can be used to approximate the true parameter $\theta^{*}$. While the difference between $\widehat{\theta}_{n}$ and $\theta^{*}$ had been studied extensively in the literature via several tools from the empirical process theory, the convergence rates of $\theta_{n}^{t}$, updates from the gradient descent algorithm, to optimal neighborhood around the true parameter $\theta^{*}$, has still remained a nascent topic.

A natural approach to analyze the difference between the updates $\theta_{n}^{t}$ and the true parameter $\theta^{*}$ is to study the convergence rate of $\theta_{n}^{t}$ to $\widehat{\theta}_n$, stationary points of optimization problem~\eqref{eq:unconstrained_opt}, and the gap between $\widehat{\theta}_n$ and $\theta^{*}$, namely, we use the following triangle inequality:
\begin{align}
    \|\theta_{n}^{t} - \theta^{*}\| \leq \|\theta_{n}^{t} - \widehat{\theta}_n\| + \|\widehat{\theta}_n - \theta^{*}\|. \label{eq:direct_approach}
\end{align}
This approach is often referred to as \emph{direct approach} and has been used in several earlier works (e.g.,~\cite{agarwal2012fast, yuan2013truncated, loh2015regularized,chen2018gradient}). However, to ensure that the radius of convergence for $\|\theta_{n}^{t} - \theta^{*}\|$ is at the order of final statistical rate, we need to obtain a tight optimization convergence rate of the term $\|\theta_{n}^{t} - \widehat{\theta}_n\|$ based on the sample size $n$ and the number of iterations $t$. It requires a precise understanding of the noise-structure in the gradient of the empirical loss function, which is generally non-trivial to study in practice. 

To circumvent the challenges of the direct analysis~\eqref{eq:direct_approach}, a popular approach to analyze the difference between the updates $\theta_{n}^{t}$ and the true parameter $\theta^{*}$ is the \emph{population to sample analysis}~\citep{yi2015regularized, hardt16, Siva_2017, kuzborskij2018data, Fanny-2017, charles2018stability, Raaz_Ho_Koulik_2020, Raaz_Ho_Koulik_2018_second, Ho_Instability, Kwon_minimax}. In particular, we define the corresponding population version of optimization problem~\eqref{eq:pop_unconstrained_opt} as follows:
\begin{align}
    \min_{\theta \in \mathbb{R}^{d}} f (\theta), \label{eq:pop_unconstrained_opt}
\end{align}
where $f(\cdot) : = \Exs_{X^{n}} \brackets{f_{n}(\cdot)}$ is the population loss function and $X^{n} = (X_{1}, \ldots, X_{n})$. When the step size $\eta$ of the gradient descent algorithm is fixed, which we refer to as \emph{fixed-step size gradient descent algorithm}, the idea of the population to sample analysis is to analyze the radius of convergence of $\theta_{n}^{t}$ via the following triangle inequality:
\begin{align}
    \|\theta_{n}^{t+1} - \theta^{*}\| \leq \|F_{\text{GD}}(\theta_{n}^{t}) - \theta^{*}\| + \eta \|\nabla f_{n}(\theta_{n}^{t}) - \nabla f(\theta_{n}^{t})\| : = A + B, \label{eq:pop_to_sam_analysis}
\end{align}
where $F_{\text{GD}}(\theta) : = \theta - \eta \nabla f(\theta)$ is the corresponding population operator of the fixed-step size gradient descent algorithm. The bound~\eqref{eq:pop_to_sam_analysis} suggests that we can relate the behaviors of the sample fixed-step size gradient descent iterate $\theta_{n}^{t+1}$ to two terms: (i) Term A: the convergence rate of gradient descent iterates for solving population loss function~\eqref{eq:pop_unconstrained_opt}; (ii) Term B: the uniform concentration of $\nabla f_{n}(\theta)$ around $\nabla f(\theta)$ when $\theta$ lies in a certain neighborhood around $\theta^{*}$. 

\vspace{0.5 em}
\noindent
\textbf{Complexity of fixed-step size gradient descent:} When the population loss function is locally strongly convex and smooth around $\theta^{*}$, under the local initialization the convergence rate of gradient descent iterates for solving the population loss function is linear, i.e., the term A in equation~\eqref{eq:pop_to_sam_analysis} behaves like $\kappa \|\theta_{n}^{t} - \theta^{*}\|$ where $\kappa < 1$ is some constant. When the deviation bound between $\nabla f_{n}(\theta)$ and $\nabla f(\theta)$ is at the order $\varepsilon(n, \delta)$ with probability $1 - \delta$ as long as $\|\theta - \theta^{*}\| \leq r$ where $\varepsilon(n, \delta)$ is the noise function, the statistical radius of the sample fixed-step size gradient descent updates is at the order of $\mathcal{O}(\varepsilon(n, \delta))$ as long as the number of iterations is at least $\mathcal{O}(\log(1/ \varepsilon(n, \delta)))$. For practical high dimensional statistical models, the noise function $\varepsilon(n, \delta)$ is at the order of $\sqrt{d/n}$ (here we skip $\delta$ for simplicity); therefore, we have parametric statistical radius of the sample gradient descent iterates after $\log(n/d)$ number of iterations.

When the population loss function is no longer locally strongly convex around the true parameter $\theta^{*}$, analyzing the convergence rate of $\theta_{n}^{t}$ is non-trivial as simply applying triangle inequality in equation~\eqref{eq:pop_to_sam_analysis} can get to sub-optimal rate. To get a sharp statistical radius of $\theta_{n}^{t}$, Ho et al.~\cite{Ho_Instability} recently utilize a localization argument from the empirical process theory to progressively balance the two terms A and B when the sample fixed-step size gradient descent updates $\theta_{n}^{t}$ move closer to the true parameter $\theta^{*}$. They show that when the convergence rate of the population fixed-step size gradient descent iterates is at the order of $\mathcal{O}(1/t^{1/\alpha})$ for some $\alpha > 0$ and the deviation bound between $\nabla f_{n}(\theta)$ and $\nabla f(\theta)$ is slow and at the order of $\mathcal{O}(r^{\gamma} \varepsilon(n, \delta))$ with probability $1 - \delta$ as long as $\|\theta - \theta^{*}\| \leq r$ where $\gamma \geq 0$, the final statistical radius of the fixed-step size gradient descent iterates $\|\theta_{n}^{t} - \theta^{*}\|$ is upper bounded by $\mathcal{O}(\varepsilon(n, \delta)^{\frac{1}{1 + \alpha - \gamma}})$ as long as $t \geq \mathcal{O}(\varepsilon(n, \delta)^{-\frac{\alpha}{\alpha + 1 - \gamma}})$ and $\alpha \geq \gamma$. In practical high dimensional statistical models, the noise function $\varepsilon(n, \delta)$ is proportional to $\sqrt{d/n}$; therefore, the required number of iterations for the fixed-step size gradient descent updates to reach the final radius is proportional to $(n/d)^{\frac{\alpha}{\alpha + 1 - \gamma}}$. Since each iteration of the gradient descent requires $\mathcal{O}(nd)$ arithmetic operations, the total computational complexity for the fixed-step size gradient descent algorithm to reach the final statistical radius is of the order of $\mathcal{O}(n^{\frac{\alpha }{\alpha + 1 - \gamma} + 1})$ for fixed dimension $d$. It is much more computationally expensive than the optimal computational complexity $\mathcal{O}(n)$ when the sample size is sufficiently large in practice.

\vspace{0.5 em}
\noindent
\textbf{Contribution.} In this paper, we show that by using Polyak step size gradient descent method~\cite{Polyak_Introduction}, an adaptive gradient descent algorithm, we can overcome the high computational complexity of the fixed-step size gradient descent algorithm for reaching the final statistical radius when the population loss function is not locally strongly convex. Our contribution is two-fold and can be summarized as follows:
\begin{enumerate}
    \item \textbf{Complexity of Polyak step size gradient descent algorithm}: We study the computational and statistical complexities of the Polyak step size gradient descent iterates under the generalized smoothness and Łojasiewicz properties of the population loss function, which are characterized by parameter $\alpha \geq 0$. Under these assumptions, we demonstrate that the population Polyak step size gradient descent iterates have a linear convergence rate to the true parameter $\theta^{*}$. When the deviation bound between the gradients of sample and population loss functions is growing at the order of $\mathcal{O}(r^{\gamma} \varepsilon(n, \delta))$ with probability $1 - \delta$, we further prove that the sample Polyak step size gradient descent updates reach the final statistical radius $\mathcal{O}(\varepsilon(n, \delta)^{\frac{1}{1 + \alpha - \gamma}})$ around the true parameter $\theta^{*}$ as long as $t \geq \mathcal{O}(\log(1/\varepsilon(n, \delta)))$. It indicates that the sample Polyak step size gradient descent iterates reach the same final statistical radius as that of the fixed-step size gradient descent iterates and they only require a logarithmic number of iterations, which is much smaller than those from the fixed-step size gradient descent updates. Since each iteration of the Polyak step size gradient descent algorithm only requires $\mathcal{O}(nd)$ arithmetic operations, the total computational complexity for the Polyak step size algorithm to reach the final statistical radius is at the order of $\mathcal{O}(n \log(1/\varepsilon(n, \delta)))$ for fixed dimension $d$, which is much cheaper than $\mathcal{O}(n \cdot \varepsilon(n, \delta)^{-\frac{\alpha}{\alpha + 1 - \gamma}})$ from the fixed-step size gradient descent algorithm. See Table~\ref{tab:summary_rates} for a more detailed comparison between the Polyak step size and fixed-step size methods.
    \item \textbf{Illustrative examples}: We illustrate the general theory under three statistical models: generalized linear model, symmetric two-component mixture model, and mixed linear regression model. For the generalized linear model with link function $g(x) = x^{p}$ where $p \in \mathbb{N}$ and $p \geq 2$, we demonstrate that when we have no signal, i.e., $\theta^{*} = 0$, the Polyak step size gradient descent iterates converge to a radius of convergence $\mathcal{O}((d/n)^{1/2p})$ around the true parameter after $\mathcal{O}(\log(n/d))$ number of iterations. It is much faster than the required number of iterations $\mathcal{O}((n/d)^{\frac{p - 1}{p}})$ of the fixed-step size gradient descent algorithm. For both the symmetric two-component mixture model and mixed linear regression, under the low signal-to-noise regime, e.g., $\theta^{*} = 0$, we prove that the final optimal statistical radius of the Polyak step size iterates are at the order of $\mathcal{O}((d/n)^{1/4})$ as long as we run the algorithm for $\mathcal{O}(\log(n/d))$ iterations, which is faster than $\mathcal{O}(\sqrt{n/d})$ number of iterations required for the EM algorithm, which in these settings is equivalent to gradient descent with step size 1, in order to reach the same final statistical radii.
\end{enumerate}

\vspace{0.5 em}
\noindent
\textbf{Organization.} The paper is organized as follows. In Section 2, we first introduce our assumptions on generalized smoothness and Łojasiewicz property of the population loss function and the growth condition on the concentration of the gradient of sample loss function around the gradient of the population loss function. Then, we establish convergence rates of the Polyak step size gradient descent iterates under these assumptions. In Section~\ref{sec:examples}, we illustrate these convergence rates under specific settings of generalized linear model, mixture model, and mixed linear regression. We carry out experiments in Section~\ref{sec:experiments} to verify the convergence rates studied in Section~\ref{sec:examples} while concluding the paper with a few discussions in Section~\ref{sec:discussion}. Proofs of main results are in Section~\ref{sec:proofs} while proofs of the remaining results are deferred to the Appendices. 

\vspace{0.5 em}
\noindent
\textbf{Notation.} For any matrix $A \in \mathbb{R}^{d \times d}$, we denote by $\lambda_{\max}(A)$ the maximum eigenvalue of the matrix A. For any $x \in \mathbb{R}^{d}$, $\|x\|$ denotes the $\ell_{2}$ norm of $x$. For any two sequences $\{a_{n}\}_{n \geq 1}, \{b_{n}\}_{n \geq 1}$, we denote $a_{n} = \mathcal{O}(b_{n})$ to mean that $a_{n} \leq C b_{n}$ for all $n \geq 1$ where $C$ is some universal constant. Furthermore, we denote $a_{n} = \Theta(b_{n})$ to indicate that $C_{1} b_{n} \leq a_{n} \leq C_{2} b_{n}$ for any $n \geq 1$ where $C_{1}, C_{2}$ are some universal constants. 

\section{Polyak Step Size Gradient Descent}
\label{sec:Polyak_Step_Size}
In this section, we first provide a set of assumptions used in our analysis of the Polyak step size gradient descent algorithm in Section~\ref{sec:assumptions}. We then study the convergence rate of that algorithm under these assumptions in Section~\ref{sec:convergence_rate_Polyak}. 
\subsection{Assumptions}
\label{sec:assumptions}
We first start with the following assumption about the local generalized smoothness of the population loss function in equation~\eqref{eq:pop_unconstrained_opt}.
\begin{enumerate}[label=(W.1)]
\item\label{assump:smoothness} (Generalized Smoothness)
There exists a constant $\alpha \geq 0$ such that for all $\theta\in \mathbb{B}(\theta^*, \rho)$ for some radius $\rho > 0$, we have
\begin{align*}
    \lambda_{\max}(\nabla^2 f(\theta)) \leq & c_1 \|\theta - \theta^*\|^{\alpha},
\end{align*}
where $c_1>0$ is some universal constant.
\end{enumerate}
When $\alpha = 0$, Assumption~\ref{assump:smoothness} corresponds to the standard local smoothness condition. When $\alpha > 0$, Assumption~\ref{assump:smoothness} provides a polynomial growth condition on the Lipschitz constant when the parameter lies in some neighborhood around the true parameter $\theta^{*}$. An example of the function $f$ that satisfies Assumption~\ref{assump:smoothness} is $f(\theta) = \sum_{i = 1}^{d} \theta_{i}^{2\alpha_{i}}$ for all $\theta = (\theta_{1}, \theta_{2}, \ldots, \theta_{d}) \in \mathbb{R}^{d}$ where $\alpha_{1},\alpha_{2},\ldots, \alpha_{d} \geq 1$ are some given positive integers. In this simple example, the true parameter $\theta^{*} = 0$ and the constant $\alpha$ in Assumption~\ref{assump:smoothness} takes the value $\alpha = \min_{1 \leq i \leq d} \{2\alpha_{i} - 2\}$.

\begin{table}[t]

  {

    {\renewcommand{\arraystretch}{2}

    \begin{tabular}{ccccc}
        \toprule
        {\bf Method} & \Centerstack{\bf Smoothness~\ref{assump:smoothness}, \\ \bf Łojasiewicz~\ref{assump:nonPL}} & \Centerstack{\bf Concentration \\ \bf Bound~\ref{assump:stab}}
        & \Centerstack{\bf  Number of \\ \bf Iterations} & 
        \Centerstack {\bf  Statistical \\ \bf Radius}   \\
        \Centerstack{Fixed-step size \\ gradient descent \\ (Proposition~\ref{prop:fixed_gradient_descent})}
        & \Centerstack{$\alpha > 0$ \\ $\alpha = 0$}
        & \Centerstack{$\gamma \geq 0$ \\ $\gamma=0$}
        & \Centerstack{$\varepsilon(n, \delta)^{-\frac{\alpha}{1+\alpha - \gamma}}$ \\ $\log (1/\varepsilon(n, \delta))$} & \Centerstack{$\varepsilon(n, \delta)^{\frac{1}{\alpha + 1 - \gamma}}$ \\ $\varepsilon(n, \delta)$}

        \\[2mm]
       \Centerstack{Polyak step size \\ gradient descent \\ (Theorem~\ref{theorem:convergence_rate_Polyak})}
       &  $\alpha \geq 0$ & $\gamma\geq 0$
       & $\log (1/\varepsilon(n, \delta))$

        & $\varepsilon(n, \delta)^{\frac{1}{\alpha + 1 - \gamma}}$

        \\[2ex] \bottomrule \hline
    \end{tabular}
    }
    }
    \noindent\caption{An overview of the convergence rates of fixed-step size and Polyak step size gradient descent iterates under the assumptions on generalized smoothness of the population loss function (Assumptions~\ref{assump:smoothness}), generalized Łojasiewicz property of the population loss function (Assumption~\ref{assump:nonPL}), and uniform concentration bound between the gradients of the population and sample loss functions (Assumption~\ref{assump:stab}). The results in the table show that when $\alpha > 0$, the Polyak step size gradient descent iterates reach to the same statistical radius $\varepsilon(n, \delta)^{\frac{1}{\alpha + 1 - \gamma}}$ as that of fixed-step size gradient descent iterates after much fewer number of iterations ($\log (1/\varepsilon(n, \delta))$ iterations of Polyak step size method versus $\varepsilon(n, \delta)^{\frac{1}{\alpha + 1 - \gamma}}$ of fixed-step size method). As the complexity per iteration of the Polyak step size method and the fixed-step size method is similar, the Polyak method is more computationally efficient than the fixed-step size method for reaching the same final statistical radius. When $\alpha = 0$ and $\gamma = 0$, e.g., locally strongly convex setting, both the Polyak and fixed-step size methods reach the statistical radius $\varepsilon(n, \delta)$ after a logarithmic number of iterations. }
    \label{tab:summary_rates}
    \vspace{-5mm}
\end{table}

Now, to obtain a convergence rate for the Polyak step size gradient descent algorithm for solving the minima of the population loss function, we need another assumption, which we refer to as \emph{generalized Łojasiewicz property}, on the growth of the gradient of the population loss function $f$.
\begin{enumerate}[label=(W.2)] 
\item \label{assump:nonPL} (Generalized Łojasiewicz Property) For all $\theta \in \mathbb{B}(\theta^{*}, \rho)$ for some radius $\rho > 0$, there exists a constant $\alpha \geq 0$ such that we have
\begin{align*}
    \|\nabla f(\theta)\| \geq c_2(f(\theta) - f(\theta^*))^{1 - \frac{1}{\alpha + 2}}
\end{align*}
where $c_2 > 0$ is some universal constant.
\end{enumerate}
When $\alpha = 0$, the generalized Łojasiewicz property is simply the well-known local Polyak-Łojasiewicz inequality~\cite{bubeck2015convex}. This inequality has been used to guarantee the linear convergence of the fixed-step size gradient descent algorithm. When $\alpha > 0$, the inequality in Assumption~\ref{assump:nonPL} indicates that the gradient locally grows faster than a high order polynomial function as we move around the global minima $\theta^{*}$ where the maximum degree of the polynomial function is determine by the constant $\alpha$. Similar to Assumption~\ref{assump:smoothness}, a simple example of the function $f$ that satisfies Assumption~\ref{assump:nonPL} is $f(\theta) = \sum_{i = 1}^{d} \theta_{i}^{2\alpha_{i}}$ for all $\theta = (\theta_{1}, \theta_{2}, \ldots, \theta_{d}) \in \mathbb{R}^{d}$ where $\alpha_{1},\alpha_{2},\ldots, \alpha_{d} \geq 1$ are some given positive integers. The constant $\alpha$ in Assumption~\ref{assump:nonPL} takes the value $\alpha = \max_{1 \leq i \leq d} \{2\alpha_{i} - 2\}$. If we would like the function $f$ in this example to satisfy both Assumptions~\ref{assump:smoothness} and~\ref{assump:nonPL} with the same constant $\alpha$, we need to have $\alpha_{1} = \alpha_{2} = \ldots = \alpha_{d} = \alpha$, namely, homogeneous polynomial function. This behavior turns out to be popular in several statistical models, such as generalized linear model, mixture model, and mixed linear regression that we study in Section~\ref{sec:examples}. In Appendix~\ref{sec:inhomogenous}, we also briefly discuss the behavior of the Polyak step size gradient descent algorithm when the simple polynomial function $f$ does not have homogeneous order, i.e., the constants in Assumptions~\ref{assump:smoothness} and~\ref{assump:nonPL} are different.

Finally, to analyze the iterates from the Polyak step size gradient descent algorithm for minimizing the sample loss function in equation~\eqref{eq:unconstrained_opt}, we need a growth condition on the uniform deviation bound between the gradients of the sample and population loss functions.
\begin{enumerate}[label=(W.3)] 
\item \label{assump:stab} (Stability Property)
For a given parameter $\gamma \geq 0$, there exist a noise function $\varepsilon: \mathbb{N} \times (0,1] \to \mathbb{R}^{+}$, universal constant $c_3 > 0$, and some positive parameter $\rho > 0$ such that 
\begin{align*}
    \sup_{\theta\in \mathbb{B}(\theta^*, r)} \|\nabla f_n(\theta) - \nabla f(\theta)\|\leq c_3 r^\gamma \varepsilon(n, \delta),
\end{align*}
for all $r \in (0, \rho)$ with probability $1 - \delta$.
\end{enumerate}
A simple interpretation of the Assumption~\ref{assump:stab} is that we would like to control the growth of the noise function, resulting from the difference between the sample and population loss functions, when the radius of the ball around $\theta^{*}$ goes to 0. That assumption also suggests that $\theta^{*}$ is some stationary point of the sample loss function $f_{n}$ when $\gamma > 0$. A simple example for Assumption~\ref{assump:stab} is that $f_{n}(\theta) = \frac{\|\theta\|^{2p}}{2p} - \frac{\omega \|\theta\|^{2q} }{2q} \sqrt{\frac{d}{n}}$ where $\omega \sim \mathcal{N}(0, 1)$ and $p, q$ are positive integers such that $p > q$. Under this simple case, $f(\theta) = \frac{\|\theta\|^{2p}}{2p}$ and the constant $\gamma$ in Assumption~\ref{assump:stab} takes the value $\gamma = 2q - 1$ while the noise function $\varepsilon(n, \delta) = \sqrt{\frac{d \log(1/\delta)}{n}}$. For more practical examples, we refer readers to Section~\ref{sec:examples}.

\subsection{Convergence rate of the Polyak step size gradient descent}
\label{sec:convergence_rate_Polyak}
The Polyak step size gradient descent iterates $\{\theta_{n}^{t}\}_{t \geq 0}$ for solving the sample loss function $f_{n}$ in equation~\eqref{eq:unconstrained_opt} take the following form:
\begin{align}
    F_{n}(\theta_{n}^{t}) : = \theta_{n}^{t + 1} = \theta_{n}^{t} - \frac{f_{n}(\theta_{n}^{t}) - f_{n}(\widehat{\theta}_{n})}{\|\nabla f_{n}(\theta_{n}^{t})\|^2} \cdot \nabla f_{n}(\theta_{n}^{t}), \label{eq:sample_operator_Polyak}
\end{align}
where $\widehat{\theta}_{n}$ is some optimal solution of the optimization problem~\eqref{eq:unconstrained_opt} (See our discussion after Theorem~\ref{theorem:convergence_rate_Polyak} about an adaptive version of Polyak step size gradient descent algorithm to deal with the unknown value of $f_{n}(\widehat{\theta}_{n})$). The operator $F_{n}$ in equation~\eqref{eq:sample_operator_Polyak} is referred to as \emph{sample Polyak operator}. To analyze the convergence rate of the sample iterates $\theta_{n}^{t}$, we will use the population to sample analysis discussed in equation~\eqref{eq:pop_to_sam_analysis}. In particular, we define the following \emph{population Polyak operator} for solving the population loss function $f$ in equation~\eqref{eq:pop_unconstrained_opt}:
\begin{align}
    F(\theta) : = \theta - \frac{f(\theta) - f(\theta^{*})}{\|\nabla f(\theta)\|^2} \cdot \nabla f(\theta), \label{eq:population_operator_Polyak}
\end{align}
As being indicated in the population to sample analysis for analyzing the fixed-step size gradient descent algorithm, to analyze the sample iterates $\{\theta_{n}^{t}\}_{t \geq 0}$ of the Polyak step size gradient descent algorithm we use the following triangle inequality:
\begin{align}
    \|\theta_{n}^{t + 1} - \theta^{*}\| \leq \|F_{n}(\theta_{n}^{t}) - F(\theta_{n}^{t})\| + \|F(\theta_{n}^{t}) - \theta^{*}\|. \label{eq:Polyak_sample_to_population}
\end{align}
Therefore, to obtain an upper bound for the gap between $\theta_{n}^{t+1}$ and $\theta^{*}$, we need to understand the contraction of the population operator $F$ to $\theta^{*}$ as well as the deviation between the sample operator $F_{n}$ and population operator $F$. The following lemma shows the linear contraction of the population operator $F$ towards $\theta^{*}$.   
\begin{lemma}
\label{lincon}
Assume that Assumptions~\ref{assump:smoothness} and~\ref{assump:nonPL} hold. Then, given the definition of Polyak population operator in equation~\eqref{eq:population_operator_Polyak} we have
\begin{align*}
    \|F(\theta) - \theta^*\| \leq \kappa \|\theta - \theta^*\|,
\end{align*}
where $\kappa : = \left(1 - \frac{c_2^{\alpha + 2}}{2c_1 (\alpha + 2)^{\alpha + 2}}\right)^{1/2}$ and $c_{1}, c_{2}$ are universal constants in Assumptions~\ref{assump:smoothness} and~\ref{assump:nonPL}.
\end{lemma}
\noindent
The proof of Lemma~\ref{lincon} is in Section~\ref{subsec:proof:lincon}. The result of Lemma~\ref{lincon} indicates that if $\{\theta^{t}\}_{t \geq 0}$ is a sequence of population Polyak step size gradient descent iterates, i.e., $\theta^{t + 1} = F(\theta^{t})$, then we have
\begin{align*}
    \|\theta^{t} - \theta^{*}\| \leq \kappa^{t} \|\theta^{0} - \theta^{*}\|.
\end{align*}
The linear convergence of population Polyak step size gradient descent iterates is in stark different from the sub-linear convergence $\Theta(t^{-1/\alpha})$ of the fixed-step size gradient descent iterates under Assumptions~\ref{assump:smoothness} and~\ref{assump:nonPL} (See Lemma~\ref{lem:convergence_gd} in Appendix~\ref{sec:auxiliary_result}).

Our next result establishes an uniform concentration bound between the sample Polyak operator $F_{n}$ and the population Polyak operator $F$.
\begin{lemma}
\label{supdif}
Assume that Assumptions~\ref{assump:smoothness}, \ref{assump:nonPL}, and~\ref{assump:stab} hold with $\alpha \geq \gamma$. Assume that $\|\widehat{\theta}_{n} - \theta^{*}\| \leq r_{n}$ where $\widehat{\theta}_{n}$ is the optimal solution of the sample loss function $f_{n}$ and $r_{n} : = \bar{C}\varepsilon(n, \delta)^{\frac{1}{\alpha + 1 - \gamma}}$ where $\bar{C} = \left(\frac{C \cdot c_3 (\alpha + 2)^{\alpha + 1}}{c_2^{\alpha + 2}}\right)^{\frac{1}{1+\alpha - \gamma}}$, $c_{2}$, $c_3$ are the universal constant in Assumption~\ref{assump:nonPL} and~\ref{assump:stab} and $C$ is some universal constant. Then for any $r_n \leq r < \rho$ and for some universal constants $c_4 \geq 1$, we have 
\begin{align*}
    \sup_{\theta \in \mathbb{B}(\theta^{*}, r) \backslash \mathbb{B}(\theta^{*}, r_{n})} \|F_n(\theta)-F(\theta)\| \leq c_4 r^{\gamma - \alpha} \varepsilon(n, \delta).
\end{align*}
\end{lemma}
\noindent
The proof of Lemma~\ref{supdif} is in Section~\ref{subsec:supdif}. A few comments with that lemma are in order. First, the condition $\alpha \geq \gamma$ is to guarantee that the signal is stronger than the noise in statistical model in which we can derive the meaningful statistical rate for our estimator. Second, the assumption that $\|\widehat{\theta}_{n} - \theta^{*}\| \leq r_{n}$ is natural as from Proposition~\ref{prop:fixed_gradient_descent}, we demonstrate that that statistical radius is at the order of $\mathcal{O}(\varepsilon(n, \delta)^{\frac{1}{\alpha + 1 - \gamma}})$. Third, as indicated in Lemma~\ref{supdif}, the uniform concentration bound between the sample Polyak operator $F_{n}$ and the population Polyak operator $F$ only holds when $r_{n} \leq \|\theta - \theta^{*}\| \leq r$. The condition $\|\theta - \theta^{*}\| \geq r_{n}$ is important to ensure that the concentration bound is stable. When $\|\theta - \theta^{*}\| <  r_{n}$, it happens that $\|F_{n}(\theta) - F(\theta)\|$ goes to infinity. This instability behavior of the concentration bound between $F_{n}$ and $F$ when the parameter approaches $\theta^{*}$ is different from the stable concentration bound of the sample fixed-step size gradient descent operator around the population fixed-step size gradient descent operator, which is proportional to $r^{\gamma} \cdot \varepsilon(n, \delta)$ according to Assumption~\ref{assump:stab} and holds for all $\theta \in \mathbb{B}(\theta^{*}, r)$.

Equipped with the linear convergence of the population Polyak operator in Lemma~\ref{lincon} and the uniform deviation bound between the sample Polyak operator $F_{n}$ and the population Polyak operator $F$, we are ready to state our main result about the statistical and computational complexity of the sample Polyak step size gradient descent iterates.
\begin{theorem}
\label{theorem:convergence_rate_Polyak}
Assume that Assumptions~\ref{assump:smoothness},~\ref{assump:nonPL} and~\ref{assump:stab} and assumptions in Lemma~\ref{supdif} hold with $\alpha \geq \gamma$. Assume that the sample size $n$ is large enough such that $\varepsilon(n, \delta)^{\frac{1}{\alpha + 1 - \gamma}} \leq \frac{(1 - \kappa) \rho}{c_{4} \bar{C}^{\gamma - \alpha}}$ where $\kappa$ is defined in Lemma~\ref{lincon}, $c_{4}$ and $\bar{C}$ are the universal constants in Lemma~\ref{supdif}, and $\rho$ is the local radius. Then, there exist universal constants $C_{1}$, $C_{2}$ such that for $t \geq C_{1}\log (1/\varepsilon(n, \delta))$, the following holds:
\begin{align*}
    \min_{k\in \{0, 1, \cdots, t\}}\|\theta_n^k - \theta^*\| \leq C_{2} \cdot \varepsilon(n, \delta)^{\frac{1}{\alpha + 1 - \gamma}},
\end{align*}
\end{theorem}
\noindent
The proof of Theorem~\ref{theorem:convergence_rate_Polyak} is in Section~\ref{subsec:proof:theorem:convergence_rate_Polyak}. Below, we have the following discussions with the result of Theorem~\ref{theorem:convergence_rate_Polyak}:

\vspace{0.5 em}
\noindent
\textbf{Comparing to fixed-step size gradient descent:} Since the convergence rate of the fixed-step size gradient descent iterates is at the order of $t^{-1/\alpha}$ when $\alpha > 0$ under Assumptions~\ref{assump:smoothness} and~\ref{assump:nonPL} (See Lemma~\ref{lem:convergence_gd} in Appendix~\ref{sec:auxiliary_result})  and the concentration bound between the sample gradient descent and population gradient descent operators are of the order $r^{\gamma} \cdot \varepsilon(n, \delta)$ under Assumption~\ref{assump:stab}, the result of Theorem 1 in~\cite{Ho_Instability} indicates the following convergence rate of the fixed-step size gradient descent updates when $\alpha > 0$ and $\alpha \geq \gamma$.
\begin{proposition}
\label{prop:fixed_gradient_descent}
Assume that Assumptions~\ref{assump:smoothness},~\ref{assump:nonPL} and~\ref{assump:stab} hold with $\alpha \geq \gamma$ and $\alpha > 0$. As long as the sample size $n$ is large enough such that $\varepsilon(n, \delta) \leq C$ for some universal constant $C$, there exist universal constants $C_{1}'$ and $C_{2}'$ such that for any fixed $\tau \in (0, \frac{1}{1+\alpha - \gamma})$ as long as $t \geq C_{1}' \varepsilon(n, \delta)^{-\frac{\alpha}{1+\alpha - \gamma}} \log(1/ \tau)$, we have
\begin{align*}
    \|\theta_{n,\text{GD}}^{t} - \theta^{*}\| \leq C_{2}'\varepsilon(n, \delta)^{\frac{1}{\alpha + 1 - \gamma} - \tau},
\end{align*}
where $\{\theta_{n,\text{GD}}^{t}\}_{t \geq 0}$ is a sequence of sample fixed-step size gradient descent iterates. 
\end{proposition}
When $\alpha \geq \gamma$ and $\alpha > 0$, the result of Proposition~\ref{prop:fixed_gradient_descent} indicates that the fixed-step size gradient descent algorithm requires $\mathcal{O}(\varepsilon(n, \delta)^{-\frac{\alpha}{1+\alpha - \gamma}})$ number of iterations such that its updates can reach to the final statistical radius $\mathcal{O}(\varepsilon(n, \delta)^{\frac{1}{\alpha + 1 - \gamma}})$. Since each step of the gradient descent algorithm takes $\mathcal{O}(nd)$ arithmetic operations, it demonstrates that the total computational complexity for the fixed-step size gradient descent algorithm to reach the final statistical radius is $\mathcal{O}(n \cdot \varepsilon(n, \delta)^{-\frac{\alpha}{1+\alpha - \gamma}})$ for fixed dimension $d$. On the other hand, with a similar argument, Theorem~\ref{theorem:convergence_rate_Polyak} indicates that the total computational complexity for the Polyak step size gradient descent iterates to reach the final statistical radius is at the order of $\mathcal{O}(n \cdot \log (1/\varepsilon(n, \delta))$, which is much cheaper than that of the fixed-step size gradient descent algorithm when $\alpha \geq \gamma$. 

\vspace{0.5 em}
\noindent
\textbf{Cross-validation with the minimum number of iterates:} 
Note that, in Theorem \ref{theorem:convergence_rate_Polyak} we only guarantee for the existence of some $k < t $ in the iterate that $\|\theta_n^k - \theta^*\| = \mathcal{O}(\varepsilon(n, \delta)^{\frac{1}{\alpha + 1 - \gamma}})$, instead of the generally desired last iterate $t$. As Ho et al.~\citep{Ho_Instability} pointed out, such minimum is unavoidable without further regularity conditions. Fortunately, we can still obtain the desired estimator in the iterate by cross-validation~\citep{stone1974cross}, which only accounts for an additional $\mathcal{O}(nd)$ computation and keeps the computational efficiency of the Polyak step-size gradient descent algorithm.

\vspace{0.5 em}
\noindent
\textbf{Practical consideration of the Polyak step size gradient descent:} A practical issue of the original Polyak step size gradient descent algorithm is that it requires the knowledge of the optimal value of the sample loss function $f_{n}(\widehat{\theta}_{n})$ (see equation~\eqref{eq:sample_operator_Polyak}). Even though it may look restrictive at the first sight, it appears that we can utilize an adaptive version of that algorithm, named \emph{adaptive Polyak step size gradient descent}, from~\cite{Sham_Polyak} to deal with the unknown value of $f_{n}(\widehat{\theta}_{n})$. The detailed description of that algorithm is in Algorithm~\ref{algorithm:adaptive_Polyak_v1}. 

As indicated in Algorithm~\ref{algorithm:adaptive_Polyak_v1}, we first choose some lower bound $\tilde{f}_{0}$ of $f_{n}(\widehat{\theta}_{n})$ and using it as a surrogate for $f_{n}(\widehat{\theta}_{n})$. Then, we run the Polyak step size algorithm for $T$ times, which is the time horizon, with that surrogate choice. We then perform binary search to update that surrogate value to $\tilde{f}_{1}$ based on the current Polyak step size gradient descent iterates. We repeat that procedure $K$ times where $K$ is some given number of epochs to obtain a surrogate value $\tilde{f}_{K}$ of $f_{n}(\widehat{\theta}_{n})$. As indicated in Theorem 2 of~\cite{Sham_Polyak}, to have $\tilde{f}_{K} - f_{n}(\widehat{\theta}_{n}) < \varepsilon$, we can choose $K = \mathcal{O}(\log (\frac{f_{n}(\widehat{\theta}_{n}) - \tilde{f}_{0}}{\varepsilon}))$ and $T = \mathcal{O}(\log(\frac{1}{\varepsilon}))$. Therefore, if we choose $\varepsilon = \mathcal{O}(\varepsilon(n, \delta)^{\frac{\alpha + 2}{\alpha + 1 - \gamma}})$ (note that here $\varepsilon$ is the gap for value of the objective function), then based on the proof of Theorem~\ref{theorem:convergence_rate_Polyak}, the adaptive Polyak step size gradient descent iterates converge to a final radius of convergence $\mathcal{O}(\varepsilon(n, \delta)^{\frac{1}{\alpha + 1 - \gamma}})$ after $\mathcal{O}(\log (1/\varepsilon(n, \delta))^2)$ number of iterations. It indicates that the adaptive Polyak step size gradient descent is still cheaper than the fixed-step gradient descent algorithm for reaching the same final statistical radius when $\alpha \geq \gamma$ and $\alpha > 0$, i.e., when the population loss function is not locally strongly convex.
\begin{algorithm}[t!]
\DontPrintSemicolon
  
  \KwInput{Sample loss function $f_n$, initialization $\theta_{n}^0$, lower bound function $\tilde{f}_{0}$ such that $\tilde{f}_{0} < f_{n}(\widehat{\theta}_{n})$ where $\widehat{\theta}_{n}$ is some optimal solution of $f_{n}$, time horizon $T$, number of epochs $K$}
  $\bar{\theta} = \theta_{n}^{0}$ \\
  \For {$k= 0, 1,2,\ldots,K - 1$}
  {$\theta_{n}^{Tk} = \bar{\theta}$ \\
  \For {$i = 0, 1, 2, \ldots, T - 1$}
  {$\theta_{n}^{Tk + i + 1}=\theta_{n}^{Tk + i}-\frac{f_n(\theta_{n}^{Tk + i})-\tilde{f}_{k}}{\|\nabla f_n(\theta_{n}^{Tk+i})\|^2}\nabla f_n(\theta_{n}^{Tk+i})$
  }
  $\bar{\theta} = \arg \min_{0 \leq i \leq T} f_{n}(\theta_{n}^{Tk + i})$ \\
  $\tilde{f}_{k + 1} = \frac{f_{n}(\bar{\theta}) - \tilde{f}_{k}}{2}$ 
  }
  \KwOutput{$\bar{\theta}$}
\caption{Adaptive Polyak Step Size Gradient Descent}
\label{algorithm:adaptive_Polyak_v1}
\end{algorithm}

\section{Examples}
\label{sec:examples}
In this section, we consider an application of our theories in Section~\ref{sec:Polyak_Step_Size} to three specific examples: generalized linear model, over-specified Gaussian mixture model, and mixed linear regression model.
\subsection{Generalized Linear Model}
\label{sec:example_glm}
Generalized linear model is a generalization of linear regression model that allows the response variable to relate to the covariates via a link function. In particular, assume that $(Y_{1}, X_{1}), \ldots, (Y_{n}, X_{n}) \in \mathbb{R} \times \mathbb{R}^{d}$ satisfy
\begin{align}
    Y_{i} = g(X_{i}^{\top}\theta^{*}) + \varepsilon_{i}, \quad \quad \forall i \in [n] \label{eq:generalized_linear}
\end{align}
where $g: \mathbb{R} \to \mathbb{R}$ is a given link function, $\theta^{*}$ is a true but unknown parameter, and $\varepsilon_{1},\ldots,\varepsilon_{n}$ are i.i.d. noises from $\mathcal{N}(0, \sigma^2)$ where $\sigma > 0$ is a given variance parameter. Note that, the Gaussian assumption on the noise is just for the simplicity of the proof argument; the result in this section still holds for sub-Gaussian i.i.d. noise. Furthermore, we assume the random design setting of the generalized linear model, namely, $X_{1}, X_{2}, \ldots, X_{n}$ are i.i.d. from $\mathcal{N}(0, I_{d})$.

For our study, we specifically consider $g(r) : = r^{p}$ for any $p \in \mathbb{N}$ and $p \geq 2$. Note that, our choice of $g$ is motivated by phase retrieval problem \cite{Fienup_82,Shechtman_Yoav_etal_2015, candes_2011,Netrapalli_Prateek_Sanghavi_2015} where $g(r) = r^2$. To estimate $\theta^{*}$, we consider minimizing the least-square loss function, which is given by:
\begin{align}
    \min_{\theta \in \mathbb{R}^{d}} \mathcal{L}_{n}(\theta) : = \frac{1}{2n} \sum_{i = 1}^{n} (Y_{i} - (X_{i}^{\top} \theta)^{p})^{2}. \label{eq:sample_loss_linear}
\end{align}
We then also have the corresponding population least-square loss function, which admits the following form:
\begin{align*}
    \min_{\theta \in \mathbb{R}^{d}} \mathcal{L}(\theta) : = \mathbb{E}_{(X, Y)}[(Y - (X^{\top} \theta)^{p})^{2}],
\end{align*}
where the outer expectation is taken with respect to $X \sim \mathcal{N}(0, I_{d})$ and $Y = g(X^{\top} \theta^{*}) + \varepsilon$ where $\varepsilon \sim \mathcal{N}(0, \sigma^2)$. Note that $\mathbb{E}[Y^2|X]=\mathbb{E}[g(X^{\top}\theta^*)^{2}] + \sigma^2$. Thus, by taking conditional expectation, the population loss function has the following form: 
\begin{align}
    \mathcal{L}(\theta)
    &= \mathbb{E}\left[\frac{1}{2} (Y-(X^{\top}\theta)^p)^2\right] \nonumber \\
    &= \frac{1}{2} \left( \mathbb{E}\left[\left((X^{\top}\theta^*)^{p}-(X^{\top}\theta)^{p}\right)^2\right] + \sigma^2 \right). \label{eq:general_form_population_linear_model}
\end{align}

\vspace{0.5 em}
\noindent
\textbf{Strong signal-to-noise regime:} When $\theta^{*}$ is bounded away from 0, i.e., $\|\theta^{*}\| \geq C$ for some universal constant $C$, the population least-square loss function $\mathcal{L}$ is locally strongly convex around $\theta^{*}$ and locally smooth, namely, the Assumptions~\ref{assump:smoothness} and~\ref{assump:nonPL} become
\begin{align}
    \lambda_{\max}(\nabla^2 \mathcal{L}(\theta)) \leq c_{1}, \quad \|\nabla \mathcal{L}(\theta)\| \geq c_2(f(\theta) - f(\theta^*))^{1/2} \label{eq:geometry_generalized_linear_strong}
\end{align}
for all $\theta \in \mathbb{B}(\theta^{*}, \rho)$ where $\rho$ is some universal constant depending on $p$, as we demonstrate in  Appendix~\ref{sec:smoothness_PL_generalized_linear}.
Furthermore, for Assumption~\ref{assump:stab}, for any $r > 0 $ we can demonstrate that there exist universal constants $C_{1}$ and $C_{2}$ such that as long as $n \geq C_{1} (d \log(d/ \delta))^{2p}$ with probability $1 - \delta$
\begin{align}
    \sup_{\theta\in \mathbb{B}(\theta^*, r)} \|\nabla \mathcal{L}_n(\theta) - \nabla \mathcal{L}(\theta)\|\leq C_{2} \sqrt{\frac{d + \log(1/\delta)}{n}}. \label{eq:concentration_generalized_linear_strong}
\end{align}
The proof for this uniform concentration bound is also in Appendix~\ref{sec:smoothness_PL_generalized_linear}.
These results indicate that $\alpha = \gamma = 0$ in Assumptions~\ref{assump:smoothness}-\ref{assump:stab}. Therefore, a direct application of Theorem~\ref{theorem:convergence_rate_Polyak} shows that we have the iterates of Polyak step size gradient descent algorithm converge to a radius of convergence $\mathcal{O}(\sqrt{d/n})$ around $\theta^{*}$ within $\mathcal{O}(\log(n/ d))$ number of iterations. 

\vspace{0.5 em}
\noindent
\textbf{Low signal-to-noise regime:} On the other hand, when $\|\theta^{*}\|$ is sufficiently small, the population loss function is no longer locally strongly convex and the precise understandings of the sample updates from the Polyak step size gradient descent algorithm for solving the sample loss function $\mathcal{L}_{n}$ have remained poorly understood. To illustrate the behaviors of the Polyak step size gradient descent algorithm, we only focus on the no signal-to-noise setting $\theta^{*} = 0$ in this section. Under this setting, the population least-square loss function can be written as
\begin{align}
    \min_{\theta \in \mathbb{R}^{d}} \mathcal{L}(\theta) = \frac{\sigma^2 + (2p - 1)!! \|\theta - \theta^{*}\|^{2p}}{2}. \label{eq:population_no_signal}
\end{align}
Different from the setting when $\theta^{*}$ is bounded away from 0, the population loss function in equation~\eqref{eq:population_no_signal} is not locally strongly convex around $\theta^{*}$ when $\theta^{*} = 0$. Indeed, we demonstrate in Appendix~\ref{sec:smoothness_PL_generalized_linear} that for all $\theta \in \mathbb{B}(\theta^{*}, \rho)$ for some radius $\rho$, we have
\begin{align}
    \lambda_{\max}(\nabla^2 \mathcal{L}(\theta)) & \leq c_{1} \|\theta - \theta^{*}\|^{2p - 2}, \label{eq:Lipschitz_generalized_model} \\
    \|\nabla \mathcal{L}(\theta)\| & \geq c_{2} (\mathcal{L}(\theta) - \mathcal{L}(\theta^{*}))^{1 - \frac{1}{2p}}, \label{eq:geometry_generalized_model}
\end{align}
where $c_{1}, c_{2}$ are some universal constants depending on $r$. Furthermore, for Assumption~\ref{assump:stab}, from Appendix A.2 in~\citep{mou2019diffusion}, there exist universal constants $C_{1}$ and $C_{2}$ such that for any $r > 0$ and $n \geq C_{1} (d \log(d/ \delta))^{2p}$ we have
\begin{align}
    \sup_{\theta \in \mathbb{B}(\theta^{*}, r)} \|
    \nabla \mathcal{L}_{n}(\theta) - \nabla \mathcal{L}(\theta) \| \leq C_{2} (r^{p - 1} + r^{2p - 1}) \sqrt{\frac{d + \log(1/ \delta)}{n}}  \label{eq:concentration_gradient_generalized_model}
\end{align}
with probability at least $1 - \delta$. These results suggest that as long as $r \in (0, \rho)$ for some given $\rho$, the values of constants $\alpha$ and $\gamma$ in Assumptions~\ref{assump:smoothness}-\ref{assump:stab} are $\alpha = 2p - 2$ and $\gamma = p - 1$. 

Given the above studies, a direct application of Theorem~\ref{theorem:convergence_rate_Polyak} leads to the following bounds on the statistical radius of the sample Polyak step size gradient descent iterates. 
\begin{corollary}
\label{corollary:generalized_model} 
For the generalized linear model~\eqref{eq:generalized_linear} with the link function $g(r) = r^{p}$ for some natural number $p \geq 2$, as long as $n \geq c (d \log(d/ \delta))^{2p}$ for some positive universal constant $c$ and $\theta_{n}^{0} \in \mathbb{B}(\theta^{*}, \rho)$ for some $\rho > 0$, with probability $1 - \delta$ the sequence of sample Polyak step size gradient descent iterates $\{\theta_{n}^{t}\}_{t \geq 0}$ satisfies the following bounds
\begin{itemize}
\item[(i)] Strong signal-to-noise regime: When $\|\theta^{*}\| \geq C$ for some constant $C$, we have
\begin{align*}
        \min_{1 \leq k \leq t} \| \theta_{n}^{k} - \theta^{*} \| & \leq c_{1} \sqrt{\frac{d + \log(1/\delta)}{n}}, \quad \quad \text{for} \ t \geq c_{2} \log \parenth{\frac{n}{d + \log(1/\delta)}},
\end{align*}
\item[(ii)] Low signal-to-noise regime: When $\theta^{*} = 0$, we find that
\begin{align*}
        \min_{1 \leq k \leq t} \| \theta_{n}^{k} - \theta^{*} \| & \leq c_{1}' \parenth{\frac{d + \log(1/\delta)}{n}}^{1/(2p)}, \quad \quad \text{for} \ t \geq c_{2}' \log \parenth{\frac{n}{d + \log(1/\delta)}}
\end{align*}
\end{itemize}
Here, $c_{1}, c_{2}, c_{1}', c_{2}'$ are some universal constants.
\end{corollary}
In light of Proposition~\ref{prop:fixed_gradient_descent}, when $\theta^{*} = 0$ the iterates from the fixed-step size gradient descent algorithm have similar statistical radius $(d/n)^{1/(2p)}$ as that of the Polyak step size gradient descent updates. However, the fixed-step size gradient descent algorithm need at least $\mathcal{O}((n/d)^{\frac{p - 1}{p}})$ number of iterations to reach that radius of convergence. It demonstrates that the Polyak step size gradient descent algorithm is much cheaper than the fixed-step size gradient descent algorithm in terms of the sample size $n$.
\subsection{Mixture model}
\label{sec:over_mixture_model}
Gaussian mixture models are one of the most popular tools in machine learning and statistics for modeling heterogeneous data~\cite{Lindsay-1995, Mclachlan-1988}. In these models, learning location and scale parameters associated with each sub-population is important to understand the heterogeneity of the data. A popular approach to estimate these parameters is to maximize the log-likelihood function. Since the log-likelihood function of Gaussian mixture models is non-concave and complicated to study, a full picture about the convergence rates of optimization algorithms for solving the log-likelihood function of the over-specified Gaussian mixture models has still remained elusive.

In this section, we aim to shed light on the convergence rates of the Polyak step size gradient descent algorithm for solving the parameters of Gaussian mixture models. We specifically consider the symmetric two-component Gaussian mixture and provide comprehensive analysis of that algorithm. In particular, we assume that the data $X_{1}, X_{2}, \ldots, X_{n}$ are i.i.d. samples from $\frac{1}{2} \mathcal{N}(-\theta^{*}, \sigma^2 I_{d}) + \frac{1}{2} \mathcal{N}(\theta^{*}, \sigma^2 I_{d})$ where $\sigma > 0$ is given and $\theta^{*}$ is true but unknown parameter. To estimate $\theta^{*}$, we fit the data by the symmetric two-component Gaussian mixture 
\begin{align}
\frac{1}{2} \mathcal{N}(-\theta, \sigma^2 I_{d}) + \frac{1}{2} \mathcal{N}(\theta, \sigma^2 I_{d}). \label{eq:overspecify_mixture}
\end{align}
The maximum likelihood estimation is then given by:
\begin{align}
    \min_{\theta \in \mathbb{R}^{d}} \bar{\mathcal{L}}_{n}(\theta) : = -\frac{1}{n} \sum_{i = 1}^{n} \log \left(\frac{1}{2}\phi(X_{i}|\theta, \sigma^2 I_{d}) + \frac{1}{2}\phi(X_{i}|-\theta, \sigma^2 I_{d})\right), \label{eq:sample_loglikelihood}
\end{align}
where $\phi(\cdot|\theta, \sigma^2I_{d})$ is the density function of multivariate Gaussian distribution with mean $\theta$ and covariance matrix $\sigma^2I_{d}$. The corresponding population version of the maximum likelihood estimation~\eqref{eq:sample_loglikelihood} takes the following form:
\begin{align}
    \min_{\theta \in \mathbb{R}^{d}} \bar{\mathcal{L}}(\theta) : =  - \mathbb{E}_{X}\brackets{\log \left(\frac{1}{2} \phi(X|\theta, \sigma^2 I_{d}) + \frac{1}{2} \phi(X|-\theta, \sigma^2 I_{d})\right)}, \label{eq:population_loglikelihood}
\end{align}
where the outer expectation is taken with respect to $X \sim \frac{1}{2} \mathcal{N}(-\theta^{*}, \sigma^2 I_{d}) + \frac{1}{2} \mathcal{N}(\theta^{*}, \sigma^2I_{d})$. 

\vspace{0.5 em}
\noindent
\textbf{Strong signal-to-noise regime:} When $\|\theta^{*}\| \geq C \sigma$ for some universal constant $C$, the Corollary 1 in~\cite{Siva_2017} demonstrates that the population loss function $\bar{\mathcal{L}}$ is locally strongly convex and locally smooth as long as $\theta \in \mathbb{B}(\theta^{*}, \frac{\|\theta^{*}\|}{4})$. It indicates that we have
\begin{align}
    \lambda_{\max}(\nabla^2 \bar{\mathcal{L}}(\theta)) \leq c_{1}, \quad \|\nabla \bar{\mathcal{L}}(\theta)\| \geq c_2(f(\theta) - f(\theta^*))^{1/2}, \label{eq:geometry_mixture_model_strong}
\end{align}
i.e., the Assumptions~\ref{assump:smoothness} and~\ref{assump:nonPL} are satisfied with the constant $\alpha = 0$. Furthermore, for any $r \leq \frac{\|\theta^{*}\|}{4}$ and $n \geq C_{1} d \log(1/ \delta)$ for some universal constant $C_{1}$ we have
\begin{align}
    \sup_{\theta \in \mathbb{B}(\theta^{*}, r)} \| \nabla \bar{\mathcal{L}}_{n}(\theta) - \nabla \bar{\mathcal{L}}(\theta) \| \leq C_{2} \sqrt{\frac{d \log(1/ \delta)}{n}}  \label{eq:concentration_gradient_mixture_model}
\end{align}
with probability at least $1 - \delta$ where $C_{2}$ is some universal constant. See Corollary 4 in~\cite{Siva_2017} for the proof of this concentration result.

\vspace{0.5 em}
\noindent
\textbf{Low signal-to-noise regime:}
We specifically consider the setting $\theta^{*} = 0$. This setting corresponds to the popular over-specified Gaussian mixture models~\cite{Rousseau-2011, Ho-Nguyen-AOS-17}, namely, when we choose some given number of components that can be (much) larger than the true number of components and estimating the parameters from the mixture models with that chosen number of components. We prove in Appendix~\ref{sec:smoothness_PL_mixture_model} that for all $\theta \in \mathbb{B}(\theta^{*}, \frac{\sigma}{2})$:
\begin{align}
    \lambda_{\max}(\nabla^2 \bar{\mathcal{L}}(\theta)) & \leq c_{1} \|\theta - \theta^{*}\|^{2}, \label{eq:Lipschitz_mixture_model_low_signal} \\
    \|\nabla \bar{\mathcal{L}}(\theta)\| & \geq c_{2} (\bar{\mathcal{L}}(\theta) - \bar{\mathcal{L}}(\theta^{*}))^{3/4}. \label{eq:geometry_mixture_model_low_signal}
\end{align}
Furthermore, from Lemma 1 in~\citep{Raaz_Ho_Koulik_2020}, there exist universal constants $C_{1}$ and $C_{2}$ such that for any $r > 0$ and $n \geq C_{1} d \log(1/ \delta)$ we have:
\begin{align}
    \sup_{\theta \in \mathbb{B}(\theta^{*}, r)} \| \nabla \bar{\mathcal{L}}_{n}(\theta) - \nabla \bar{\mathcal{L}}(\theta) \| \leq C_{2} r \sqrt{\frac{d \log(1/ \delta)}{n}}  \label{eq:concentration_gradient_mixture_model}
\end{align}
with probability at least $1 - \delta$.

Combining the above results to Theorem~\ref{theorem:convergence_rate_Polyak}, we have the following results on the final statistical radius of the Polyak step size iterates under different regimes of the two-component Gaussian mixture model.
\begin{corollary}
\label{corollary:mixture_model} 
For the symmetric two-component mixture model~\eqref{eq:overspecify_mixture}, there exist positive universal constants $c_{1}, c_{2}, c_{1}', c_{2}'$ such that when $n \geq c d \log(1/ \delta)$ for some universal constant $c$, with probability $1 - \delta$ the sequence of sample Polyak step size gradient descent iterates $\{\theta_{n}^{t}\}_{t \geq 0}$ satisfies the following bounds:
\begin{itemize}
\item[(i)] Strong signal-to-noise regime: When $\|\theta^{*}\| \geq C$ for some constant $C$ and $\theta_{n}^{0} \in \mathbb{B}(\theta^{*}, \frac{\|\theta\|^{*}}{4})$, we have
\begin{align*}
        \min_{1 \leq k \leq t} \| \theta_{n}^{k} - \theta^{*} \| & \leq c_{1} \sqrt{\frac{d \log(1/\delta)}{n}}, \quad \quad \text{for} \ t \geq c_{2} \log \parenth{\frac{n}{d \log(1/\delta)}},
\end{align*}
\item[(ii)] Low signal-to-noise regime: When $\theta^{*} = 0$ and $\theta_{n}^{0} \in \mathbb{B}(\theta^{*}, \frac{\sigma}{2})$ we find that
\begin{align*}
        \min_{1 \leq k \leq t} \| \theta_{n}^{k} - \theta^{*} \| & \leq c_{1}' \parenth{\frac{d \log(1/\delta)}{n}}^{1/4}, \quad \quad \text{for} \ t \geq c_{2}' \log \parenth{\frac{n}{d \log(1/\delta)}}.
\end{align*}
\end{itemize}
\end{corollary}
A few comments with the results of Corollary~\ref{corollary:mixture_model} are in order. First, the Expectation-Maximization (EM) algorithm~\cite{Rubin-1977} is a popular algorithm for solving the parameters of Gaussian mixture models. In the symmetric two-component Gaussian mixture~\eqref{eq:overspecify_mixture}, the EM algorithm is simply the gradient descent with step size being 1. In light of the results of Proposition~\ref{prop:fixed_gradient_descent} and the results in~\cite{Raaz_Ho_Koulik_2020}, the EM iterates reach to the final statistical radius $\mathcal{O}((d/n)^{1/4})$ after $\mathcal{O}(\sqrt{n})$ number of iterations. The results in Corollary~\ref{corollary:mixture_model} indicate that the Polyak step size gradient descent iterates reach to the final statistical radius with a much fewer number of iterations, namely, $\mathcal{O}(\log(n))$, while each iteration of the Polyak step size gradient descent has similar computational complexity as that of the EM algorithm. Therefore, the Polyak step size gradient descent algorithm is more efficient than the EM algorithm for the low-signal-to noise regime of symmetric two-component Gaussian mixture model. Second, the statistical radius $(d/n)^{1/4}$ that the Polyak iterates reach to in the low signal-to-noise regime is optimal according to the work~\cite{Ho-Nguyen-EJS-16}. 
\subsection{Mixed linear regression}
\label{sec:mix_linear}
Mixed linear regression is a generalization of vanilla linear regression model when we have multiple mean parameters and each data can associate with one of these parameters. In statistics, mixed linear regression is often referred to as mixture of regression~\cite{Khalili-2007}, which is also a special case of mixture of experts~\cite{Jacob_Jordan-1991, Jordan-1994} where the mixing weights are assumed to be independent of the covariates. 

Similar to mixture model in Section~\ref{sec:over_mixture_model}, we also aim to shed light on the convergence rate of the Polyak step size gradient descent algorithm under the simple symmetric two-component mixed linear regression setting. In particular, we assume that $(X_{1}, Y_{1}), (X_{2}, Y_{2}), \ldots, (X_{n}, Y_{n})$ are i.i.d. samples from symmetric two components 
\begin{align}
    \biggr(\frac{1}{2} \mathcal{N}(Y| -(\theta^{*})^{\top}X, \sigma^2) + \frac{1}{2} \mathcal{N}(Y| (\theta^{*})^{\top}X, \sigma^2)\biggr) \cdot \mathcal{N}(X|0, I_{d}), \label{eq:mixed_linear}
\end{align}
where $\sigma > 0$ is known variance and $\theta^{*}$ is true but unknown parameter. To estimate $\theta^{*}$, we fit the data with the following symmetric two-component mixed linear regression:
\begin{align}
    \parenth{\frac{1}{2} \mathcal{N}(Y| - \theta^{\top}X, \sigma^2) + \frac{1}{2} \mathcal{N}(Y| \theta^{\top}X, \sigma^2)} \cdot \mathcal{N}(X|0, I_{d}). \label{eq:fitted_mixed_linear}
\end{align}
A common approach to obtain an estimator of $\theta^{*}$ is maximum likelihood estimator, which is given by:
\begin{align}
    \min_{\theta \in \mathbb{R}^{d}} \widetilde{\mathcal{L}}_{n}(\theta) : = -\frac{1}{n} \sum_{i = 1}^{n} \log \left(\frac{1}{2}\phi(Y_{i}|\theta^{\top}X_{i}, \sigma^2) + \frac{1}{2}\phi(Y_{i}|-\theta^{\top}X_{i}, \sigma^2)\right). \label{eq:sample_loglikelihood_mixed_linear}
\end{align}
The corresponding population version of the optimization problem~\eqref{eq:sample_loglikelihood_mixed_linear} is
\begin{align}
    \min_{\theta \in \mathbb{R}^{d}} \widetilde{\mathcal{L}}(\theta) : = -\mathbb{E}_{X, Y} \brackets{\log \left(\frac{1}{2}\phi(Y|\theta^{\top}X, \sigma^2) + \frac{1}{2}\phi(Y|-\theta^{\top}X, \sigma^2)\right)}, \label{eq:population_loglikelihood_mixed_linear}
\end{align}
where the outer expectation is taken with respect to $X \sim \mathcal{N}(0, I_{d})$ and $Y|X \sim \frac{1}{2} \mathcal{N}(Y| -(\theta^{*})^{\top}X, \sigma^2) + \frac{1}{2} \mathcal{N}(Y| (\theta^{*})^{\top}X, \sigma^2)$.

\vspace{0.5 em}
\noindent
\textbf{Strong signal-to-noise regime:} We first consider the setting when $\|\theta^{*}\| \geq C$ where $C$ is some universal constant. Corollary 2 in~\cite{Siva_2017} proves that for that strong signal-to-noise regime, the population negative log-likelihood function $\tilde{\mathcal{L}}$ is locally strongly convex and smooth when $\theta \in \mathbb{B}(\theta^{*}, \frac{\|\theta^{*}\|}{32})$. Therefore, the Assumptions~\ref{assump:smoothness} and~\ref{assump:nonPL} are satisfied with the constant $\alpha = 0$. Furthermore, according to the result of Corollary 5 in~\cite{Siva_2017}, Assumption~\ref{assump:stab} is satisfied with $\gamma = 0$ and for any radius $r \leq \|\theta^{*}\|/32$.

\vspace{0.5 em}
\noindent
\textbf{Low signal-to-noise regime:} We consider specifically the setting that $\theta^{*} = 0$. We prove in Appendix~\ref{sec:smoothness_PL_mixed_linear} that for all $\theta \in \mathbb{B}(\theta^{*}, \frac{\sigma}{\sqrt{20}})$, there exist universal constants $c_{1}$ and $c_{2}$ such that:
\begin{align}
    \lambda_{\max}(\nabla^2 \widetilde{\mathcal{L}}(\theta)) & \leq c_{1} \|\theta - \theta^{*}\|^{2}, \label{eq:smooth_mixed_linear_low_signal} \\
    \|\nabla \widetilde{\mathcal{L}}(\theta)\| & \geq c_{2} (\widetilde{\mathcal{L}}(\theta) - \tilde{\mathcal{L}}(\theta^{*}))^{3/4}. \label{eq:geometry_mixed_linear_low_signal}
\end{align}
These results indicate that the Assumptions~\ref{assump:smoothness} and~\ref{assump:nonPL} are satisfied with the constant $\alpha = 2$. Furthermore, from the concentration result from Lemma 2 of~\cite{Kwon_minimax}, there exist universal constants $C_{1}$ and $C_{2}$ such that as long as $n \geq C_{1} d \log(1/\delta)$, we have for any $r > 0$
\begin{align*}
    \mathbb{P} \parenth{\sup_{\theta \in \mathbb{B}(\theta^{*}, r)} \| \nabla \widetilde{\mathcal{L}}_{n}(\theta) - \nabla \widetilde{\mathcal{L}}(\theta) \| \leq C_{2} r \sqrt{\frac{d \log(1/ \delta)}{n}}} \geq 1 - \delta.
\end{align*}
It indicates that Assumption~\ref{assump:stab} is satisfied when $\gamma = 1$. Collecting all of the above results under both the strong and low signal-to-noise regimes, we have the following bounds on the statistical radii of the Polyak step size gradient descent iterates.
\begin{corollary}
\label{corollary:mixed_linear} 
For the symmetric two-component mixed linear regression~\eqref{eq:mixed_linear}, when $n \geq c \cdot d \log(1/ \delta)$ for some universal constant $c$, there exist positive universal constants $c_{1}, c_{2}, c_{1}', c_{2}'$ such that with probability $1 - \delta$ the sequence of sample Polyak step size gradient descent iterates $\{\theta_{n}^{t}\}_{t \geq 0}$ satisfies the following bounds:
\begin{itemize}
\item[(i)] Strong signal-to-noise regime: When $\|\theta^{*}\| \geq C$ for some constant $C$ and $\theta_{n}^{0} \in \mathbb{B}(\theta^{*}, \frac{\|\theta^{*}\|}{32})$, we have
\begin{align*}
        \min_{1 \leq k \leq t} \| \theta_{n}^{k} - \theta^{*} \| & \leq c_{1} \sqrt{\frac{d \log(1/\delta)}{n}}, \quad \quad \text{for} \ t \geq c_{2} \log \parenth{\frac{n}{d \log(1/\delta)}},
\end{align*}
\item[(ii)] Low signal-to-noise regime: When $\theta^{*} = 0$ and $\theta_{n}^{0} \in \mathbb{B}(\theta^{*}, \frac{\sigma}{\sqrt{20}})$ we find that
\begin{align*}
        \min_{1 \leq k \leq t} \| \theta_{n}^{k} - \theta^{*} \| & \leq c_{1}' \parenth{\frac{d \log(1/\delta)}{n}}^{1/4}, \quad \quad \text{for} \ t \geq c_{2}' \log \parenth{\frac{n}{d \log(1/\delta)}}.
\end{align*}
\end{itemize}
\end{corollary}
Similar to the symmetric two-component Gaussian mixture, the EM algorithm for solving the symmetric two-component mixed linear regression is simply the gradient descent with step size one. The results of Corollary~\ref{corollary:mixed_linear} and Proposition~\ref{prop:fixed_gradient_descent} indicate that the Polyak iterates take much fewer number of iterations, i.e., $\mathcal{O}(\log(n))$ than that of the EM algorithm, which is $\mathcal{O}(\sqrt{n})$. It indicates that the Polyak step size gradient descent algorithm is computationally more efficient than the EM algorithm for reaching to the optimal statistical radius $\mathcal{O}((d/n)^{1/4})$ in the low signal-to-noise regime. 

\section{Experiments}
\label{sec:experiments}
In this section, we illustrate the behaviors of Polyak step size gradient descent iterates for three statistical examples in Section~\ref{sec:examples}. In Section~\ref{sec:population_expriment}, we compare the behaviors of population Polyak step size gradient descent iterates and population fixed-step size gradient descent iterates for solving the population loss functions of the given statistical models. In Section~\ref{sec:sample_expriment}, we compare the sample iterates from both (adaptive) Polyak step size and fixed-step size gradient descent methods. 
\subsection{Population loss function}
\label{sec:population_expriment}
We first use Polyak step size and fixed-step size gradient descent algorithms to find the minima of the population losses of three examples in Section~\ref{sec:examples}. We consider these examples in $d = 2$ dimensions. For the strong signal-to-noise regime, we choose $\theta^* = (2,1)$. We compare the convergence rates of Polyak step size and fixed-step size iterates to the optimal solution $\theta^{*}$ of the population losses in Figure~\ref{pl:pop_diff_0}. In this figure, GLM, GMM, MLR respectively stand for generalized linear model, Gaussian mixture model, and mixed linear regression. All the plots in this figure are $\log$-$\log$ scale plots. From this figure, the Polyak step size GD iterates converge linearly to $\theta^{*}$  while the fixed-step size gradient descent iterates converge sub-linearly to $\theta^{*}$. These experiment results are consistent with our theories in Section~\ref{sec:examples}.

\begin{figure}[!t]
\centering
\subfigure{
\includegraphics[width=0.3\textwidth]{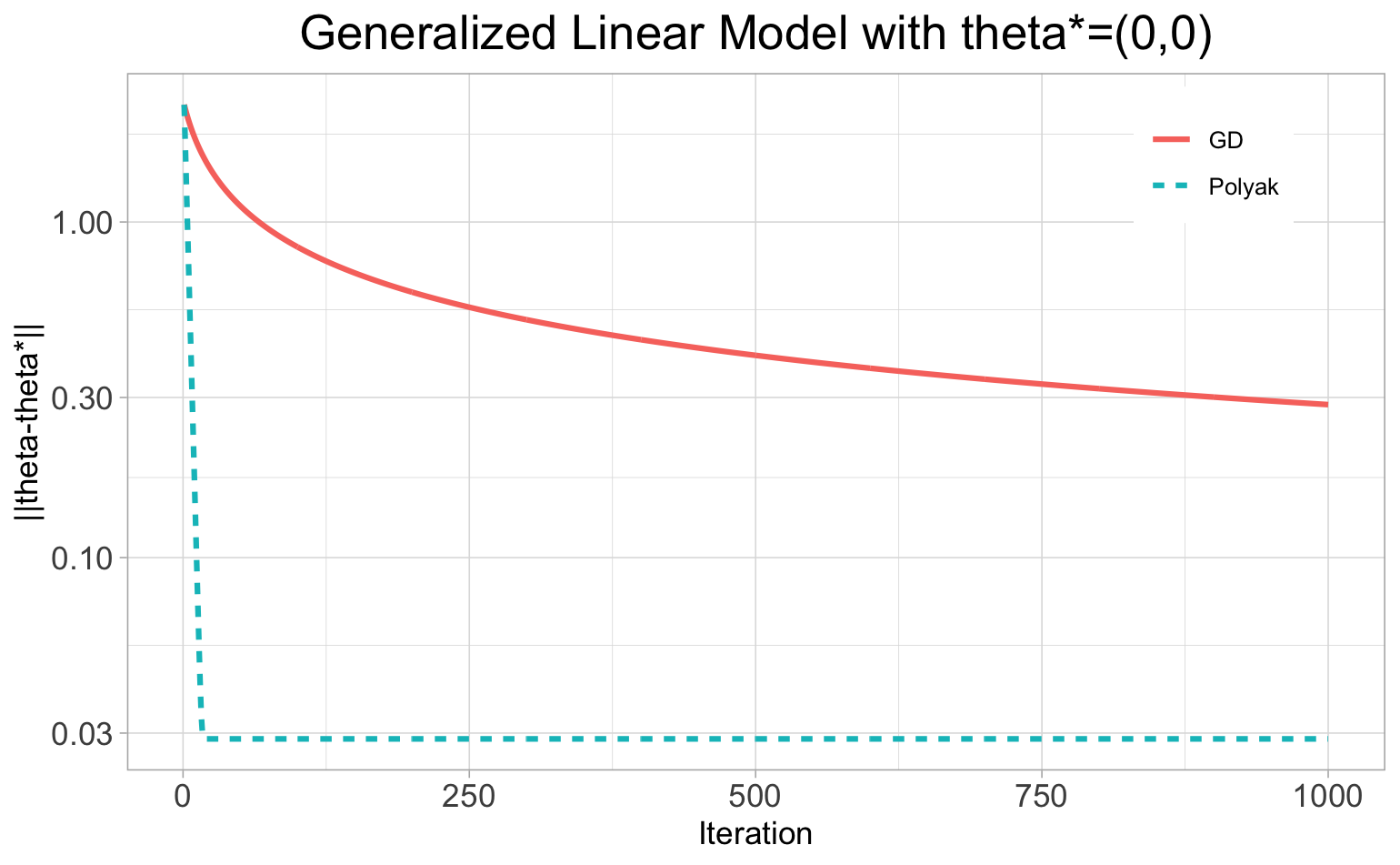}}
\subfigure{
\includegraphics[width=0.3\textwidth]{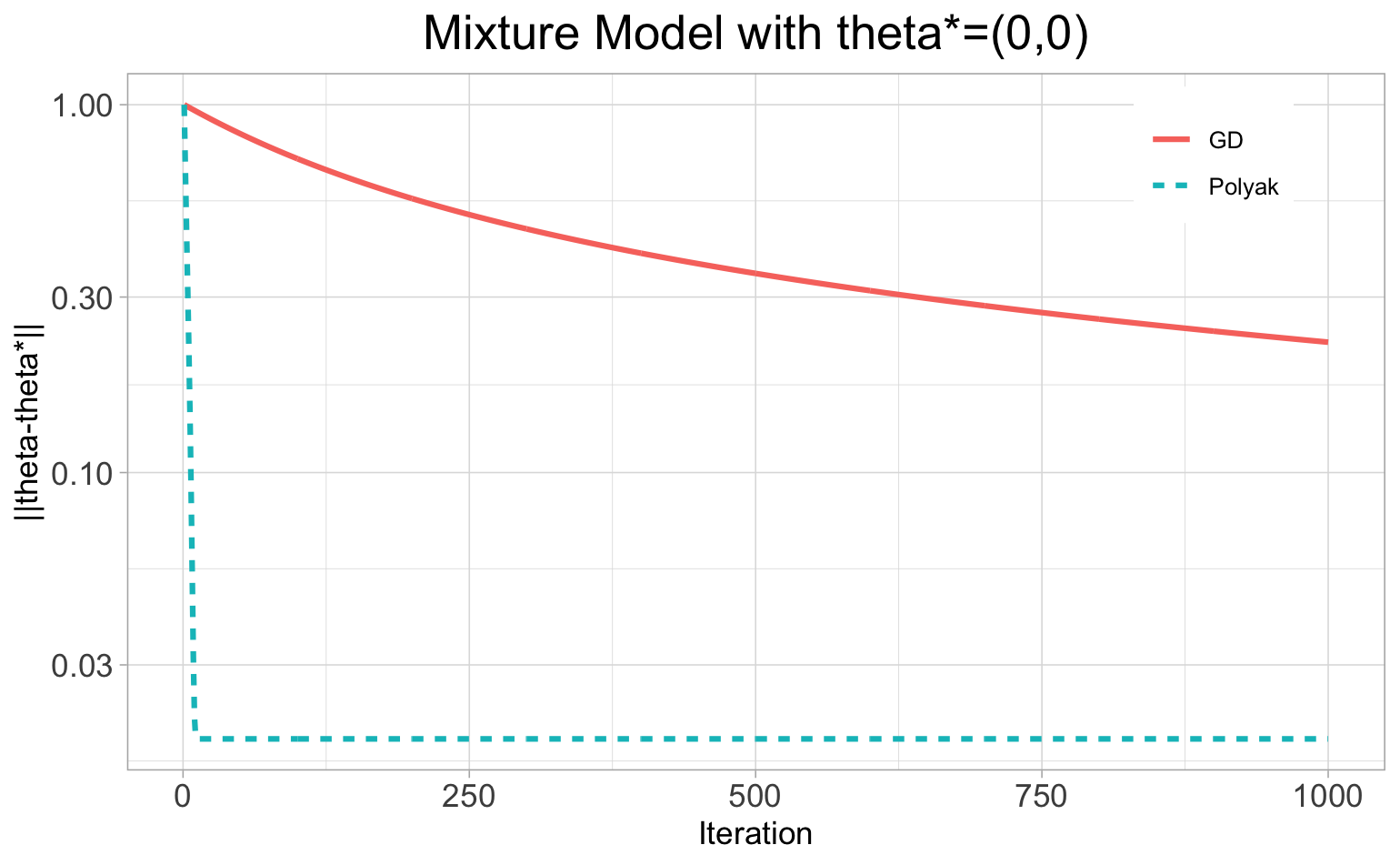}}
\subfigure{
\includegraphics[width=0.3\textwidth]{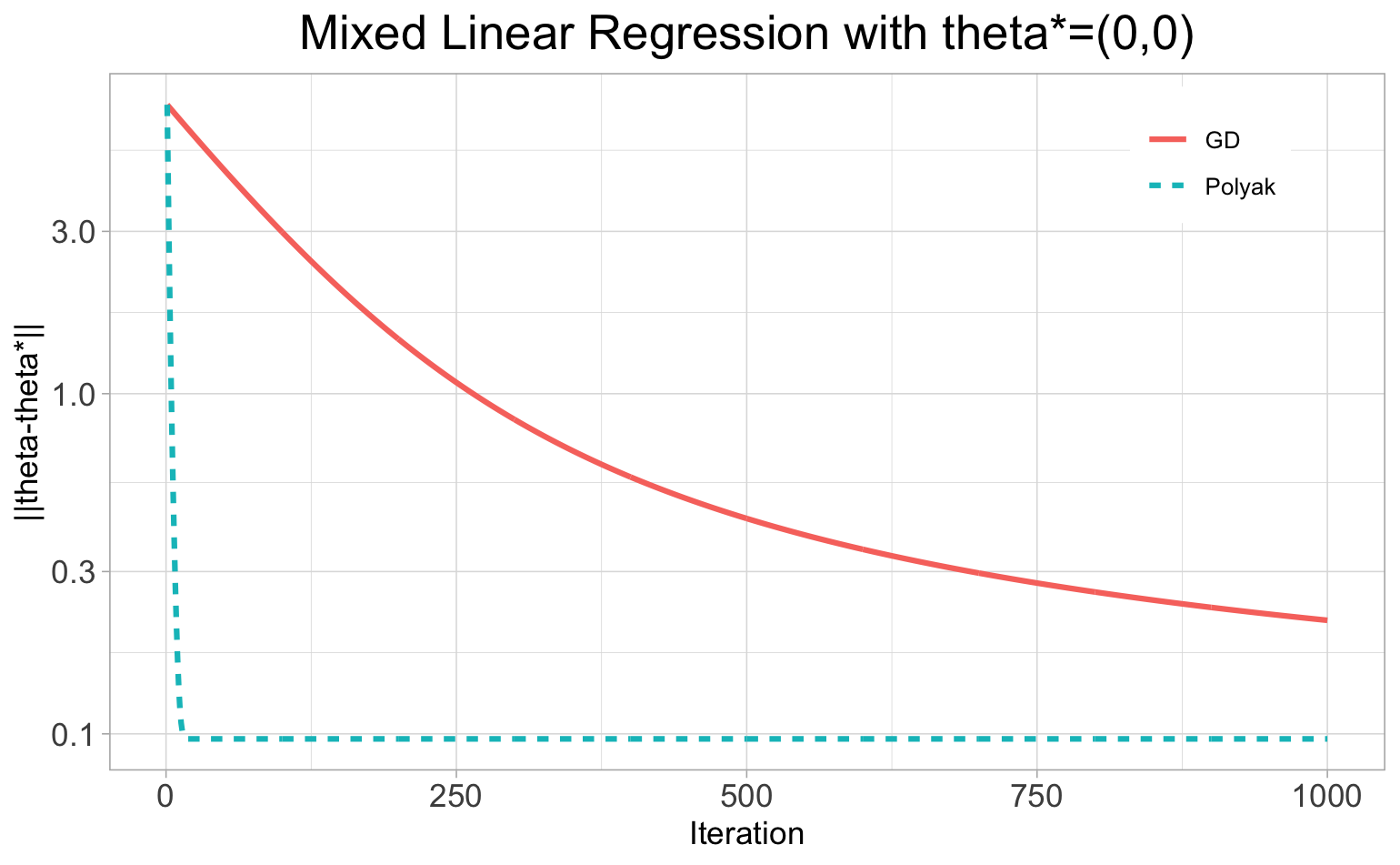}}
\bigskip
\subfigure{
\label{pl:GLM_pop_diff_n0}
\includegraphics[width=0.3\textwidth]{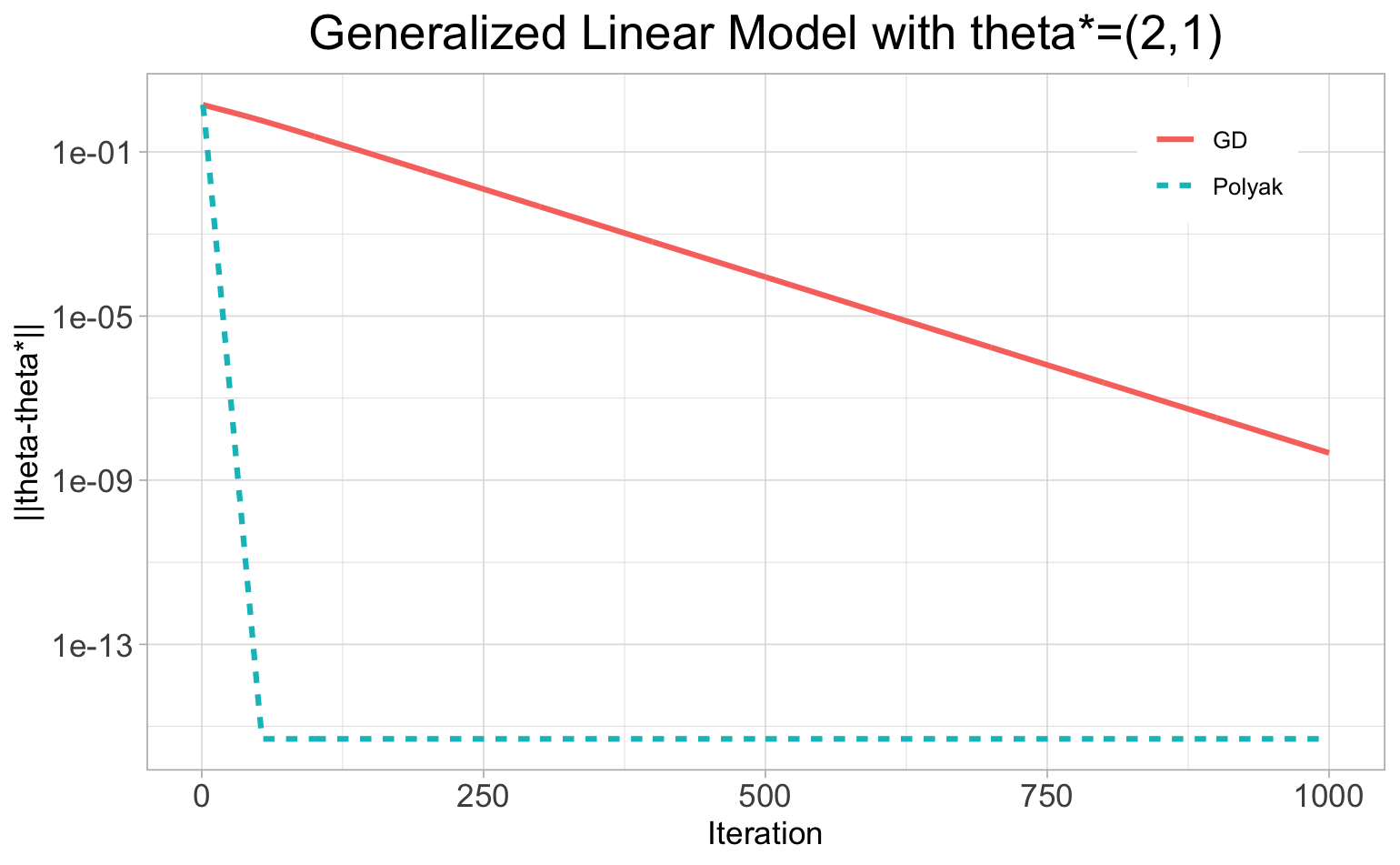}}
\subfigure{
\label{pl:GMM_pop_diff_n0}
\includegraphics[width=0.3\textwidth]{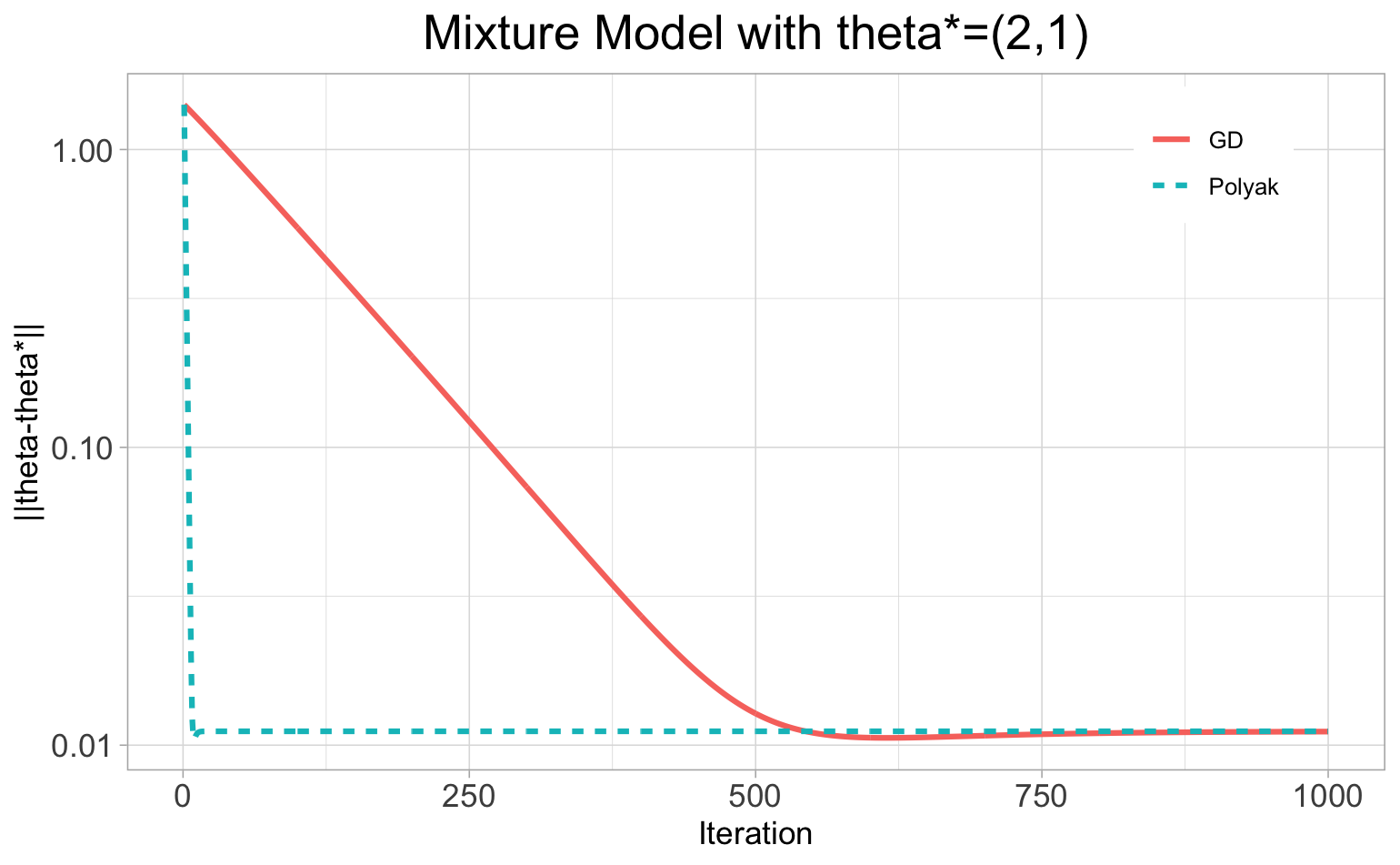}}
\subfigure{
\label{pl:MLR_pop_diff_n0}
\includegraphics[width=0.3\textwidth]{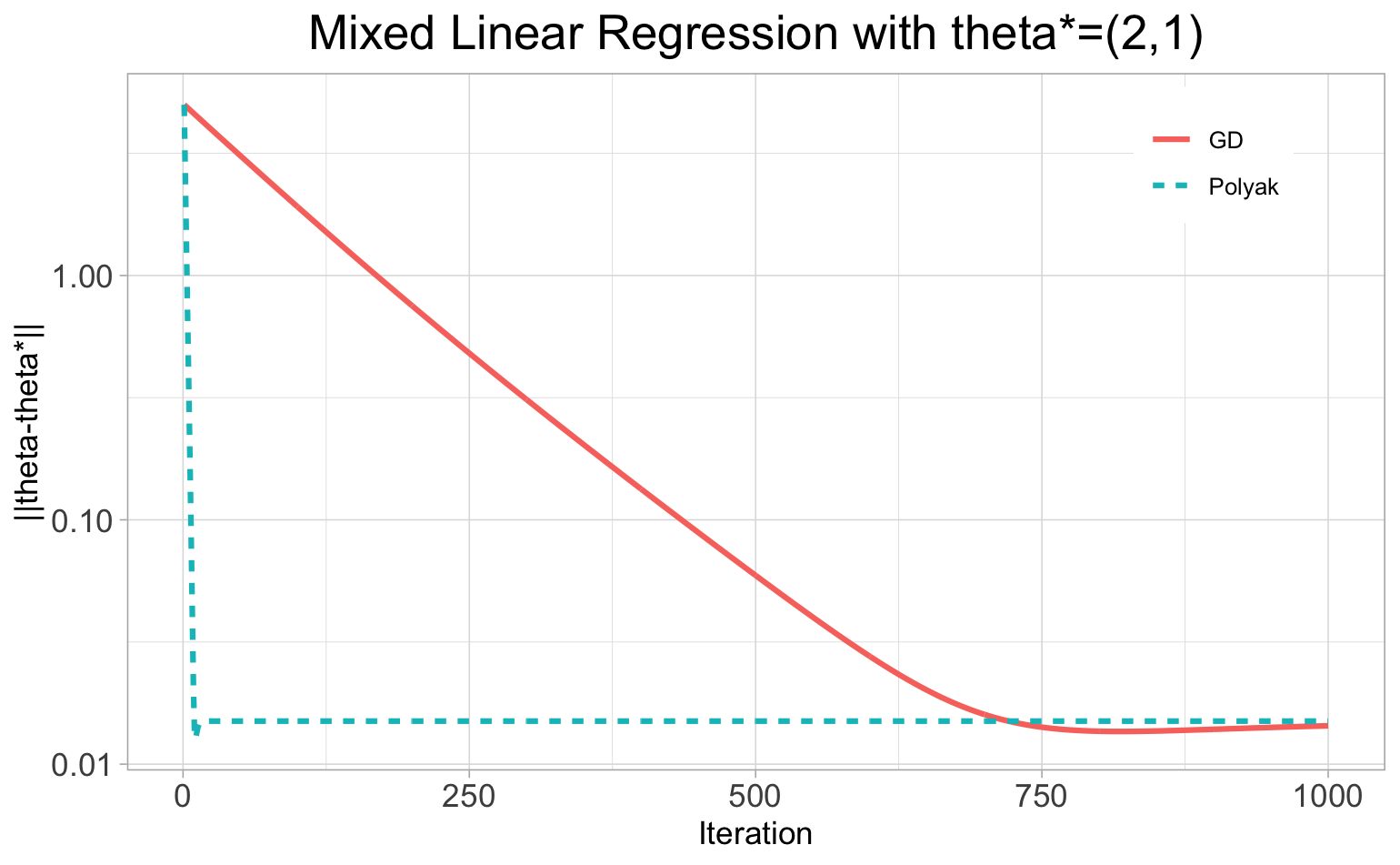}}
\caption{The convergence rates of Polyak step size and fixed-step size gradient descent iterates for solving the population losses of generalized linear model, Gaussian mixture model, and mixed linear regression model in Section~\ref{sec:examples}. The first row corresponds to the low signal-to-noise regime $\theta^{*} = (0,0)$ while the second row is for the strong signal-to-noise regime $\theta^{*} = (2,1)$.}
\label{pl:pop_diff_0}
\end{figure}



\subsection{Sample loss function}
\label{sec:sample_expriment}
Now, we carry out the experiments to compare the behaviors of Polyak step size and fixed-step size gradient descent iterates for solving the sample loss functions in three examples in Section~\ref{sec:examples}. In these examples, since we only observe the data, we do not have access to the optimal value of the sample loss functions. Therefore, we instead use the adaptive Polyak step size gradient descent in Algorithm~\ref{algorithm:adaptive_Polyak_v1} for these examples. The strategy for choosing the lower bound of the optimal value of the sample loss functions in that algorithm will be described in details in each example. In our experiments, the sample size $n$ is chosen to be in the set $\left\{ 1000,2000,\cdots,100000 \right\}$.
%

\vspace{0.5 em}
\noindent
\textbf{Generalized linear model:}
\label{sec:GLM_exp}
We first consider the generalized linear model in Section \ref{sec:example_glm}. We specifically choose the link functions $g(r) = r^2$ , i.e., $p = 2$. The data $(Y_{1}, X_{1}), \ldots, (Y_{n}, X_{n})$ satisfy
$$Y_i=(X_i^{\top}\theta^{*})^2+\varepsilon_i,$$
where $X_{1}, \ldots, X_{n} \overset{i.i.d.} \sim \mathcal{N}(0,I_{2})$ and $\varepsilon_{1}, \ldots, \varepsilon_{n} \overset{i.i.d.} \sim \mathcal{N}(0,0.01)$. We choose $\theta^{*} = (0, 0)$ for the low signal-to-noise regime and  $\theta^*= (0.5, 1)$ for the strong signal-to-noise regime in our experiments.

\begin{figure}[!t]
\centering
\subfigure{
\label{pl:GLM_0_fixerror}
\includegraphics[width=0.45\textwidth]{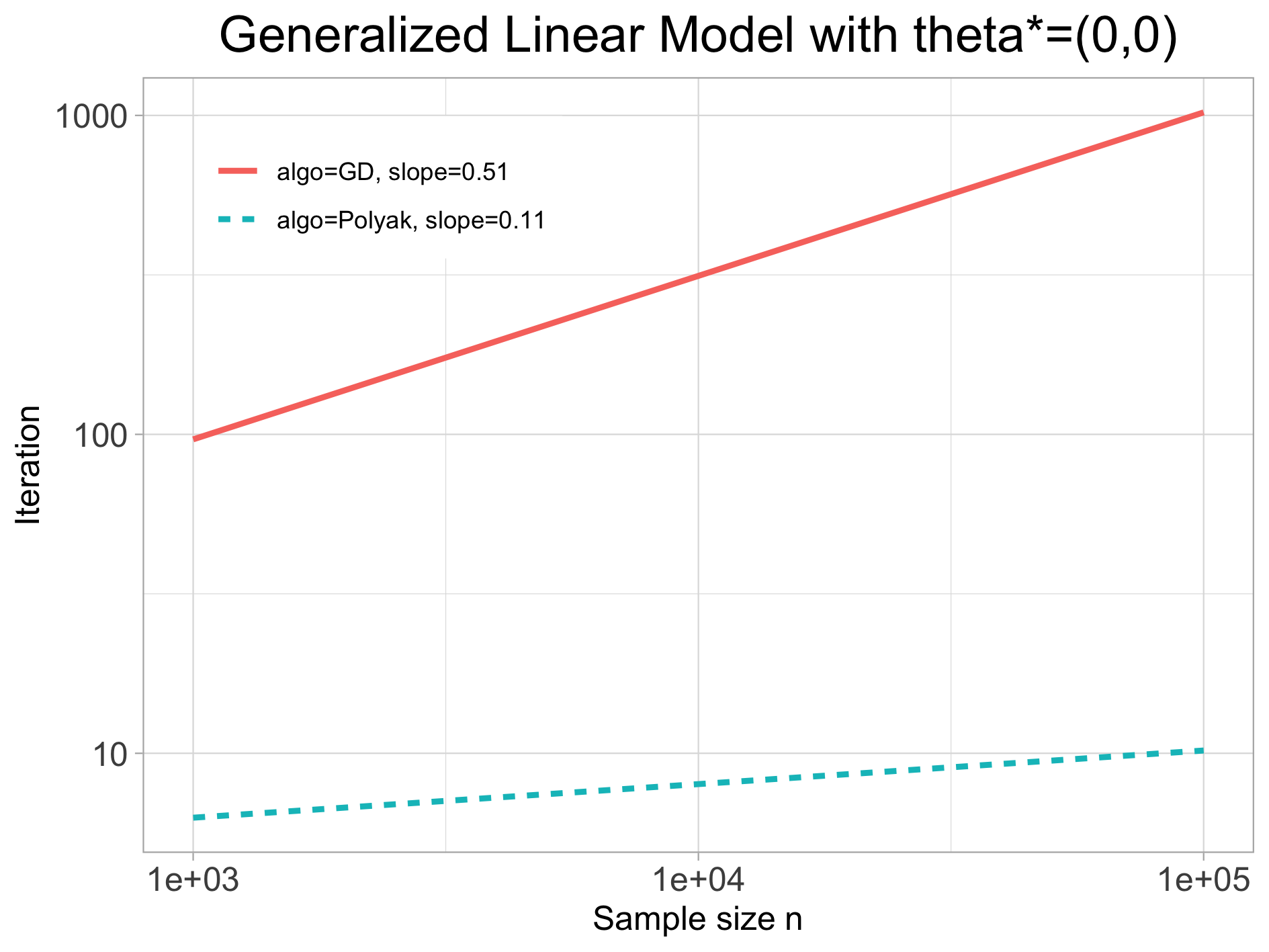}}
\subfigure{
\label{pl:GLM_0_fixiter}
\includegraphics[width=0.45\textwidth]{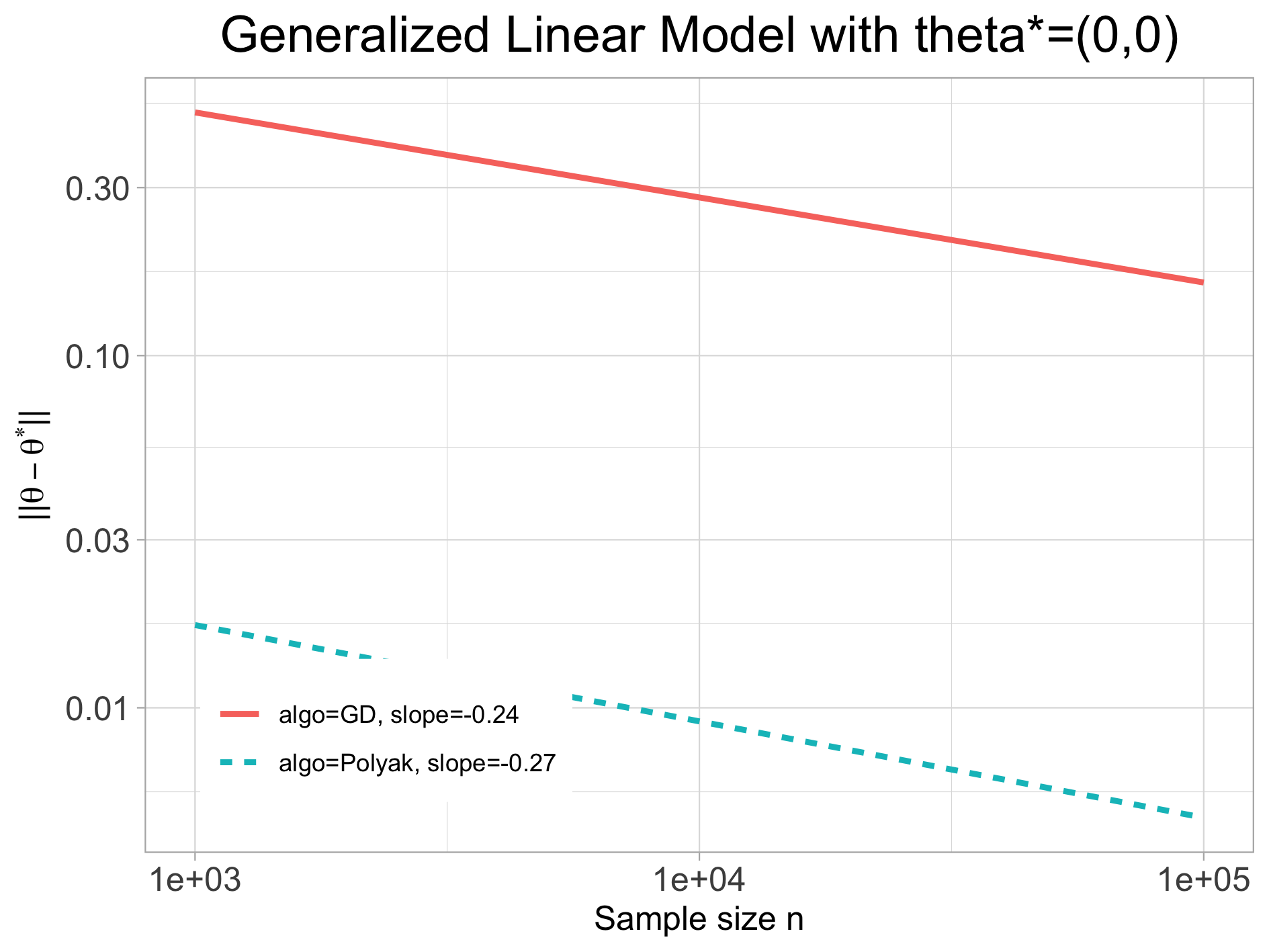}}
\bigskip
\centering
\subfigure{
\label{pl:GLM_n0_fixerror}
\includegraphics[width=0.45\textwidth]{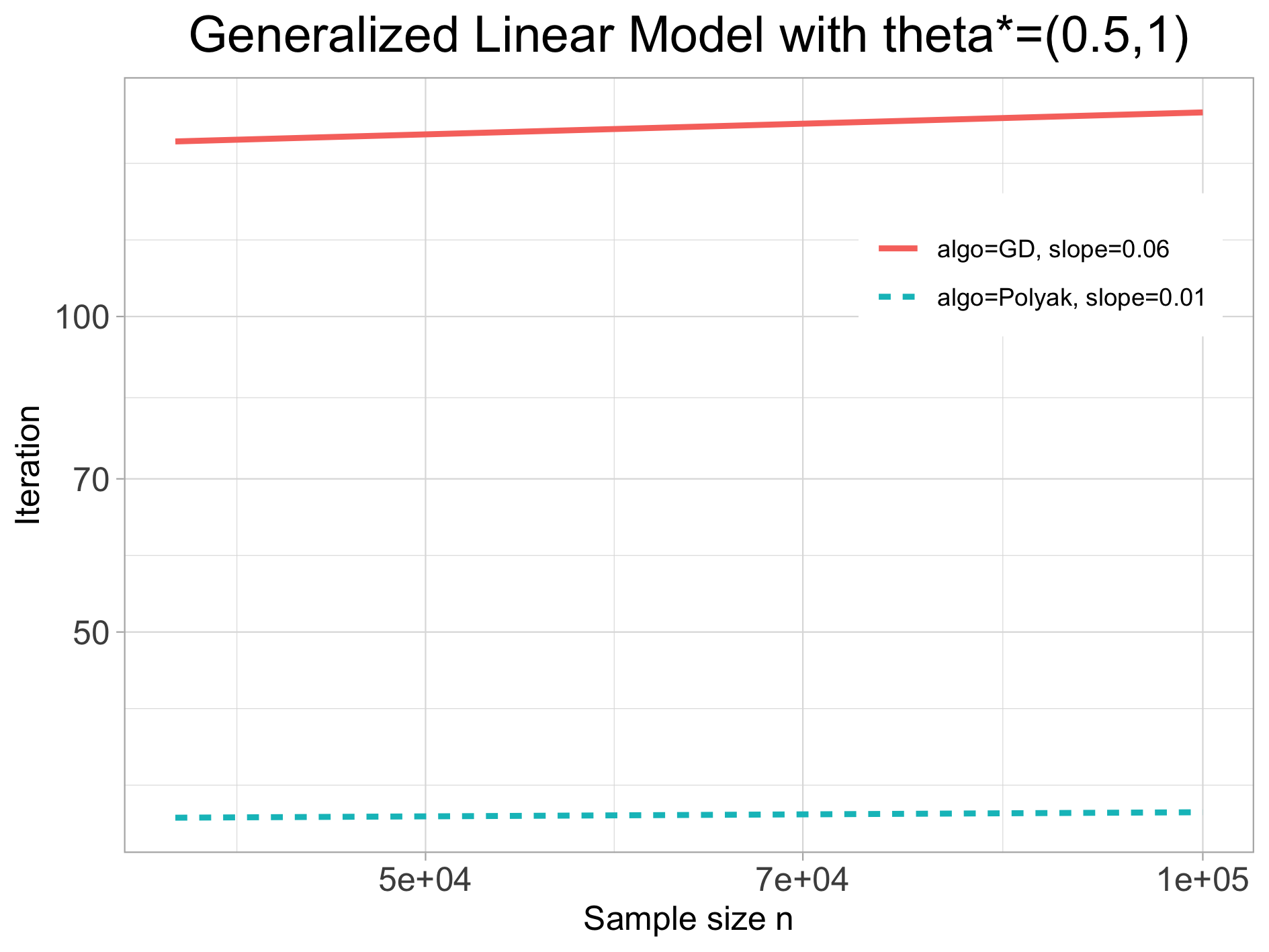}}
\subfigure{
\label{pl:GLM_n0_fixiter}
\includegraphics[width=0.45\textwidth]{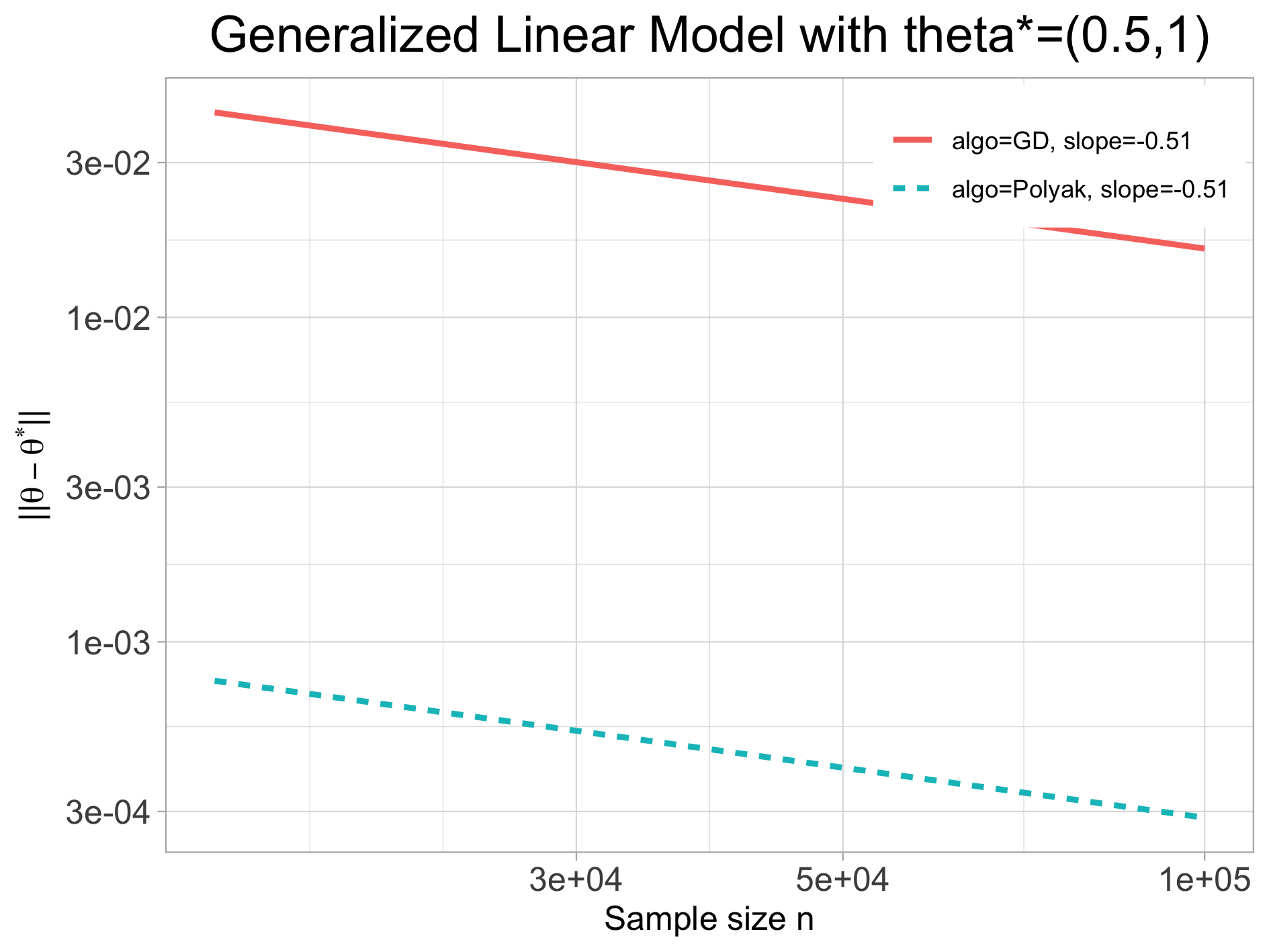}}
\caption{The convergence rates of adaptive Polyak step size gradient descent and fixed-step size gradient descent iterates for solving the sample loss function of the generalized linear model when the link function $g(r) = r^2$. The first row corresponds to the low signal-to-noise regime $\theta^{*} = (0,0)$ while the second row is for the strong signal-to-noise regime $\theta^{*} = (0.5,1)$. For the left images, we use log-log plots to illustrate the iteration complexities of these algorithms to reach the final estimate. For the right images, log-log plots for the final statistical radius versus the sample size are presented. For the low signal-to-noise regime, both the adaptive Polyak step size and fixed-step size gradient descent iterates reach the statistical radius $n^{-1/4}$. The adaptive Polyak step size method takes much fewer number of iterations to reach the final statistical radius than the fixed-step size method, namely, from $\log(n)$ number of iterations of adaptive Polyak step size method to to $\sqrt{n}$ number of iterations of fixed-step size method. For the strong signal-to-noise regime, both adaptive Polyak and fixed-step size methods only take logarithmic number of iterations to reach the statistical radius $n^{-1/2}$.}
\label{pl:GLM_n0}
\end{figure}

Since we do not have access to $\mathcal{L}_n(\widehat{\theta}_n)$ where $\widehat{\theta}_{n}$ is the optimal solution of the sample loss function $\mathcal{L}_{n}$ in equation~\eqref{eq:sample_loss_linear}, we will consider its approximated value according to the adaptive Polyak step size gradient descent algorithm in Algorithm~\ref{algorithm:adaptive_Polyak_v1}. By concentration inequality with the chi-squared random variables, the concentration of $\mathcal{L}_n(\hat{\theta}_n)$ is at the order of $\mathcal{O}(\frac{1}{\sqrt{n}})$ with high probability. Therefore, we use $\frac{c}{\sqrt{n}}$ to approximate $\mathcal{L}_n(\hat{\theta}_n)$, where $c$ here is a parameter to choose in the our experiment. 

\begin{figure}[!t]
\centering
\subfigure{
\label{pl:GMM_0_fixerror}
\includegraphics[width=0.45\textwidth]{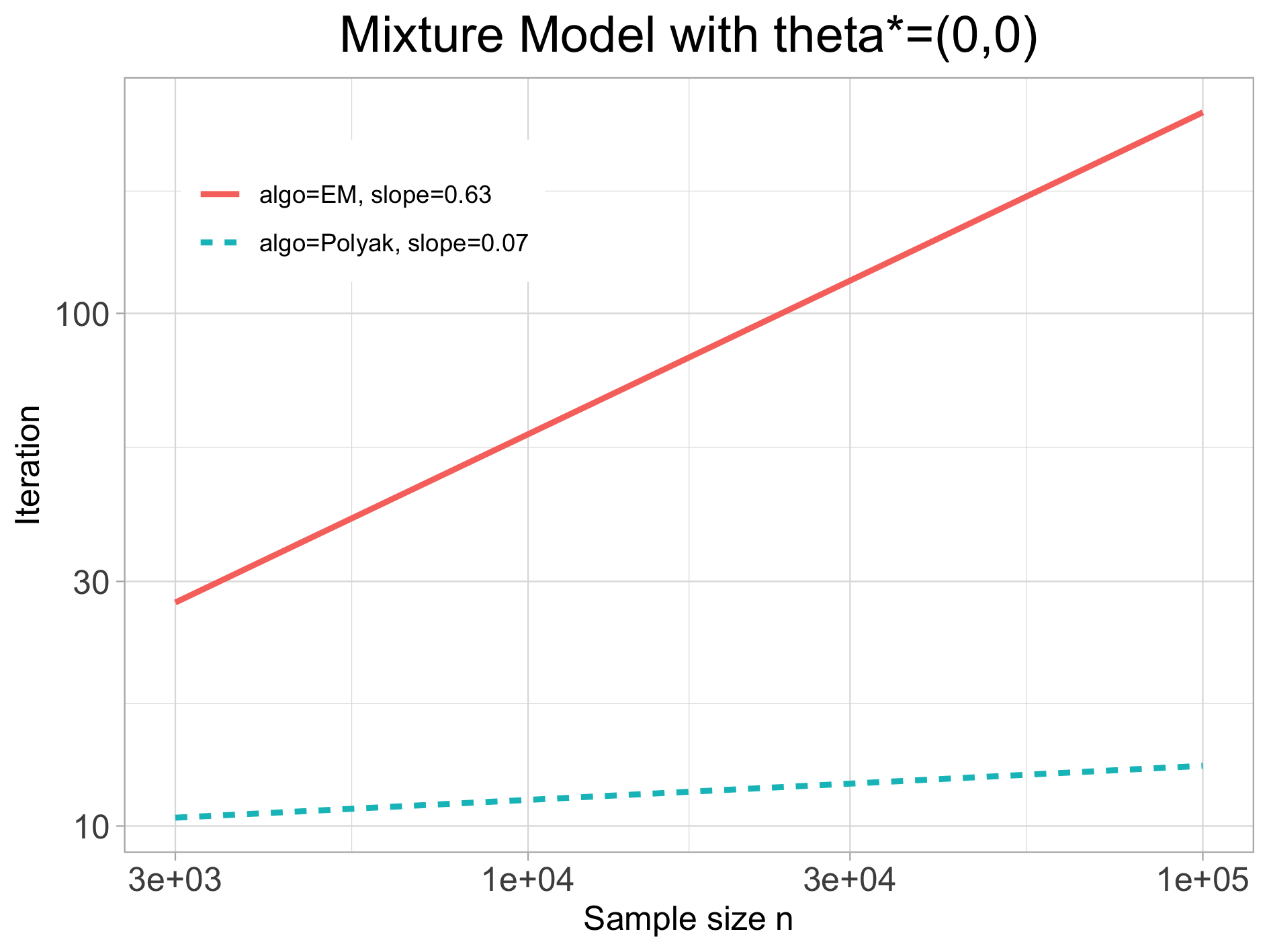}}
\subfigure{
\label{pl:GMM_0_fixiter}
\includegraphics[width=0.45\textwidth]{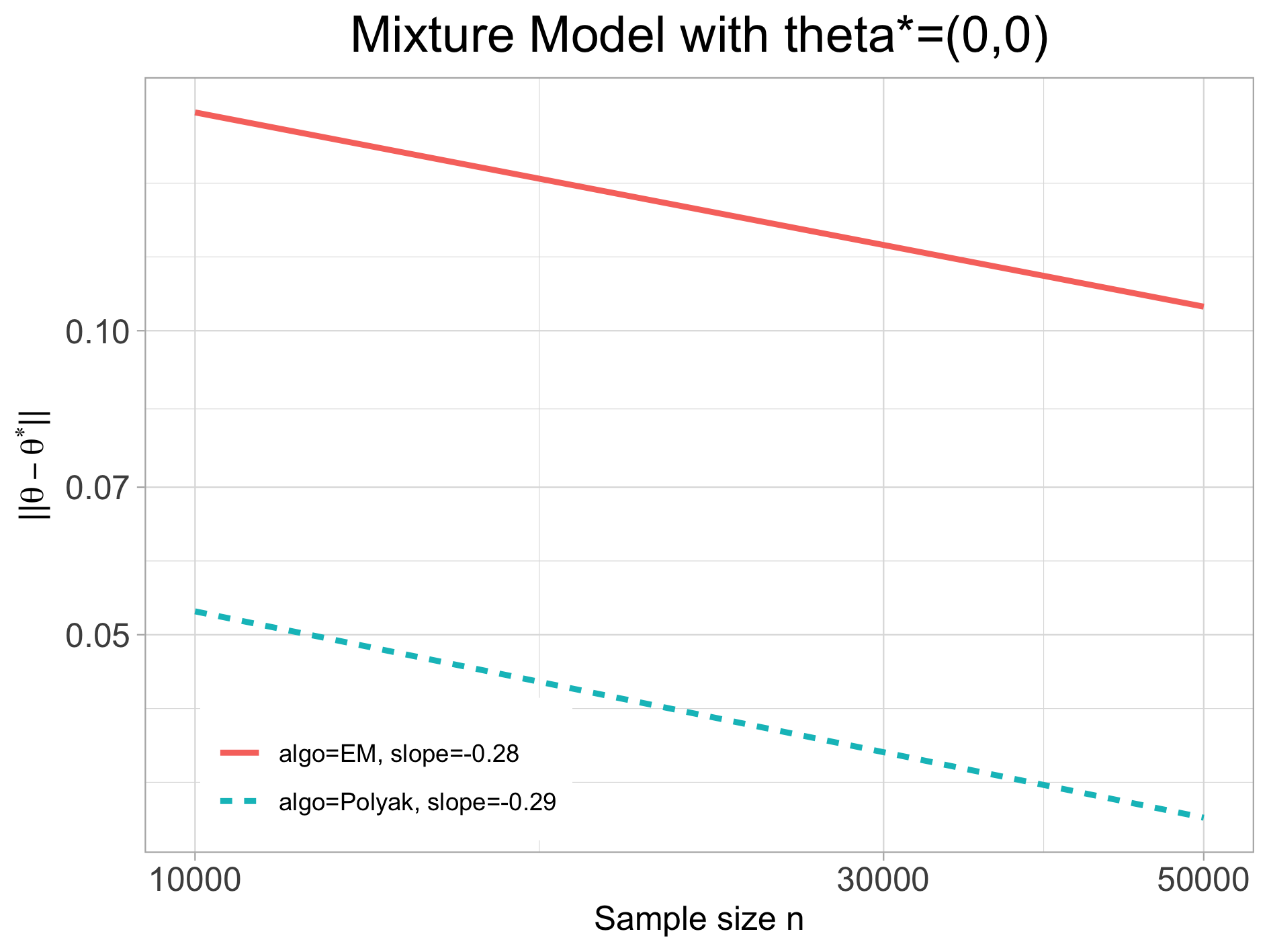}}
\label{pl:GMM_0}
\bigskip
\centering
\subfigure{
\label{pl:GMM_n0_fixerror}
\includegraphics[width=0.45\textwidth]{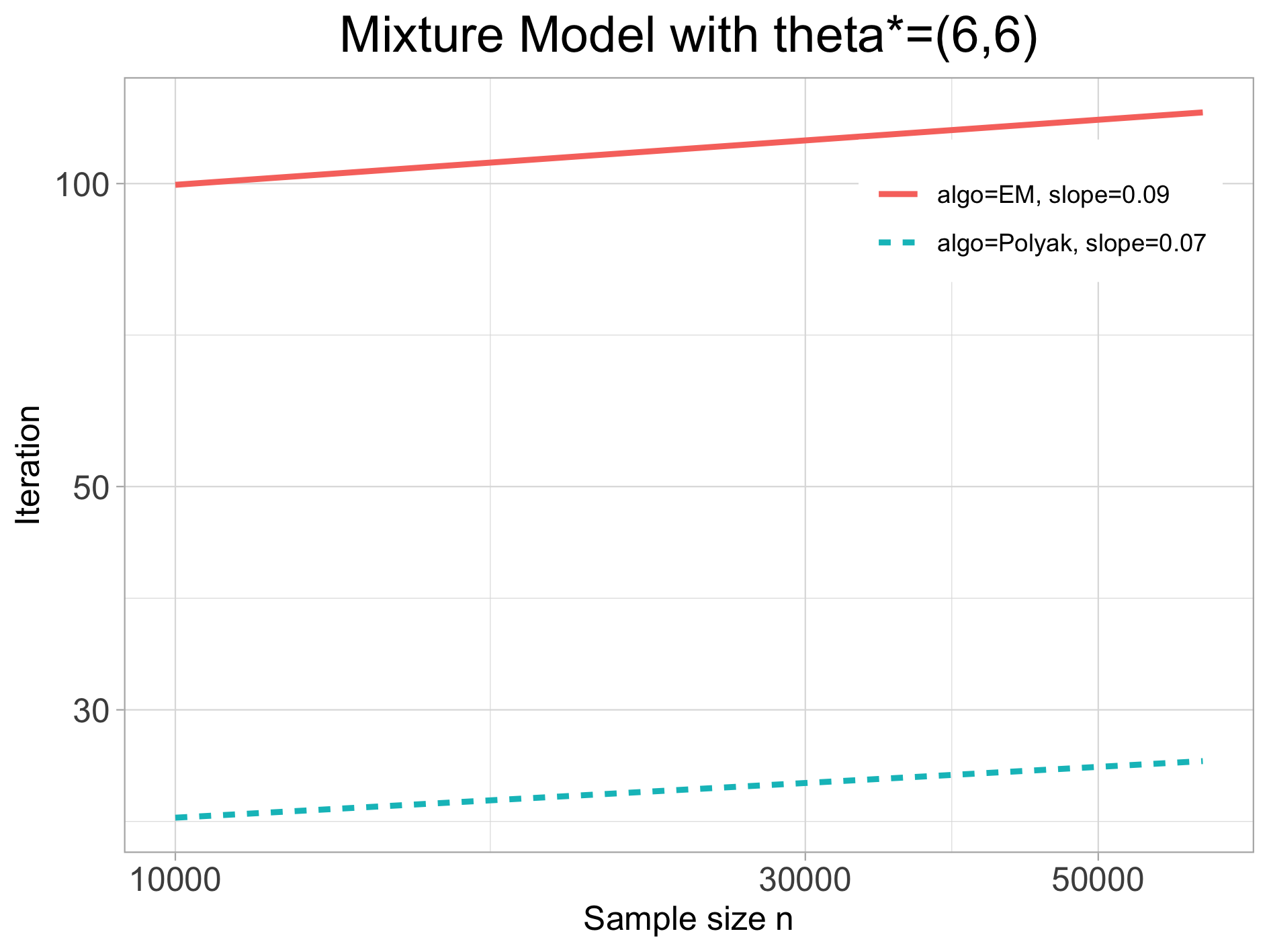}}
\subfigure{
\label{pl:GMM_n0_fixiter}
\includegraphics[width=0.45\textwidth]{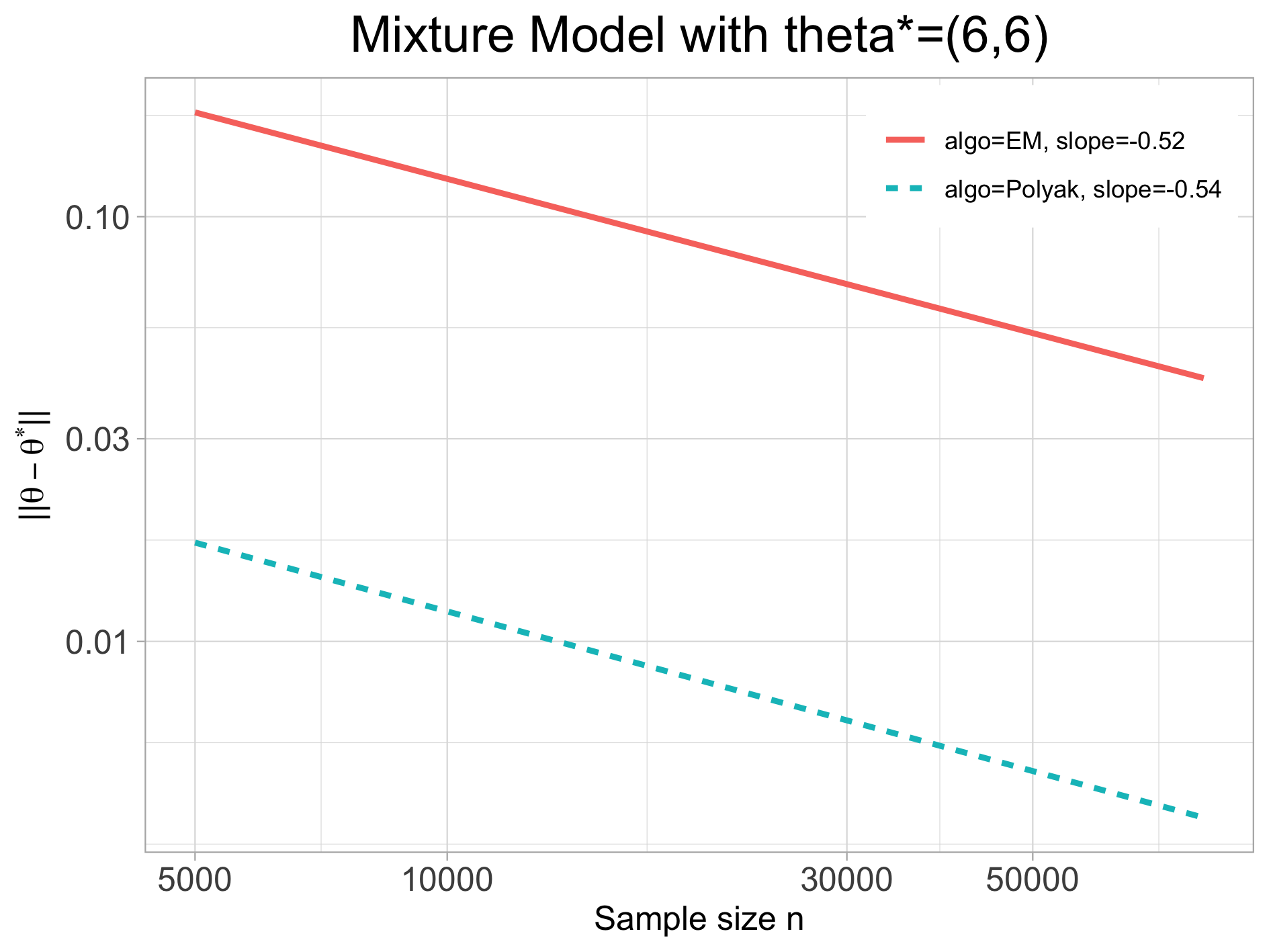}}
\caption{Illustrations for the convergence rates of adaptive Polyak step size and the EM algorithm (equivalently gradient descent algorithm with step size 1) for solving the sample log-likelihood function of the symmetric two-component Gaussian mixtures. The first row corresponds to the low signal-to-noise regime $\theta^{*} = (0,0)$ while the second row is for the strong signal-to-noise regime $\theta^{*} = (6,6)$. The structures of the images are similar to those in Figure~\ref{pl:GLM_n0}. The images in the first row for low signal-to-noise regime show that the adaptive Polyak step size iterates only need roughly $\log(n)$ number of iterations in comparison to $\sqrt{n}$ number of iterations of the EM algorithm to reach the final statistical radius $n^{-1/4}$. The images in the second row for strong signal-to-noise regime show that these optimization methods have similar sample and iteration complexities.}

\label{pl:GMM_n0}
\end{figure}

The updates from the adaptive Polyak step size gradient descent based on that approximation are given by:
$$\theta^{t+1}_n =\theta_n^t-\frac{\mathcal{L}_n(\theta_n^t)-\frac{c}{\sqrt{n}}}{\|\nabla \mathcal{L}_n(\theta_n^t)\|^2}\nabla \mathcal{L}_n(\theta_n^t).$$
When implementing the adaptive Polyak step size gradient descent algorithm, we use binary search to update the value of $c$ periodically. In particular, when the algorithm is stuck at some point, we decrease $c$; when it become very unstable, we increase $c$. For the fixed-step size gradient descent algorithm, we choose the step size to be $0.01$. 

\begin{figure}[!t]
\centering
\subfigure[]{
\label{pl:MLR_0_fixerror}
\includegraphics[width=0.45\textwidth]{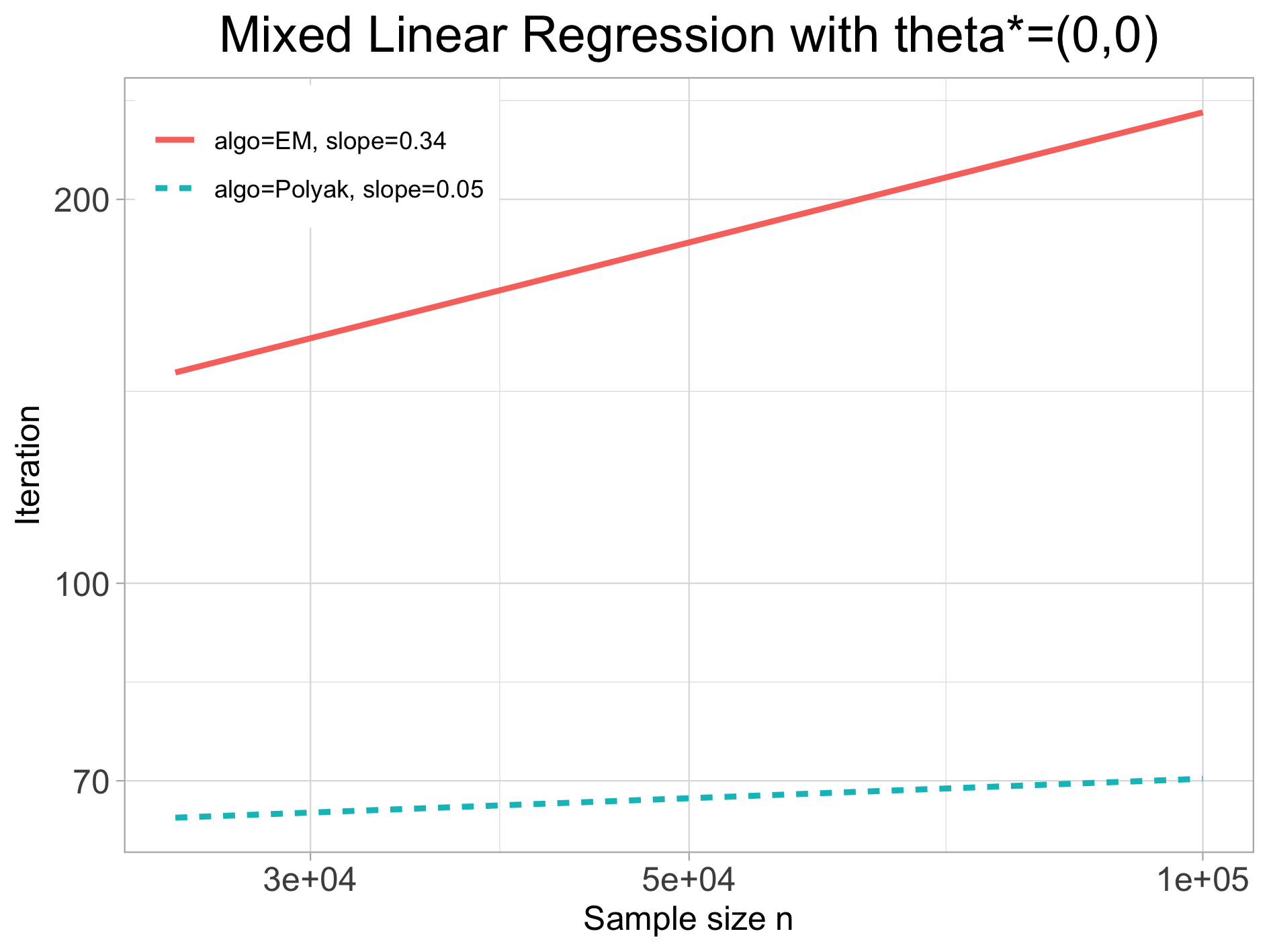}}
\subfigure[]{
\label{pl:MLR_0_fixiter}
\includegraphics[width=0.45\textwidth]{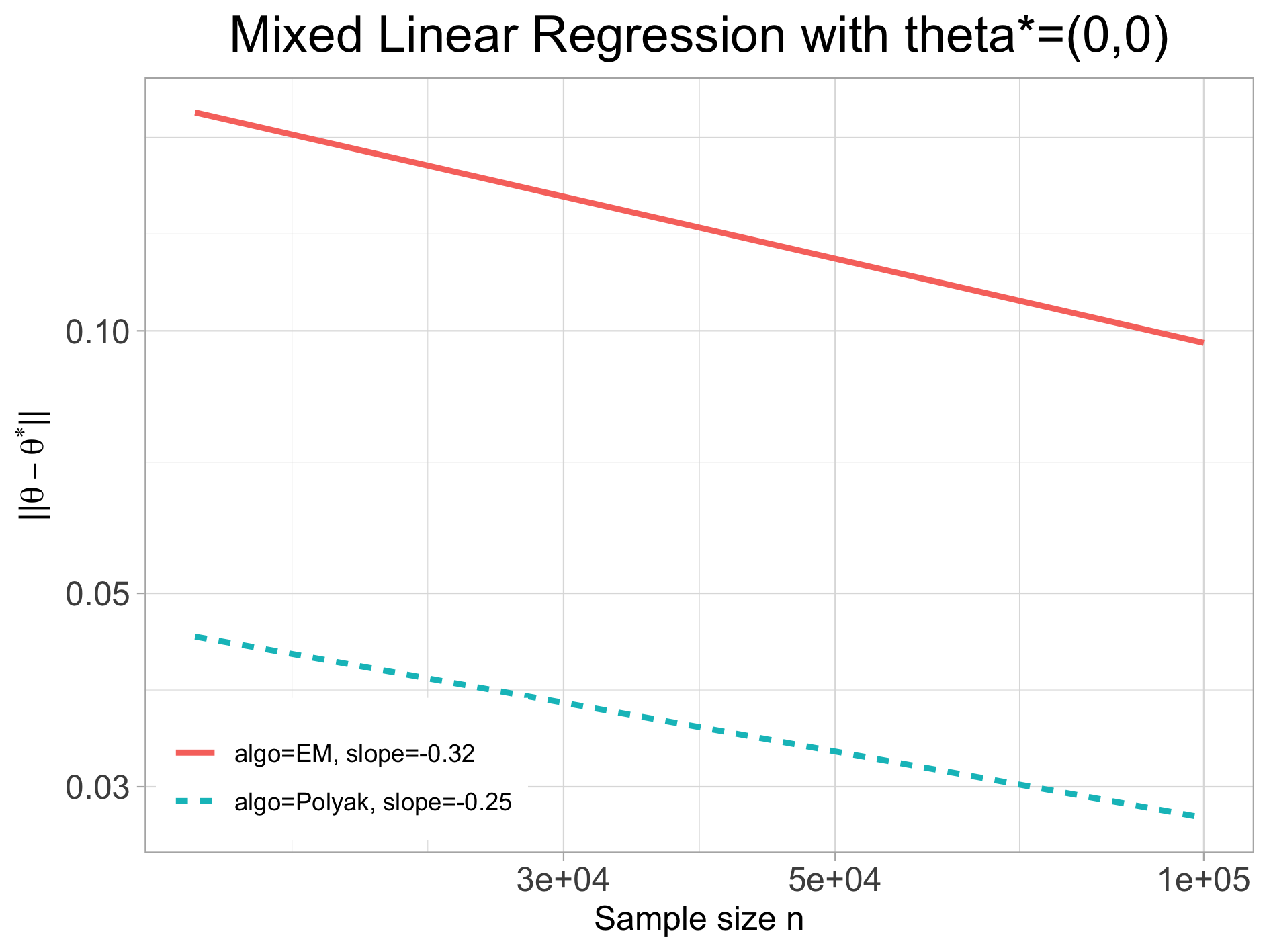}}
\bigskip
\centering
\subfigure[]{
\label{pl:MLR_n0_fixerror}
\includegraphics[width=0.45\textwidth]{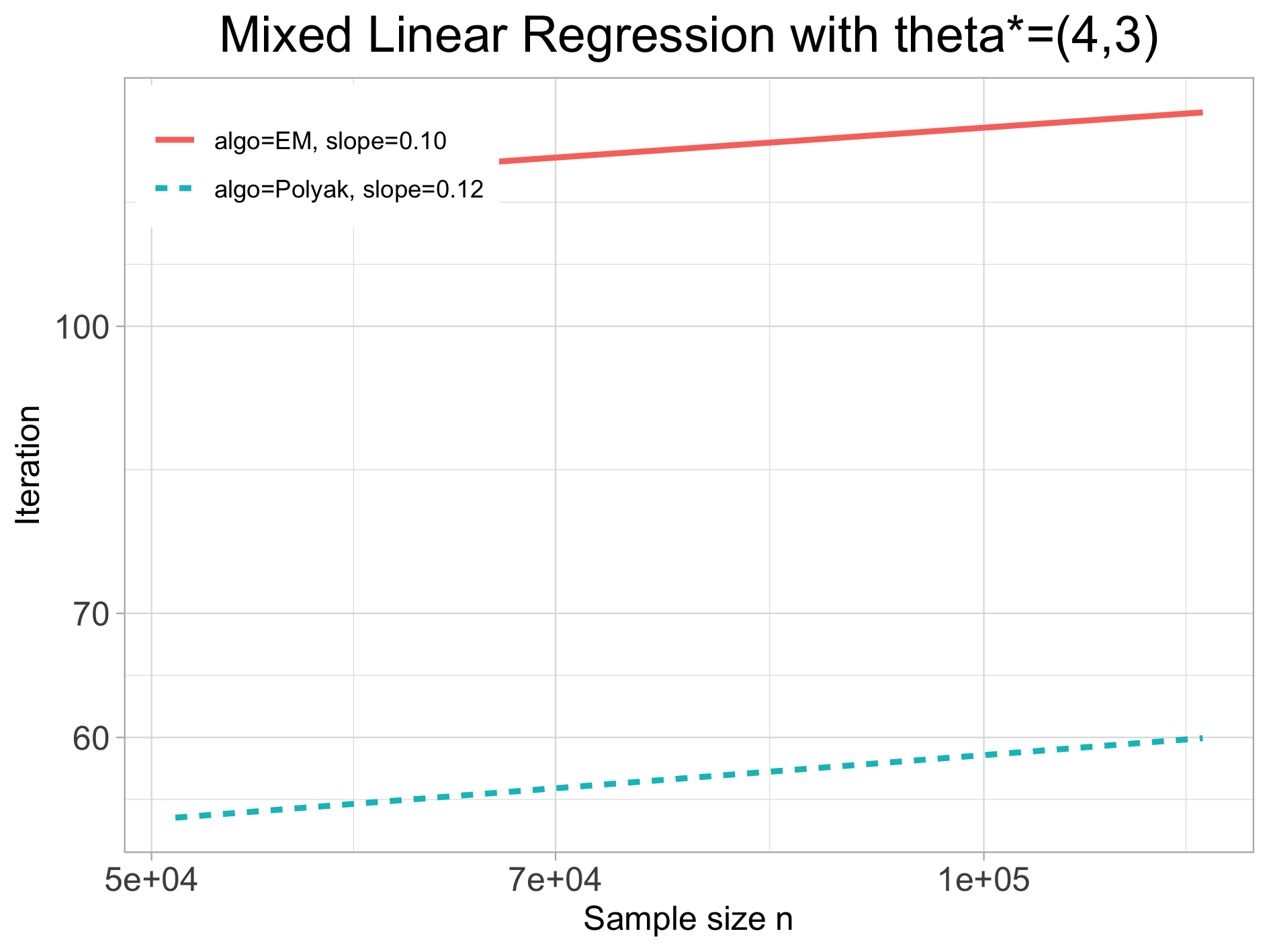}}
\subfigure[]{
\label{pl:MLR_n0_fixiter}
\includegraphics[width=0.45\textwidth]{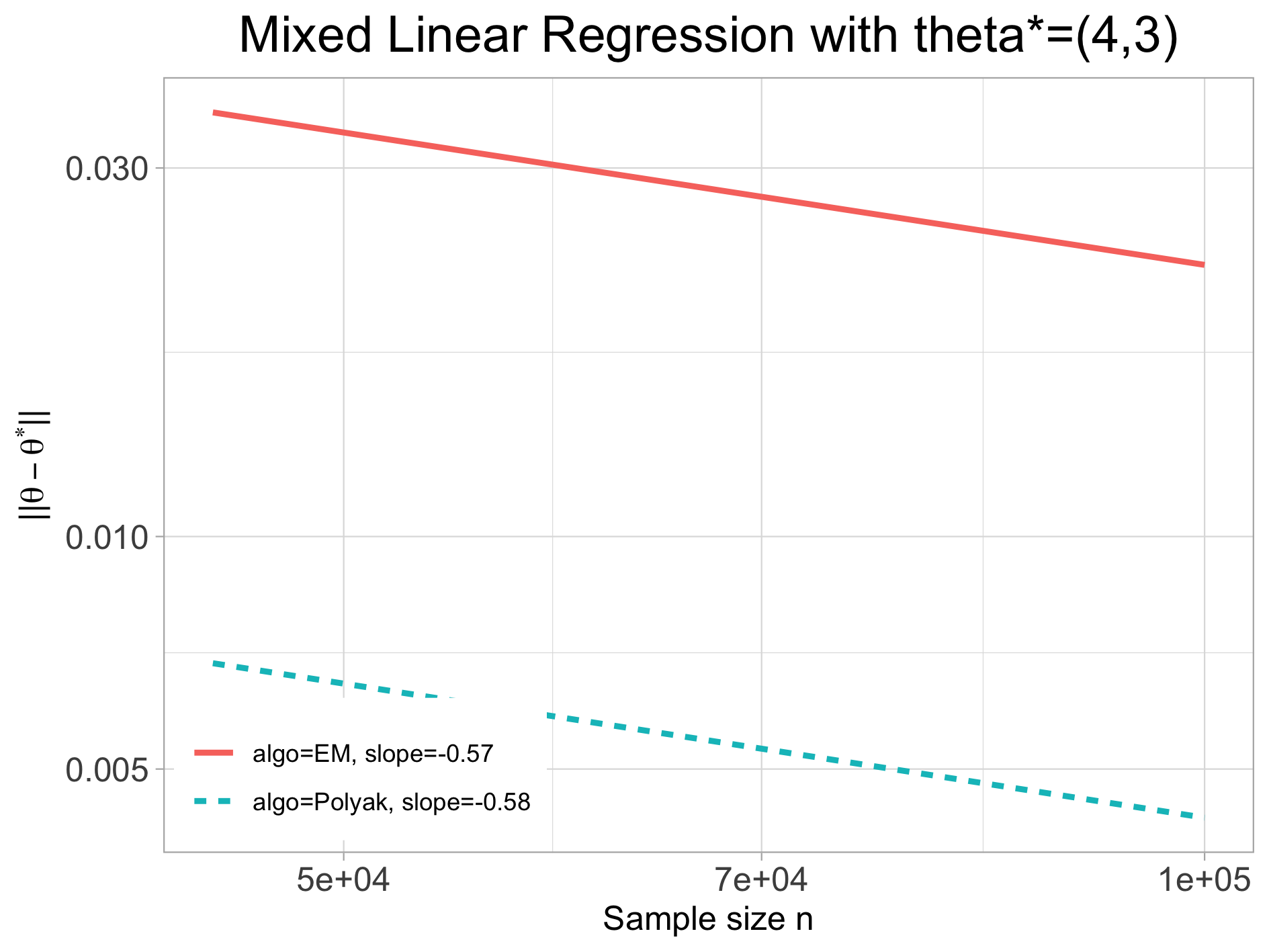}}
\caption{Plots characterizing the convergence rates of adaptive Polyak step size and EM algorithm (equivalently gradient descent algorithm with step size 1) for solving the sample log-likelihood function of the symmetric two-component mixed linear regression model. The first row corresponds to the low signal-to-noise regime $\theta^{*} = (0,0)$ while the second row is for the strong signal-to-noise regime $\theta^{*} = (4,3)$. From the images in the first row for low signal-to-noise regime, the iteration complexity of the adaptive Polyak step size method is roughly $\log(n)$ while that of the EM algorithm scales like $\sqrt{n}$ to reach the final statistical radius $n^{-1/4}$. From the images in the second row for strong signal-to-noise regime, both these optimization algorithms have sample complexity $n^{-1/2}$ and iteration complexity $\log(n)$.}
\label{pl:MLR_n0}
\end{figure}

The experiment results are shown in Figure~\ref{pl:GLM_n0}. For the left images in that figure, we use log-log plot to illustrate the iteration complexity of the adaptive Polyak step size and fixed-step size gradient descent algorithms versus the sample size under the low signal-to-noise setting $\theta^{*} = (0, 0)$ (first row) and the strong signal-to-noise setting $\theta^{*} = (0.5, 1)$ (second row). When $\theta^{*} = (0, 0)$, we observe that the number of iterations for the fixed-step size gradient descent algorithm to reach the final statistical radius is at the order close to $\sqrt{n}$ while the iteration complexity for the adaptive Polyak step size gradient descent algorithm is roughly $\log(n)$. On the other hand, when $\theta^{*} = (0.5, 1)$, both the iteration complexities of these algorithms scale like $\log (n)$. For the right images in Figure~\ref{pl:GLM_n0}, we plot the final statistical radii of the adaptive Polyak and fixed-step size gradient descent iterates versus the sample size under different settings of $\theta^{*}$. As being indicated in these images, the radius scales like $n^{-1/4}$ when $\theta^{*} = (0, 0)$ while it is roughly $n^{-1/2}$ when $\theta^{*} = (0.5, 1)$. These results, along with our comments about the adaptive Polyak step size gradient descent algorithm after Theorem~\ref{theorem:convergence_rate_Polyak}, confirm our theories in Section~\ref{sec:example_glm}.



\vspace{0.5 em}
\noindent
\textbf{Mixture model:}
\label{sec:over_mixture_model_exp}
We now move to the symmetric two-component Gaussian mixture model considered in Section~\ref{sec:over_mixture_model}.  We set dimension $d=2$, the variance $\sigma=1, \theta^*=(0,0)$ for the low signal-to-noise regime and $\theta^{*} = (6,6)$ for the strong signal-to-noise regime. To obtain an estimation of $\theta^{*}$, we maximize the log-likelihood in equation~\eqref{eq:sample_loglikelihood}. We use $\frac{c}{n}$ to approximate the optimal value of sample log-likelihood function $\bar{\mathcal{L}}_n$ where $c$ is some universal constant. We also use binary search to adaptively update the value of the constant $c$ when we run the adaptive Polyak step size algorithm. We compare the performance of that algorithm to the EM algorithm (equivalently gradient descent algorithm with step size 1) in Figure~\ref{pl:GMM_n0}. When $\theta^{*} = (0, 0)$, the images in the first row of Figure~\ref{pl:GMM_n0} show that the adaptive Polyak step size iterates only need roughly $\log(n)$ number of iterations in comparison to $\sqrt{n}$ number of iterations of the EM algorithm to reach the final statistical radius $n^{-1/4}$. When $\theta^{*} = (6, 6)$, the images in the second row for strong signal-to-noise regime show that these optimization methods have similar sample complexities $n^{-1/2}$ and iteration complexities $\log(n)$. These experiment results prove that the adaptive Polyak step size gradient descent algorithm is computationally more efficient than the EM algorithm to reach the final estimate for the low signal-to-noise regime, which confirms out theories in Section~\ref{sec:over_mixture_model}.

\vspace{0.5 em}
\noindent
\textbf{Mixed linear regression:}
\label{sec:mixture_model_exp}
Finally, we consider the two-component mixed linear regression example in Section~\ref{sec:mix_linear}. We consider $\theta^*=( 0,0)$ for the low signal-to-noise regime and $\theta^{*} = (4,3)$ for the strong signal-to-noise regime. We choose the variance $\sigma=1$ in model~\eqref{eq:mixed_linear}. Our goal is to maximize the log-likelihood in equation~\eqref{eq:sample_loglikelihood_mixed_linear}. Similar to the two-component Gaussian mixture model, we use use $\frac{c}{n}$ to approximate the optimal value of $\widetilde{\mathcal{L}}_n$ in the adaptive Polyak step size gradient descent method where $c$ is adaptively updated via the binary search. We compare the adaptive Polyak step size algorithm to the EM algorithm (equivalently gradient descent algorithm with step size 1) in Figure~\ref{pl:MLR_n0}. When $\theta^{*} = (4,3)$, both optimization algorithms reach the final statistical radius $n^{-1/2}$ around $\theta^{*}$ after $\log(n)$ number of iterations. When $\theta^{*} = (0, 0)$, the adaptive Polyak step size iterates reach the statistical radius $n^{-1/4}$ after $\log(n)$ number of iterations while the EM algorithm needs roughly $\sqrt{n}$ number of iterations to reach the same radius. These observations are consistent with our theories in Section~\ref{sec:mix_linear}.


\section{Proofs}
\label{sec:proofs}
In this section, we provide the proofs for main results in Section~\ref{sec:convergence_rate_Polyak}.
\subsection{Proof of Lemma~\ref{lincon}}
\label{subsec:proof:lincon}
First, we notice that
\begin{align*}
    \|\theta^{t+1} - \theta^*\|^2 - \|\theta^t - \theta^*\|^2
    & = \frac{(f(\theta^t) - f(\theta^*))^2}{\|\nabla f(\theta^t)\|^2} - \frac{2(f(\theta^t) - f(\theta^*))}{\|\nabla f(\theta^t)\|^2} \langle \nabla f(\theta^t), \theta^t - \theta^*\rangle \\
    & = \frac{f(\theta^t) - f(\theta^*)}{\|\nabla f(\theta^t)\|^2} \left(f(\theta^t) - f(\theta^*) - 2 \langle \nabla f(\theta^t), \theta^t - \theta^*\rangle\right) \\
    & \leq - \frac{(f(\theta^t) - f(\theta^*))^2}{\|\nabla f(\theta^t)\|^2} \leq 0
\end{align*}
where the inequality is due to the convexity of the population loss function $f$. This result indicates that the sequence $\{\|\theta^t - \theta^*\|\}_{t \geq 0}$ is monotonically decreasing and thus $\theta^t \in \mathbb{B}(\theta^*, \rho)$ for all $t \geq 0$ as long as $\theta^1 \in \mathbb{B}(\theta^*, \rho)$. Furthermore, we find that
\begin{align*}
    \|\theta^{t+1} - \theta^*\|^2 - \|\theta^t - \theta^*\|^2
    \leq  - \frac{(f(\theta^t) - f(\theta^*))^2}{\|\nabla f(\theta^t)\|^2}
    \leq & - \frac{f(\theta^t) - f(\theta^*)}{2c_1\|\theta^t - \theta^*\|^\alpha}\\
    \leq & - \frac{c_2^{\alpha + 2}}{2c_1 (\alpha + 2)^{\alpha + 2}}\|\theta^t - \theta^*\|^2,
\end{align*}
where the second inequality is based on the fact that $f(\theta^t) - f(\theta^*)\geq  \frac{\|\nabla f(\theta^t)\|^2}{2c_1\|\theta^t-\theta^*\|^{\alpha}}$ which can be recovered from Lemma 3.5 in \citep{bubeck2015convex} (also stated in Lemma~\ref{lem:smooth_update}) and Assumption~\ref{assump:smoothness}, and the third inequality is from Lemma \ref{lem:obj_to_params}. The above inequality is equivalent to
\begin{align}
    \|\theta^{t+1} - \theta^*\|^2 \leq \parenth{1 - \frac{c_2^{\alpha + 2}}{2c_1 (\alpha + 2)^{\alpha + 2}}}\|\theta^t - \theta^*\|^2. \label{eq:contraction_Polyak_population}
\end{align}
We can further see that for any $\theta \in \mathbb{B}(\theta^{*}, \rho)$
\begin{align*}
    \|\theta - \theta^*\| \leq \frac{\alpha + 2}{c_2}(f(\theta) - f(\theta^*))^{\frac{1}{\alpha + 2}}\leq \frac{\alpha + 2}{c_2} \cdot\left(\frac{c_1}{2}\right)^{\frac{1}{\alpha + 2}}\|\theta  - \theta^*\|,
\end{align*}
which means $\left(\frac{c_2}{\alpha + 2}\right)^{\alpha + 2} \leq \frac{c_1}{2}$ and $\frac{c_2^{\alpha + 2}}{2c_1 (\alpha + 2)^{\alpha + 2}} \leq \frac{1}{4}$.
Thus, the contraction coefficient $\frac{3}{4} \leq 1 - \frac{c_2^{\alpha + 2}}{2c_1 (\alpha + 2)^{\alpha + 2}} < 1$, which means that it is positive and strictly less than $1$. By repeating the inequality~\eqref{eq:contraction_Polyak_population}, we eventually have the following inequality:
\begin{align*}
    \|\theta^{t+1} - \theta^*\|^2 \leq \left(1 - \frac{c_2^{\alpha + 2}}{2c_1 (\alpha + 2)^{\alpha + 2}}\right)^{t} \|\theta^0 - \theta^*\|^2.
\end{align*}
As a consequence, we reach the conclusion of Lemma~\ref{lincon}.

\subsection{Proof of Lemma~\ref{supdif}}
\label{subsec:supdif}
Recall that, from Assumptions~\ref{assump:smoothness}, ~\ref{assump:nonPL} and Lemma~\ref{lem:obj_to_params}, as long as $\theta \in \mathbb{B}(\theta^{*}, \rho)$ we have the following relations:
\begin{align}
    f(\theta) - f(\theta)^* \leq & \frac{c_1}{2}\|\theta - \theta^*\|^{\alpha + 2}, \label{eq:key_inequality_concentration_Polyak_first} \\
    c_2 \left(\frac{c_2}{\alpha + 2} \|\theta - \theta^*\|\right)^{\alpha + 1} \leq & \|\nabla f(\theta)\| \leq c_1 \|\theta - \theta^*\|^{\alpha + 1}. \label{eq:key_inequality_concentration_Polyak_second} 
\end{align}
From the definitions of the population and sample Polyak operators $F_{n}$ and $F$ in equations~\eqref{eq:sample_operator_Polyak} and~\eqref{eq:population_operator_Polyak}, we make the following decomposition on $\|F_n(\theta) - F(\theta)\|$:
\begin{align}
     \|F_n(\theta) - F(\theta)\|= & \left\|\frac{f_n(\theta) - f_n(\widehat{\theta}_{n})}{\|\nabla f_n(\theta)\|^2} \nabla f_n(\theta) - \frac{f(\theta) - f(\theta^*)}{\|\nabla f(\theta)\|^2}\nabla f(\theta)\right\| \nonumber \\
     \leq & \left\|\left(\frac{f_n(\theta) - f_n(\widehat{\theta}_{n})}{\|\nabla f_n(\theta)\|^2}  - \frac{f(\theta) - f(\theta^*)}{\|\nabla f(\theta)\|^2}\right)\nabla f_n(\theta)\right\| \nonumber \\
     & \hspace{9 em} + \left\|\frac{f(\theta) - f(\theta^*)}{\|\nabla f(\theta)\|^2}(\nabla f(\theta) - \nabla f_n(\theta))\right\| \nonumber \\
     : = & T_{1} + T_{2}. \label{eq:key_bound_concentration_Polyak}
\end{align}
\paragraph{Upper bound on $T_{2}$:} We first deal with the second term $T_{2}$. With Assumption~\ref{assump:stab}, with probability $1 - \delta$ we have that
\begin{align*}
    \left\|\frac{f(\theta) - f(\theta^*)}{\|\nabla f(\theta)\|^2}(\nabla f(\theta) - \nabla f_n(\theta))\right\|
    = & \frac{f(\theta) - f(\theta^*)}{\|\nabla f(\theta)\|^2} \|\nabla f(\theta) - \nabla f_n(\theta)\| \\
    \leq & \frac{c_3 r^\gamma \varepsilon(n, \delta)}{c_2^2}(f(\theta) - f(\theta^*))^{\frac{2}{\alpha + 2} - 1}\\
\end{align*}
for all $\theta \in \mathbb{B}(\theta^{*}, r) \backslash \mathbb{B}(\theta^{*}, r_{n})$ where $r < \rho$. Combining the above inequality with the inequality~\eqref{eq:key_inequality_concentration_Polyak_first}, we obtain
\begin{align}
    \left\|\frac{f(\theta) - f(\theta^*)}{\|\nabla f(\theta)\|^2}(\nabla f(\theta) - \nabla f_n(\theta))\right\| \leq \frac{c_3 }{c_2^2} \cdot \left(\frac{c_1}{2}\right)^{\frac{2}{\alpha + 2} - 1} r^{\gamma-\alpha} \varepsilon(n, \delta) \label{eq:upper_bound_T2}
\end{align}
for all $\theta \in \mathbb{B}(\theta^{*}, r) \backslash \mathbb{B}(\theta^{*}, r_{n})$ where $r < \rho$. 
\paragraph{Upper bound on $T_{1}$:} For the first term $T_{1}$, we have that
\begin{align}
    & \left\|\left(\frac{f_n(\theta) - f_n(\widehat{\theta}_n)}{\|\nabla f_n(\theta)\|^2}  - \frac{f(\theta) - f(\theta^*)}{\|\nabla f(\theta)\|^2}\right)\nabla f_n(\theta)\right\| \nonumber \\
    & \hspace{3 em} \leq  \frac{\left|(f_n(\theta) - f_n(\widehat{\theta}_n))\|\nabla f(\theta)\|^2 - (f(\theta) - f(\theta^*))\|\nabla f_n(\theta)\|^2\right|}{\|\nabla f_n(\theta)\|\|\nabla f(\theta)\|^2} \nonumber \\
    & \hspace{3 em} \leq  \frac{\left|(f_n(\theta) - f_n(\widehat{\theta}_n) - f(\theta) + f(\theta^*))\|\nabla f(\theta)\|^2\right| + (f(\theta) - f(\theta^*))\left|\|\nabla f_n(\theta)\|^2 - \|\nabla f(\theta)\|^2\right|}{(\|\nabla f(\theta)\| - c_{3} r^\gamma \varepsilon(n, \delta))\|\nabla f(\theta)\|^2} \nonumber \\
    & \hspace{3 em} \leq \frac{|f_n(\theta) - f_n(\widehat{\theta}_n) - f(\theta) + f(\theta^*)|\|\nabla f(\theta)\|^2 + 2(f(\theta) - f(\theta^*))\|\nabla f(\theta)\|c_3r^\gamma \varepsilon(n, \delta)}{(\|\nabla f(\theta)\| - c_{3} r^\gamma \varepsilon(n, \delta))\|\nabla f(\theta)\|^2} \nonumber \\
    & \hspace{20 em} + \frac{(f(\theta) - f(\theta^*))c_3^2 r^{2\gamma}\varepsilon^2(n, \delta)}{(\|\nabla f(\theta)\| - c_{3} r^\gamma \varepsilon(n, \delta))\|\nabla f(\theta)\|^2}. \label{eq:bound_T1_first}
\end{align}
To bound the RHS of equation~\eqref{eq:bound_T1_first}, we need to upper bound
\begin{align*}
    & \hspace{-5 em} f_n(\theta) - f(\theta) - (f_n(\widehat{\theta}_n) - f(\theta^*))\\
    & = f_n(\theta) - f_n(\theta^*) - (f(\theta) - f(\theta^*)) - (f_n(\widehat{\theta}_n) - f_n(\theta^*))
\end{align*}
Indeed, from Assumption \ref{assump:stab}, with probability $1 - \delta$ we have that 
\begin{align}
    & \hspace{-5 em} f_n(\theta) - f_n(\theta^*) - (f(\theta) - f(\theta^*)) \nonumber\\
    & \leq  \int_{0}^{1} \|\nabla f_n(\theta^* + t(\theta - \theta^*)) - \nabla f(\theta^* + t(\theta - \theta^*))\| d t \nonumber\\
    & \leq \frac{c_3 r^{\gamma + 1} \varepsilon(n, \delta)}{\gamma + 1} \label{eq:bound_T1_second}
\end{align}
for any $\theta \in \mathbb{B}(\theta^{*}, r)$. Furthermore, we find that
\begin{align}
    |f_n(\widehat{\theta}_n) - f_n(\theta^*)| & \leq |f_n(\widehat{\theta}_n) - f(\widehat{\theta}_n) - f_n(\theta^*) + f(\theta^*)| + |f(\widehat{\theta}_n) - f(\theta^*)| \nonumber\\
    & \leq \frac{c_3 r^{\gamma + 1} \varepsilon(n, \delta)}{\gamma + 1} + \frac{c_1 \|\widehat{\theta}_n - \theta^*\|^{\alpha + 2}}{2}, \label{eq:bound_T1_third}
\end{align}
where the final inequality is due to inequalities~\eqref{eq:bound_T1_second} and~\eqref{eq:key_inequality_concentration_Polyak_first}.
Plugging the bounds~\eqref{eq:bound_T1_second} and~\eqref{eq:bound_T1_third} into~\eqref{eq:bound_T1_first}, we find that
\begin{align*}
    & \left\|\left(\frac{f_n(\theta) - f_n(\theta_n^*)}{\|\nabla f_n(\theta)\|^2}  - \frac{f(\theta) - f(\theta^*)}{\|\nabla f(\theta)\|^2}\right)\nabla f_n(\theta)\right\|\\
    & \hspace{4 em} \leq  \frac{\left(\frac{2c_3 r^{\gamma + 1}\varepsilon(n, \delta)}{\gamma + 1} + \frac{c_1 \|\widehat{\theta}_n - \theta^*\|^{\alpha + 2}}{2}\right) c_1^2 r^{2\alpha + 2} + c_1^2 c_3 r^{2\alpha + 3 + \gamma} \varepsilon(n, \delta) + \frac{c_1 c_3^2}{2} r^{2\gamma + \alpha + 2} \varepsilon^2(n, \delta)}{c_2^2 \left(\frac{c_2 r}{(\alpha + 2)} \right)^{2\alpha + 2}\left(c_2\left(\frac{c_2 r}{(\alpha + 2)} \right)^{\alpha +1} - c_3 r^{\gamma} \varepsilon(n, \delta)\right)}.
\end{align*}
As $r \geq \bar{C}\varepsilon(n, \delta)^{\frac{1}{\alpha + 1 - \gamma}}$ where $\bar{C} = \left(\frac{C \cdot c_3 (\alpha + 2)^{\alpha + 1}}{c_2^{\alpha + 2}}\right)^{\frac{1}{\alpha + 1 - \gamma}}$
, we know $c_3 r^{\gamma}\varepsilon(n, \delta) \leq \frac{c_2}{C}\left(\frac{c_2 r}{\alpha + 2}\right)^{\alpha + 1}$, and we can simplify this term to
\begin{align}
    & \left\|\left(\frac{f_n(\theta) - f_n(\theta_n^*)}{\|\nabla f_n(\theta)\|^2}  - \frac{f(\theta) - f(\theta^*)}{\|\nabla f(\theta)\|^2}\right)\nabla f_n(\theta)\right\|\nonumber\\
    \leq & \frac{C}{C-1} \left( \frac{2c_1^2 c_3 (\alpha + 2)^{3\alpha + 3}}{(\gamma + 1)c_2 ^{3\alpha + 6}} r^{\gamma - \alpha} \varepsilon(n, \delta) + \frac{c_1^3(\alpha + 2)^{3\alpha + 3}}{2 (\gamma + 1) c_2^{3\alpha + 6}} \left(\frac{C c_3 (\alpha + 2)^{3\alpha + 3}}{c_2^{\alpha + 2}}\varepsilon(n, \delta)\right)^{\frac{\alpha + 2}{\alpha + 1 - \gamma}} r^{-\alpha - 1}\right.\nonumber\\
    & + \left.\frac{c_1^2 c_3 (\alpha + 2)^{3\alpha + 3}}{c_2^{3\alpha + 6}} r^{\gamma - \alpha} \varepsilon(n, \delta) + \frac{c_1 c_3^2 (\alpha + 2)^{3\alpha + 3}}{2 c_2^{3\alpha + 6}} r^{2\gamma - 2\alpha - 1} \varepsilon^2(n, \delta)\right)\nonumber\\
    \leq & \frac{C}{C-1}\left(\frac{2c_1^2 c_3 (\alpha + 2)^{3\alpha + 3}}{(\gamma + 1)c_2 ^{3\alpha + 6}} r^{\gamma - \alpha} \varepsilon(n, \delta) +  \frac{C c_1^3 c_3(\alpha + 2)^{4\alpha + 4}}{2 (\gamma + 1) c_2^{4\alpha + 8}} r^{\gamma - \alpha}\varepsilon(n, \delta)\right.\nonumber\\
    & + \left.\frac{c_1^2 c_3 (\alpha + 2)^{3\alpha + 3}}{c_2^{3\alpha + 6}} r^{\gamma - \alpha} \varepsilon(n, \delta) + + \frac{c_1 c_3 (\alpha + 2)^{2\alpha + 2}}{2 c_2^{2\alpha + 4} C } r^{\gamma-\alpha} \varepsilon(n, \delta) \right), \label{eq:upper_bound_T1}
\end{align}
for any $\theta\in \mathbb{B}(\theta^{*}, r) \backslash \mathbb{B}(\theta^{*}, r_{n})$.
Combining inequalities~\eqref{eq:upper_bound_T2} and~\eqref{eq:upper_bound_T1} and taking the constant $c_4$ accordingly, we can obtain the desired result.
\subsection{Proof of Theorem~\ref{theorem:convergence_rate_Polyak}}
\label{subsec:proof:theorem:convergence_rate_Polyak}

Recall that for the radius of $r_{n}$ in Lemma~\ref{supdif}, we denote $r_{n} = \bar{C} \cdot \varepsilon(n, \delta)^{\frac{1}{\alpha + 1 - \gamma}}$. Without loss of generality, we assume $\|\theta_n^k  - \theta^*\| > \left(\frac{c_{4} \bar{C}^{\gamma - \alpha}}{1 - \kappa} + 1 \right) r_{n}$ holds for all $k < T$ where $c_{4}$ is the universal constant in Lemma~\ref{supdif}, $T : = C \log(1/ \varepsilon(n, \delta)$ and $C$ is some constant that will be chosen later; otherwise the conclusion of the theorem already holds.

We first show that, $\theta_n^{k} \in \mathbb{B}(\theta^*, \rho) \backslash \mathbb{B}(\theta^{*}, r_{n})$ for all $k < T$. The inequality $\|\theta_{n}^{k} - \theta^{*}\| > r_{n}$ is direct from the hypothesis. Therefore, we only need to prove that $\|\theta_{n}^{k} - \theta^{*}\| \leq \rho$. Indeed, we have
\begin{align*}
    \|\theta_n^{k+1} - \theta^*\| = & \|F_n(\theta_n^k) - \theta^*\|\\
    \leq & \|F_n(\theta_n^k) - F(\theta_n^k)\| + \|F(\theta_n^k) - \theta^*\|\\
    \leq & \sup_{\theta \in \mathbb{B}(\theta^*, \rho) \backslash \mathbb{B}(\theta^{*}, r_{n})} \|F_n(\theta) - F(\theta)\| + \|F(\theta_n^k) - \theta^*\| \\
    \stackrel{(i)} {\leq} & c_4 \rho^{\gamma - \alpha} \varepsilon(n, \delta) + \kappa \|\theta_n^k - \theta^*\|\\
    \stackrel{(ii)} {\leq} & c_{4} \bar{C}^{\gamma - \alpha} \varepsilon(n, \delta)^\frac{1}{\alpha + 1 - \gamma} + \kappa \rho \\
    \stackrel{(iii)} {\leq} & \rho
\end{align*}
with probability $1 - \delta$ where the inequality (i) is due to Lemma~\ref{supdif} and $c_{4}$ is the universal constant in that lemma; the inequality (ii) is due to $\rho > r_{n} = \bar{C} \varepsilon(n, \delta)^{\frac{1}{\alpha + 1 - \gamma}}$ and $\gamma \leq \alpha$; the inequality (iii) is due to the assumption that $n$ is sufficiently large such that $c_{4} \bar{C}^{\gamma - \alpha} \varepsilon(n, \delta)^{\frac{1}{\alpha + 1 - \gamma}} \leq (1 - \kappa) \rho$. As a consequence, we can guarantee that $\theta_n^{k} \in \mathbb{B}(\theta^*, \rho) \backslash \mathbb{B}(\theta^{*}, r_{n})$ for all $k < T$.

Now, we would like to show that $\|\theta_{n}^{T} - \theta^{*}\| \leq \frac{2 - \kappa}{1 - \kappa} r_{n}$. Indeed, following the earlier argument, we find that
\begin{align*}
    \|\theta_n^{T} - \theta^*\|\leq &  \|F_n(\theta_n^{T - 1}) - F(\theta_n^{T - 1})\| + \|F(\theta_n^{T - 1}) - \theta^*\|\\
    \leq & \sup_{\theta \in \mathbb{B}(\theta^*, \rho) \backslash \mathbb{B}(\theta^{*}, r_{n})} \|F_n(\theta) - F(\theta)\| + \kappa \|\theta_n^{T - 1} - \theta^*\| \\
    \leq & c_4 \cdot r_{n}^{\gamma - \alpha} \varepsilon(n, \delta) + \kappa \|\theta_n^{T - 1} - \theta^*\| \\
    = & c_{4} \bar{C}^{\gamma - \alpha} \cdot \varepsilon(n, \delta)^{\frac{1}{\alpha + 1 - \gamma}} + \kappa \|\theta_n^{T - 1} - \theta^*\|.
\end{align*}
By repeating the above argument $T$ times, we finally obtain that
\begin{align*}
    \|\theta_n^{T} - \theta^*\| \leq & c_{4} \bar{C}^{\gamma - \alpha} \cdot \varepsilon(n, \delta)^{\frac{1}{\alpha + 1 - \gamma}} \left( \sum_{t = 0}^{T - 1} \kappa^{t} \right) + \kappa^{T} \|\theta_{n}^{0} - \theta^{*}\| \\
    \leq & \frac{c_{4} \bar{C}^{\gamma - \alpha}}{1 - \kappa} \varepsilon(n, \delta)^{\frac{1}{\alpha + 1 - \gamma}}  + \kappa^{T} \rho.
\end{align*}
By choosing $T$ such that $\kappa^{T} \rho \leq \varepsilon(n, \delta)^{\frac{1}{\alpha + 1 - \gamma}}$, which is equivalent to $T \geq \frac{\log(\rho) + \frac{1}{\alpha + 1 - \gamma} \log(1/ \varepsilon(n, \delta))}{\log(1/ \kappa)}$, we can guarantee that
\begin{align*}
    \|\theta_n^{T} - \theta^*\| \leq \left( \frac{c_{4} \bar{C}^{\gamma - \alpha}}{1 - \kappa} + 1 \right) \varepsilon(n, \delta)^{\frac{1}{\alpha + 1 - \gamma}}.
\end{align*}
As a consequence, we obtain the conclusion of the theorem.
\section{Discussion}
\label{sec:discussion}
In this paper, we have provided statistical and computational complexities of the Polyak step size gradient descent iterates under the generalized smoothness and Łojasiewicz property of the population loss function as well as the uniform concentration bound between the gradients of the population and sample loss functions. Our results indicate that the Polyak step size iterates only take a logarithmic number of iterations to reach a final statistical radius, which is much fewer than the polynomial number of iterations of the fixed-step size gradient descent iterates to reach the same final statistical radius, when the population loss function is not locally strongly convex. Given that the complexity per iteration of the Polyak step size and fixed-step size gradient descent methods are similar, these results indicate that the Polyak step size gradient descent method is computationally more efficient than the fixed-step size gradient descent method in terms of the number of sample size when the dimension is fixed. Finally, we illustrate our findings under three statistical models: generalized linear model, mixture model, and mixed linear regression model. A few natural future questions arising from our work. 

First, our general theory for the convergence rate of the Polyak step size gradient descent iterates relies on the assumptions that the constants of the generalized smoothness and the generalized Łojasiewicz condition are similar. While this assumption is natural in several statistical models, there are also certain instances of statistical models that this requirement does not hold, such as general over-specified low rank matrix factorization problem, and factor analysis. Therefore, extending our theory of the Polyak step size gradient descent algorithm to the settings when the constants in these assumptions are not similar is of interest. 

Second, our results are restricted to the settings of i.i.d. data in which we can define the corresponding population loss function of the sample loss function. In dependent settings, such as time series data, since the notion of population loss function is not well-defined, it is of interest to develop a new framework beyond the population to sample framework in the current paper to analyze the behavior of Polyak step size gradient descent method for solving the optimal solution of the sample loss function.

Finally, our results shed light on the favorable performance of adaptive gradient methods for dealing with the singular settings of the statistical models, namely, those settings when the Fisher information matrix around the true parameter is degenerate or close to be degenerate, which leads to the slow convergence rates of estimating the true parameters. For the future work, it is of practical interest to extend our general theoretical studies under these settings to other popular adaptive gradient descent methods, such as Adagrad~\citep{Duchi_Adagrad} and Adam~\citep{Kingma_Adam}, that have been observed to have favorable performance in several machine learning and deep learning models.  
\section{Acknowledgements}
\label{sec:acknowledge}
This work was partially supported by the NSF IFML 2019844 award and research gifts by UT Austin ML grant to NH, and by NSF awards 1564000 and 1934932 to SS.

\appendix
\section{Proofs of remaining key results}
\label{sec:remaining_results}
In this appendix, we provide proofs for the generalized smoothness and PL conditions of the generalized linear model, over-specified mixture model, and over-specified mixed linear regression model in the main text.
\subsection{Generalized linear model}
\label{sec:smoothness_PL_generalized_linear}
We first prove the local strong convexity~\eqref{eq:geometry_generalized_linear_strong} and uniform concentration bound~\eqref{eq:concentration_generalized_linear_strong} under the strong signal-to-noise regime in Section~\ref{sec:strong_noise_generalized}. Then, we prove the generalized Łojasiewicz property~\eqref{eq:concentration_gradient_generalized_model} of the population loss $\mathcal{L}$ for the low signal-to-noise regime in Section~\ref{sec:low_signal_generalized}.
\subsubsection{Strong signal-to-noise regime}
\label{sec:strong_noise_generalized}
\textbf{Local strong convexity:} We first prove the local strong convexity in equation~\eqref{eq:geometry_generalized_linear_strong}. Recall that, we have
\begin{align*}
    \mathcal{L}(\theta) = & \frac{1}{2} \left( \mathbb{E}\left[\left((X^{\top}\theta^*)^{p}-(X^{\top}\theta)^{p}\right)^2\right] + \sigma^2 \right),
\end{align*}
where the outer expectation is taken with respect to $X \sim \mathcal{N}(0, I_{d})$. Hence, $\mathcal{L}$ is a polynomial function with degree at most $2p$ and coefficients bounded (as for Gaussian we have any finite order moment bounded). So $\mathcal{L}$ should be smooth around the optima. Furthermore, when $\|\theta - \theta^{*}\|$ is small enough we have that
\begin{align*}
    \left((X^{\top}\theta)^{p}-(X^{\top}\theta^*)^{p}\right)^2 = p(X^\top\theta^*)^{p-1} X^\top(\theta - \theta^*) + o(\|\theta - \theta^*\|).
\end{align*}
Thus, we have that
\begin{align*}
    \mathcal{L}(\theta) = & \frac{1}{2} \left( \mathbb{E}\left[\left((X^{\top}\theta^*)^{p}-(X^{\top}\theta)^{p}\right)^2\right] + \sigma^2 \right)\\
    = & \frac{p^2}{2} (\theta - \theta^*)^\top \mathbb{E}\left[ X^\top (X^\top \theta^*)^{2p-2} X\right] (\theta - \theta^*) + \frac{\sigma^2}{2} + o(\|\theta - \theta^*\|^2).
\end{align*}
As $2p-2$ is even, it is clear that we have $\mathbb{E}\left[ X^\top (X^\top \theta^*)^{2p-2} X\right]$ is positive definite matrix, which shows $\mathcal{L}$ is locally strongly convex function  (by manipulating $\|\theta - \theta^*\|$ and the constant).

\paragraph{Uniform concentration bound:} For the uniform concentration of the gradient in equation~\eqref{eq:concentration_generalized_linear_strong}, direct calculations show that
\begin{align*}
    \nabla \mathcal{L}_n(\theta) = & -\frac{p}{n} \sum_{i=1}^n \left(Y_i - (X_i^\top \theta)^p\right) (X_i^\top \theta)^{p-1} X_{i},\\
    \nabla \mathcal{L}(\theta) = & -p \cdot \mathbb{E}\left[\left((X^\top \theta^*)^p - (X^\top \theta)^p\right) (X^\top \theta)^{p-1} X\right].
\end{align*}
Hence, with triangle inequality, we have that
\begin{align*}
    \|\nabla \mathcal{L}_n(\theta) - \nabla \mathcal{L}(\theta)\| & \leq  \left\|\left(\frac{1}{n} \sum_{i=1}^n (Y_i - (X_i^\top \theta^*)^{p}) (X_i^\top \theta)^{p - 1} X_i \right)\right\|\\
    & + \left\|\left(\frac{1}{n} \sum_{i=1}^n (X_i^\top \theta^*)^p (X_i^\top \theta)^{p - 1} X_i - \mathbb{E}[(X^\top \theta^*)^p (X^\top \theta)^{p - 1} X ] \right)\right\|\\
    & + \left\|\left(\frac{1}{n}\sum_{i = 1}^{n} (X_i^\top \theta)^{2p-1} X_{i} - \mathbb{E}[(X^\top \theta)^{2p-1}X]\right)\right\| \\
    & : = T_{1} + T_{2} + T_{3}.
\end{align*}
The first and the third terms $T_{1}$ and $T_{3}$ can be upper bounded via the identical method introduced in Section A.2 in \citep{mou2019diffusion} and we only need to change the radius from $r$ to $r + \|\theta^*\|$ when $\theta \in \mathbb{B}(\theta^{*}, r)$, namely, we have the following bounds:
\begin{align}
    T_{1} & \leq c_{1} (r + \|\theta^{*}\|)^{p - 1} \biggr(\sqrt{\frac{d + \log(1/\delta)}{n}} + \frac{1}{n} \left(d + \log \left(\frac{n}{\delta}\right) \right)^{p + 1} \biggr), \label{eq:bound_T1} \\
    T_{3} & \leq c_{2} (r + \|\theta^{*}\|)^{2p - 1} \biggr(\sqrt{\frac{d + \log(1/\delta)}{n}} + \frac{1}{n} \left(d + \log \left(\frac{n}{\delta}\right) \right)^{2p + 1} \biggr) \label{eq:bound_T3}
\end{align}
with probability $1 - \delta$ where $c_{1}$ and $c_{2}$ are some universal constants. Therefore, it is sufficient to focus on the second term $T_{2}$. Without the loss of generality, we assume $\|\theta\| = 1$, and the results can be generalized to other norm of $\theta$ by rescaling. First, we know that
\begin{align*}
    & \hspace{- 5 em} \left\|\left(\frac{1}{n} \sum_{i=1}^n (X_i^\top \theta^*)^p (X_i^\top \theta)^{p - 1} X_i - \mathbb{E}[(X^\top \theta^*)^p (X^\top \theta)^{p - 1} X ] \right)\right\| \\ 
    & = \sup_{u\in\mathbb{S}^{d-1}} \left|\left(\frac{1}{n} \sum_{i=1}^n (X_i^\top \theta^*)^p (X_i^\top \theta)^{p - 1} X_i^\top u - \mathbb{E}[(X^\top \theta^*)^p (X^\top \theta)^{p - 1} X^\top u ] \right)\right|,
\end{align*}
where $\mathbb{S}^{d-1}$ is the unit norm Euclidean sphere in $\mathbb{R}^{d}$. With standard discretization arguments (e.g., Chapter 6 in~\citep{Wainwright_nonasymptotic}), let $U$ be a $1/8$-cover of $\mathbb{S}^{d-1}$ under $\|\cdot\|_2$ whose cardinality can be upper bounded by $17^d$, we know
\begin{align*}
    & \hspace{- 5 em} \sup_{u\in\mathbb{S}^{d-1}} \left|\left(\frac{1}{n} \sum_{i=1}^n (X_i^\top \theta^*)^p (X_i^\top \theta)^{p - 1} X_i^\top u - \mathbb{E}[(X^\top \theta^*)^p (X^\top \theta)^{p - 1} X^\top u ] \right)\right|\\
    & \leq 2\sup_{u\in U} \left|\left(\frac{1}{n} \sum_{i=1}^n (X_i^\top \theta^*)^p (X_i^\top \theta)^{p - 1} X_i^\top u - \mathbb{E}[(X^\top \theta^*)^p (X^\top \theta)^{p - 1} X^\top u ] \right)\right|.
\end{align*}
Hence we can focus on the upper bound with a fixed $u$ where $\|u\| = 1$. We then apply a symmetrization argument (e.g., Theorem 4.10 in~\citep{Wainwright_nonasymptotic}), we have that, for any even integer $q$,
\begin{align*}
    & \hspace{- 5 em} \mathbb{E}\left[\left|\left(\frac{1}{n} \sum_{i=1}^n (X_i^\top \theta^*)^p (X_i^\top \theta)^{p - 1} X_i^\top u - \mathbb{E}[(X^\top \theta^*)^p (X^\top \theta)^{p - 1} X^\top u ] \right)\right|^q\right]\\
    & \leq \mathbb{E}\left[\left|\left(\frac{2}{n} \sum_{i=1}^n \varepsilon_i (X_i^\top \theta^*)^p (X_i^\top \theta)^{p - 1} X_i^\top u\right)\right|^q\right],
\end{align*}
where $\{\varepsilon_i\}_{i\in [n]}$ is a set of i.i.d. Rademacher random variables. We then follow the proof strategy used in Section A.2 in \citep{mou2019diffusion}. For a compact set $\Omega$, define
\begin{align*}
    \mathcal{R}(\Omega) := \sup_{\theta\in\Omega, p^\prime\in[1, p]} \left|\frac{2}{n}\sum_{i=1}^n \varepsilon_i (X_i^\top \theta^*)^p(X_i^\top \theta)^{p^\prime-1} X_i^\top u\right|,
\end{align*}
and $\mathcal{N}(t)$ is a $t$-cover of $\mathbb{S}^{d-1}$ under $\|\cdot\|_2$. Then,
\begin{align*}
    \mathcal{R}(\mathbb{S}^{d-1}) = & \sup_{\theta\in \mathbb{S}^{d-1}, p^\prime\in[1, p]}\left|\frac{2}{n}\sum_{i=1}^n \varepsilon_i (X_i^\top \theta^*)^{p}(X_i^\top \theta)^{p^\prime-1} X_i^\top u\right|\\
    \leq & \sup_{\theta_t\in\mathcal{N}(t), \|\eta\|\leq t, p^\prime\in[1, p]} \left|\frac{2}{n}\sum_{i=1}^n \varepsilon_i (X_i^\top \theta^*)^p(X_i^\top (\theta_t + \eta))^{p^\prime-1} X_i^\top u \right|\\
    \leq & \sup_{\theta_t\in\mathcal{N}(t), p^\prime\in[1, p]} \left|\frac{4}{n}\sum_{i=1}^n \varepsilon_i (X_i^\top \theta^*)^p(X_i^\top \theta_t)^{p^\prime-1} X_i^\top u \right| \\
    & + \max_{p^\prime\in[1, p]} 3^{p^\prime - 1}\left|\frac{2}{n}\sum_{i=1}^n \varepsilon_i (X_i^\top \theta^*)^p(X_i^\top \eta)^{p^\prime -1} X_i^\top u \right|\\
    \leq & 2\mathcal{R}(\mathcal{N}(t)) + 3^{p^\prime - 1} t \mathcal{R}(\mathbb{S}^{d-1}).
\end{align*}
Take $t=3^{-p}$, we have that $\mathcal{R}(\mathbb{S}^{d-1}) \leq 3 \mathcal{R}(\mathcal{N}(3^{-p}))$. We then move to the upper bound of  $\mathcal{R}(\mathcal{N}(3^{-p}))$. With the union bound, for any $q\geq 1$ we have that
\begin{align*}
    & \hspace{- 3 em} \sup_{\theta\in\mathbb{S}^{d-1}, p^\prime\in [1, p]} \mathbb{E}\left[\left|\frac{2}{n}\sum_{i=1}^n \varepsilon_i (X_i^\top \theta^*)^p(X_i^\top \theta)^{p^\prime-1} X_i^\top u\right|^q\right] \\
    & = \sup_{\theta\in\mathbb{S}^{d-1}, p^\prime\in [1, p]}\int_{0}^{\infty}\mathbb{P}\left(\left|\frac{2}{n}\sum_{i=1}^n \varepsilon_i (X_i^\top \theta^*)^p(X_i^\top \theta)^{p^\prime-1} X_i^\top u\right|^q\geq \varepsilon\right) d \varepsilon\\
    & \geq \sup_{\theta\in \mathcal{N}(3^{-p}), p^\prime\in [1, p]}\int_{0}^{\infty}\mathbb{P}\left(\left|\frac{2}{n}\sum_{i=1}^n \varepsilon_i (X_i^\top \theta^*)^p(X_i^\top \theta)^{p^\prime-1} X_i^\top u\right|^q\geq \varepsilon\right) d \varepsilon\\
    & \geq \frac{\sup_{p^\prime\in [1, p]}\sum_{\theta\in\mathcal{N}(3^{-p})}\int_{0}^{\infty}\mathbb{P}\left(\left|\frac{2}{n}\sum_{i=1}^n \varepsilon_i (X_i^\top \theta^*)^p(X_i^\top \theta)^{p^\prime-1} X_i^\top u\right|^q\geq \varepsilon\right) d \varepsilon}{|\mathcal{N}(3^{-p})|}\\
    & \geq \frac{\sup_{p^\prime\in [1, p]}\int_{0}^{\infty}\mathbb{P}\left(\sup_{\theta\in\mathcal{N}(3^{-p})}\left|\frac{2}{n}\sum_{i=1}^n \varepsilon_i (X_i^\top \theta^*)^p(X_i^\top \theta)^{p^\prime-1} X_i^\top u\right|^q\geq \varepsilon\right) d \varepsilon}{|\mathcal{N}(3^{-p})|}\\
    & \geq \frac{\int_{0}^{\infty}\mathbb{P}\left(\sup_{\theta\in\mathcal{N}(3^{-p}), p^\prime\in[1, p]}\left|\frac{2}{n}\sum_{i=1}^n \varepsilon_i (X_i^\top \theta^*)^p(X_i^\top \theta)^{p^\prime-1} X_i^\top u\right|^q\geq \varepsilon\right) d \varepsilon}{p|\mathcal{N}(3^{-p})|} \\
    & = \frac{\mathbb{E}[\mathcal{R}^q(\mathcal{N}(3^{-p}))]}{p|\mathcal{N}(3^{-p})|}.
\end{align*}
Hence, it's sufficient to consider $\mathbb{E}\left[\left|\frac{2}{n}\sum_{i=1}^n \varepsilon_i (X_i^\top \theta^*)^p(X_i^\top \theta)^{p^\prime-1} X_i^\top u\right|^q\right] $. We apply Khintchine’s inequality \citep{boucheron2013concentration}, which guarantees that there is an universal constant $C$, such that for all $p^\prime\in[1, p]$, we have
\begin{align*}
    \mathbb{E}\left[\left|\frac{2}{n}\sum_{i=1}^n \varepsilon_i (X_i^\top \theta^*)^p(X_i^\top \theta)^{p^\prime-1} X_i^\top u\right|^q\right]\leq \mathbb{E}\left[\left(\frac{Cq}{n^2}\sum_{i=1}^n (X_i^\top \theta^*)^{2p}(X_i^\top \theta)^{2(p^\prime-1)} (X_i^\top u)^2\right)^{q/2}\right]
\end{align*}
To further upper bound the right hand side of the above equation, we consider the large deviation property of random variable $(X_i^\top \theta^*)^{2p}(X_i^\top \theta)^{2(p^\prime-1)}(X_i^\top u)^2$. It's straightforward to show that
\begin{align*}
    \mathbb{E}\left[(X_i^\top \theta^*)^{2p}(X_i^\top \theta)^{2(p^\prime-1)}(X_i^\top u)^2\right] \leq & (2(p + p^\prime))^{(p + p^\prime)},\\
    \mathbb{E}\left[\left((X_i^\top \theta^*)^{2p}(X_i^\top \theta)^{2(p^\prime-1)}(X_i^\top u)^2\right)^{q/2}\right] \leq & (2(p + p^\prime)q)^{(p + p^\prime) q}.
\end{align*}
With Lemma 2 in \citep{mou2019diffusion}, with probability at least $1-\delta$, we have
\begin{align*}
    & \hspace{- 2 em} \left|\frac{1}{n}\sum_{i=1}^n \left((X_i^\top \theta^*)^{2p}(X_i^\top \theta)^{2(p^\prime-1)} (X_i^\top u)^2\right)^{q/2} - \mathbb{E}\left[(X_i^\top \theta^*)^{2p}(X_i^\top \theta)^{2(p^\prime-1)}(X_i^\top u)^2\right]\right|\\
    & \leq (8(p+p^\prime))^{(p+p^\prime)} \sqrt{\frac{\log 4/\delta}{n}} + (2(p + p^\prime) \log (n/\delta))^{(p+p^\prime)} \frac{\log 4/\delta}{n}.
\end{align*}
Hence, we have that
\begin{align*}
    & \mathbb{E}\left[\left(\frac{1}{n}\sum_{i=1}^n(X_i^\top \theta^*)^{2p}(X_i^\top \theta)^{2(p^\prime-1)} (X_i^\top u)^2\right)^{q/2}\right]\\
    \leq & 2^{q/2}\left(\mathbb{E}\left[(X_i^\top \theta^*)^{2p}(X_i^\top \theta)^{2(p^\prime-1)}(X_i^\top u)^2\right]\right)^{q/2} \\
    & + 2^{q/2}\mathbb{E} \left[\left|\sum_{i=1}^n \left((X_i^\top \theta^*)^{2p}(X_i^\top \theta)^{2(p^\prime-1)} (X_i^\top u)^2\right)^{q/2} - \mathbb{E}\left[(X_i^\top \theta^*)^{2p}(X_i^\top \theta)^{2(p^\prime-1)}(X_i^\top u)^2\right]\right|^{q/2}\right]\\
    \leq & (4(p+p^\prime))^{(p+p^\prime)q} \\
    & + 2^{q/2}\int_{0}^{\infty} \mathbb{P}\left[\left|\sum_{i=1}^n \left((X_i^\top \theta^*)^{2p}(X_i^\top \theta)^{2(p^\prime-1)} (X_i^\top u)^2\right)^{q/2} - \mathbb{E}\left[(X_i^\top \theta^*)^{2p}(X_i^\top \theta)^{2(p^\prime-1)}(X_i^\top u)^2\right]\right|\geq \lambda\right] d\lambda^{q/2}\\
    \leq &  (4(p+p^\prime))^{(p+p^\prime)q} + 2^{q/2} q (p + p^\prime + 1)\\
    & \cdot \int_{0}^{1} \delta \left((8(p+p^\prime))^{(p+p^\prime)} \sqrt{\frac{\log 4/\delta}{n}} +  \frac{(2(p + p^\prime) \log (n/\delta))^{(p+p^\prime + 1)}}{n}\right)^{q/2} d\log(n/\delta)\\
    \leq & (4(p+p^\prime))^{(p+p^\prime)q} + C^\prime(p + p ^\prime) q\left((32(p+p^\prime))^{(p + p^\prime)q/2} n^{-q/4})\Gamma(q/4) \right.\\
    & \left.+ (8(p+p^\prime))^{(p+p^\prime + 1)q/2}n^{-q/2}\left((\log n)^{(p^\prime + p + 1)q/2} + \Gamma((p + p^\prime + 1)q/2)\right)\right),
\end{align*}
where $C^\prime$ is a universal constant and $\Gamma(\cdot)$ is the Gamma function. Notice that
\begin{align*}
    & \hspace{- 3 em} \mathbb{E}\left[\left|\left(\frac{1}{n} \sum_{i=1}^n (X_i^\top \theta^*)^p (X_i^\top \theta)^p X_i^\top u - \mathbb{E}[(X^\top \theta^*)^p (X^\top \theta) X^\top u ] \right)\right|^q\right]\\
    & \leq \mathbb{E}[\mathcal{R}^q(\mathbb{S}^{d-1})] \\
    & \leq 3^q \mathbb{E}[\mathcal{R}(\mathcal{N}(3^{-p}))]\\
    & \leq 3^q p |\mathcal{N}(3^{-p})| \sup_{\theta\in\mathbb{S}^{d-1} p^\prime\in[1, p]} \mathbb{E} \left[\left|\frac{2}{n}\sum_{i=1}^n \varepsilon_i (X_i^\top \theta^*)^p(X_i^\top \theta)^{p^\prime-1} X_i^\top u\right|^q\right]\\
    & \leq 3^q p (3^{p+1})^d \left(\frac{Cq}{n}\right)^{q/2} \left((16p)^{2pq} + 2C^\prime pq \left(64p\right)^{pq} n^{-q/4} \Gamma(q/4) \right.\\
    & \left.+ (16p)^{(2p+1)q/2}n^{-q/2}\left((\log n)^{(2p+1)q/2} + \Gamma((2p+1)q/2) \right)\right),
\end{align*}
for any $u\in U$. Eventually, with union bound, we obtain
\begin{align*}
    & \hspace{-3 em} \left(\mathbb{E}\left[\left\|\left(\frac{1}{n} \sum_{i=1}^n (X_i^\top \theta^*)^p (X_i^\top \theta)^p X_i - \mathbb{E}[(X^\top \theta^*)^p (X^\top \theta) X ] \right)\right\|^q\right]\right)^{1/q}\\
    & \leq 2 \left(\mathbb{E} \left[\sup_{u\in U} \left|\left(\frac{1}{n} \sum_{i=1}^n (X_i^\top \theta^*)^p (X_i^\top \theta)^p X_i^\top u - \mathbb{E}[(X^\top \theta^*)^p (X^\top \theta) X^\top u ] \right)\right|^q\right]\right)^{1/q}\\
    & \leq 2  \left(\mathbb{E}\left[\sum_{u\in [U]}\left|\left(\frac{1}{n} \sum_{i=1}^n (X_i^\top \theta^*)^p (X_i^\top \theta)^p X_i^\top u - \mathbb{E}[(X^\top \theta^*)^p (X^\top \theta) X^\top u ] \right)\right|^q\right]\right)^{1/q}\\
    & \leq 2 \cdot 17^{d/q}\sup_{u\in [U]}\mathbb{E}\left[\left|\left(\frac{1}{n} \sum_{i=1}^n (X_i^\top \theta^*)^p (X_i^\top \theta)^p X_i^\top u - \mathbb{E}[(X^\top \theta^*)^p (X^\top \theta) X^\top u ] \right)\right|^q\right]^{1/q}\\
    & \leq 6 \cdot (17 \cdot 3^{p+1})^{d/q} \left[\sqrt{\frac{C_p q}{n}} + \left(\frac{C_p q}{n}\right)^{3/4} + \frac{C_p}{n}(\log n + q)^{(2p+1)/2}\right],
\end{align*}
where $C_p$ is a universal constant that only depends on $p$. Take $q = d(p + 3) + \log (1/\delta)$ and use the Markov inequality, we get the following bound on the second term $T_{2}$ with probability $1 - \delta$:
\begin{align}
    T_{2} \leq c_{3} (r + \|\theta^{*}\|)^{p - 1} \biggr(\sqrt{\frac{d + \log(1/\delta)}{n}} + \frac{1}{n} \left(d + \log \left(\frac{n}{\delta}\right) \right)^{\frac{2p + 1}{2}} \biggr). \label{eq:bound_T2} 
\end{align}
Combining the bounds from equations~\eqref{eq:bound_T1},~\eqref{eq:bound_T2}, and~\eqref{eq:bound_T3}, as long as $n \geq C_{1} (d \log(d/ \delta))^{2p}$ we have
\begin{align*}
    \sup_{\theta \in \mathbb{B}(\theta^{*},r)} \|\nabla \mathcal{L}_n(\theta) - \nabla \mathcal{L}(\theta)\| \leq C_{2} (r + \|\theta^{*}\|)^{2p - 1} \sqrt{\frac{d + \log(1/\delta)}{n}}
\end{align*}
where $C_{1}, C_{2}$ are some universal constants. Since $\|\theta^{*}\|$ is bounded away from 0, the above bound concludes our claim in equation~\eqref{eq:concentration_generalized_linear_strong}. 
\subsubsection{Low signal-to-noise regime}
\label{sec:low_signal_generalized}
Now, we prove the generalized Łojasiewicz property~\eqref{eq:concentration_gradient_generalized_model} of the population loss $\mathcal{L}$ for the low signal-to-noise regime. Recall that, we assume $\theta^{*} = 0$. Now, we will demonstrate that for all $\theta \in \mathbb{B}(\theta^{*}, \rho)$ for some $\rho > 0$, we have
\begin{align*}
    \|\nabla \mathcal{L}(\theta)\| & \geq c_{2} (\mathcal{L}(\theta) - \mathcal{L}(\theta^{*}))^{1 - \frac{1}{2p}}.
\end{align*}
For the form of $\mathcal{L}(\theta)$, we have that
\begin{align*}
    &\nabla \mathcal{L}(\theta) = 2p(2p-1)!!(\theta-\theta^*) \|\theta - \theta^{*}\|^{2p - 2}, \\
    & \|\nabla \mathcal{L}(\theta)\| = 2p(2p-1)!!\|\theta - \theta^{*}\|^{2p - 1}.
\end{align*}
Also, due to equation~\eqref{eq:population_no_signal} we obtain that
\begin{align*}
    \left(\mathcal{L}(\theta) - \mathcal{L}(\theta^{*})\right)^{1-\frac{1}{2p}} &= \left(\frac{(2p - 1)!! \|\theta - \theta^{*}\|^{2p}}{2} \right)^{1-\frac{1}{2p}} \\
    &= \left(\frac{(2p - 1)!!}{2} \right)^{1-\frac{1}{2p}} \|\theta - \theta^{*}\|^{2p-1}.
\end{align*}
Thus, the Assumption~\ref{assump:nonPL} follows by selecting the constant $c_2 \leq \frac{2p(2p-1)!!}{\left(\frac{(2p - 1)!!}{2} \right)^{1-\frac{1}{2p}}}$.
\\
Next, with direct computation, we have
\begin{align*}
    \nabla^2 \mathcal{L}(\theta) = (2p(2p-1)!!)\|\theta - \theta^*\|^{2p-4} \left(\|\theta - \theta^*\|^2 I + (2p-4)(\theta - \theta^*)(\theta - \theta^*)^\top\right).
\end{align*}
Notice that, $(\theta - \theta^*)(\theta - \theta^*)^\top$ is a rank-1 matrix, so the maximum eigenvalue of $\|\theta - \theta^*\|^2 I + (2p-4)(\theta - \theta^*)(\theta - \theta^*)^\top$ is $(2p-3)\|\theta - \theta^*\|^2$, hence $\lambda_{\max}(\nabla^2\mathcal{L}(\theta)) = 2p(2p-3)(2p-1)!!\|\theta-\theta^*\|^{2p-2}$, which confirms our claim of Assumption~\ref{assump:smoothness}.

\subsection{Over-specified mixture model}
\label{sec:smoothness_PL_mixture_model}
We first present a proof of claim~\eqref{eq:geometry_mixture_model_low_signal} about the generalized PL property of the population log-likelihood function $\bar{\mathcal{L}}$ of low-signal-to-noise symmetric two-component Gaussian mixture model in Appendix~\ref{subsec:proof:eq:geometry_mixture_model_low_signal}. Then, in Appendix~\ref{subsec:proof:eq:Lipschitz_mixture_model_low_signal}, we present a proof of claim~\eqref{eq:Lipschitz_mixture_model_low_signal} about the local smoothness of $\bar{\mathcal{L}}$.
\subsubsection{Proof of claim~\eqref{eq:geometry_mixture_model_low_signal}}
\label{subsec:proof:eq:geometry_mixture_model_low_signal}
Recall that $\theta^* = 0$ and the population log-likelihood function is given by:
\begin{align*}
\bar{\mathcal{L}}(\theta) =  -\mathbb{E}_{X}\left[\log\left(\frac{1}{2}\phi(X|\theta, \sigma^2 I_d) + \frac{1}{2}\phi(X|-\theta, \sigma^2 I_d)\right)\right],
\end{align*}
where the outer expectation is taken with respect to $X \sim \mathcal{N}(\theta^{*}, I_{d})$. Using $Z$ to absorb the constant that is independent of $\theta$, we have
\begin{align*}
    \bar{\mathcal{L}}(\theta) 
    = & \frac{\|\theta\|^2}{2 \sigma^2} - \mathbb{E}_{X} \left[\log\left(\exp\left(-\frac{X^{\top} \theta}{\sigma^2}\right) + \exp\left(\frac{X^{\top}\theta}{\sigma^2}\right)\right)\right] + Z.\\
\end{align*}
It indicates that
\begin{align*}
    \bar{\mathcal{L}}(\theta) - \bar{\mathcal{L}}(\theta^{*}) = \frac{\|\theta\|^2}{2 \sigma^2} - \mathbb{E}_{X} \left[\log\left(\exp\left(-\frac{X^{\top} \theta}{\sigma^2}\right) + \exp\left(\frac{X^{\top}\theta}{\sigma^2}\right)\right)\right] + Z.
\end{align*}
To simplify the calculation, we perform a change of coordinates via an orthogonal matrix $R$ such that $R \theta = \|\theta\| e_{1}$ where $e_{1}$ denotes the first canonical basis in dimension $d$. By denoting $V = RX/ \sigma$, we have $V = (V_{1}, \ldots, V_{d}) \sim \mathcal{N}(0, I_{d})$. Therefore, we can rewrite the above equation as follows:
\begin{align*}
    \bar{\mathcal{L}}(\theta) - \bar{\mathcal{L}}(\theta^{*}) = \frac{\|\theta\|^2}{2 \sigma^2} - \mathbb{E}_{V_{1}} \left[\log\left(\exp\left(-\frac{V_{1}\|\theta\|}{\sigma}\right) + \exp\left(\frac{V_{1} \|\theta\|}{\sigma}\right)\right)\right] + Z,
\end{align*}
where the outer expectation is taken with respect to $V_{1} \sim \mathcal{N}(0, 1)$. By using the basic inequality $\exp( - x) + \exp(x) \geq 2 + x^2$ for all $x \in \mathbb{R}$, we find that
\begin{align*}
    \bar{\mathcal{L}}(\theta) - \bar{\mathcal{L}}(\theta^{*}) \leq \frac{\|\theta\|^2}{2 \sigma^2} - \mathbb{E}_{V_{1}} \left[\log\left(1 + \frac{V_{1}^2 \|\theta\|^2}{2 \sigma^2}\right)\right],
\end{align*}
Applying further the inequality $\log(1 + x) \geq x - \frac{x^2}{2}$ for all $x \geq 0$, we have
\begin{align}
    \bar{\mathcal{L}}(\theta) - \bar{\mathcal{L}}(\theta^{*}) \leq \frac{3 \|\theta\|^4}{8 \sigma^4}. \label{eq:key_inequality_population_first}
\end{align}
Now, we proceed to lower bound $\|\nabla \bar{\mathcal{L}}(\theta)\|$. Direct calculation leads to
\begin{align*}
    \nabla \bar{\mathcal{L}}(\theta) = \frac{1}{\sigma^2} \parenth{\theta - \Exs_{X} \parenth{X \tanh(\frac{X^{\top}\theta}{\sigma^2})}}.
\end{align*}
Direct application of the triangle inequality with $\|.\|$ norm indicates that
\begin{align*}
    \|\nabla \bar{\mathcal{L}}(\theta)\| \geq \frac{1}{\sigma^2} \parenth{\|\theta\| - \bigg\|\Exs_{X} \parenth{X \tanh(\frac{X^{\top}\theta}{\sigma^2})} \bigg\|}.
\end{align*}
Using the similar change of coordinates as we did earlier, we obtain that
\begin{align*}
    \bigg\|\Exs_{X} \parenth{X \tanh(\frac{X^{\top}\theta}{\sigma^2}} \bigg\| = \sigma \mathbb{E}_{V_{1}} \parenth{V_{1} \tanh(\frac{V_{1} \|\theta\|}{\sigma})},
\end{align*}
where the outer expectation is taken with respect to $V_{1} \sim \mathcal{N}(0, 1)$. An application of the inequality $x \tanh(x) \leq x^2 - \frac{x^4}{3} + \frac{2x^{6}}{15}$ for all $x \in \mathbb{R}$ leads to
\begin{align*}
    \sigma \mathbb{E}_{V_{1}} \parenth{V_{1} \tanh(\frac{V_{1} \|\theta\|}{\sigma})} & \leq \frac{\sigma^2}{\|\theta\|} \Exs_{V_{1}}\parenth{\frac{V_{1}^2\|\theta\|^2}{\sigma^2} - \frac{V_{1}^4 \|\theta\|^4}{3 \sigma^4} + \frac{2 V_{1}^6 \|\theta\|^6}{15 \sigma^6}} \\
    & = \|\theta\| - \frac{\|\theta\|^3}{\sigma^2} + \frac{2 \|\theta\|^5}{\sigma^4}.
\end{align*}
As long as $\|\theta\| \leq \frac{\sigma}{2}$, we have $2 \|\theta\|^5/\sigma^4 \leq \|\theta\|^3/(2\sigma^2)$. Putting the above inequalities together, we find that
\begin{align}
    \|\nabla \bar{\mathcal{L}}(\theta)\| \geq \frac{\|\theta\|^3}{2 \sigma^4} \label{eq:key_inequality_population_second}
\end{align}
when $\|\theta\| \leq \frac{\sigma}{2}$. Combining the results of equations~\eqref{eq:key_inequality_population_first} and~\eqref{eq:key_inequality_population_second}, we obtain
\begin{align*}
    \|\nabla \bar{\mathcal{L}}(\theta)\| \geq c_{2} \parenth{\bar{\mathcal{L}}(\theta) - \bar{\mathcal{L}}(\theta^{*})}^{\frac{3}{4}}
\end{align*}
when $\|\theta\| \leq \frac{\sigma}{2}$ where $c_{2}$ is some universal constant. Therefore, we obtain the conclusion of claim~\eqref{eq:geometry_mixture_model_low_signal}.
\subsubsection{Proof of claim~\eqref{eq:Lipschitz_mixture_model_low_signal}}
\label{subsec:proof:eq:Lipschitz_mixture_model_low_signal}
Direct calculation shows that
\begin{align*}
    \nabla^2 \bar{\mathcal{L}}(\theta) = \frac{1}{\sigma^2} \parenth{I_{d} - \frac{1}{\sigma^2} \Exs_{X} \parenth{X X^{\top} \text{sech}^2 \parenth{\frac{X^{\top}\theta}{\sigma^2}}}},
\end{align*}
where $\text{sech}^2(x) = \frac{4}{(\exp(-x) + \exp(x))^2}$ for all $x \in \mathbb{R}$. Via an application of the change of coordinates that we used earlier, we can write the above equation as:
\begin{align*}
    \nabla^2 \bar{\mathcal{L}}(\theta) = \frac{1}{\sigma^2} \parenth{I_{d} - \Exs_{V} \parenth{V V^{\top} \text{sech}^2 \parenth{\frac{V_{1} \|\theta\|}{\sigma}}}},
\end{align*}
where the outer expectation is taken with respect to $V = (V_{1}, V_{2}, \ldots, V_{d}) \sim \mathcal{N}(0, I_{d})$. The matrix $A = \Exs_{V} \parenth{V V^{\top} \text{sech}^2 \parenth{\frac{V_{1} \|\theta\|}{\sigma}}}$ is a diagonal matrix that $A_{11} = \Exs_{V_{1}} \brackets{V_{1}^2\text{sech}^2 \parenth{\frac{V_{1} \|\theta\|}{\sigma}}}$ and $V_{ii} = \Exs_{V_{1}} \brackets{\text{sech}^2 \parenth{\frac{V_{1} \|\theta\|}{\sigma}}}$ for all $2 \leq i \leq d$.

An application of the inequality $\text{sech}^2(x) \geq 1 - x^2$ for all $x \in \mathbb{R}$ leads to
\begin{align*}
    A_{11} & \geq \Exs_{V_{1}} \brackets{V_{1}^2 \parenth{1 - \frac{V_{1}^2 \|\theta\|^2}{\sigma^2}}} = 1 - \frac{3 \|\theta\|^2}{\sigma^2}, \\
    A_{ii} & \geq \Exs_{V_{1}} \brackets{1 - \frac{V_{1}^2 \|\theta\|^2}{\sigma^2}} = 1 - \frac{\|\theta\|^2}{\sigma^2},
\end{align*}
for all $i \neq 1$. These results indicate that
\begin{align*}
    \lambda_{\max}(\nabla^2 \bar{\mathcal{L}}(\theta)) \leq \frac{3 \|\theta\|^2}{\sigma^4}.
\end{align*}
As a consequence, we obtain the conclusion of claim~\eqref{eq:Lipschitz_mixture_model_low_signal}.
\subsection{Mixed linear regression model}
\label{sec:smoothness_PL_mixed_linear}
We first present a proof of claim~\eqref{eq:geometry_mixed_linear_low_signal} about the generalized PL property of the population log-likelihood function $\tilde{\mathcal{L}}$ of low-signal-to-noise symmetric two-component Gaussian mixed linear regression in Appendix~\ref{subsec:proof:eq:geometry_mixed_linear_low_signal}. Then, in Appendix~\ref{subsec:proof:eq:smooth_mixed_linear_low_signal}, we present a proof of claim~\eqref{eq:smooth_mixed_linear_low_signal} about the local smoothness of $\tilde{\mathcal{L}}$. The proof ideas of these claims are similar to those in the mixture model case. Here, we provide the proofs for the completeness.
\subsubsection{Proof of claim~\eqref{eq:geometry_mixed_linear_low_signal}}
\label{subsec:proof:eq:geometry_mixed_linear_low_signal}
When $\theta^* = 0$, we have that $Y\sim \mathcal{N}(0, \sigma^2)$. Furthermore, the population log-likelihood function $\tilde{\mathcal{L}}$ admits the following form:
\begin{align*}
    \tilde{\mathcal{L}}(\theta) 
    &= - \mathbb{E}_{X,Y}\left[\log\left(\frac{1}{2}\phi(Y|X^\top \theta, \sigma^2) + \frac{1}{2}\phi(Y|-X^\top\theta, \sigma^2)\right)\right].
\end{align*}
Using $Z$ to absorb the constant that is independent of $\theta$, when $X \sim \mathcal{N}(0, I_{d})$ and $Y|X \sim \mathcal{N}(Y|0, \sigma^2)$ we have
\begin{align*}
    \tilde{\mathcal{L}}(\theta) 
    = \frac{\|\theta\|^2}{2\sigma^2} - \mathbb{E}_{X,Y} \left[\log \left(\exp\left(\frac{Y\theta^{\top}X}{\sigma^2}\right) + \exp\left(\frac{-Y\theta^{\top}X}{\sigma^2}\right) \right)
    \right] + Z.
\end{align*}
Similar to the proof of claim~\eqref{eq:geometry_mixture_model_low_signal}, to bound the expectation in the above equation we can perform a change of coordinates using an orthonormal matrix $R$ such that $R\theta = \|\theta\|e_1$. Let $V = RX$, then $V = (V_{1}, V_{2}, \ldots, V_{d}) \sim \mathcal{N}(0,I_d)$. Moreover, since $\tilde{\mathcal{L}}(\theta^*)$ does not depend on $\theta$, we can write: 
\begin{align*}
    \tilde{\mathcal{L}}(\theta) - \tilde{\mathcal{L}}(\theta^*) 
    &= \frac{\|\theta\|^2}{2\sigma^2} - \mathbb{E}_{X,Y} \left[\log \left(\exp\left(\frac{Y\theta^{\top}X}{\sigma^2}\right) + \exp\left(\frac{-Y\theta^{\top}X}{\sigma^2}\right) \right)
    \right] + Z \\
    &= \frac{\|\theta\|^2}{2\sigma^2} - \mathbb{E}_{V_1,Y} \left[\log \left(\exp\left(\frac{Y\|\theta\|V_1}{\sigma^2}\right) + \exp\left(\frac{-Y\|\theta\|V_1}{\sigma^2}\right) \right)
    \right] + Z.
\end{align*}
Using the standard inequality $\exp(-x)+\exp(x)\geq 2+x^2$ for all $x \in \mathbb{R}$ we find that
\begin{align*}
    \tilde{\mathcal{L}}(\theta) - \tilde{\mathcal{L}}(\theta^*) 
    \leq \frac{\|\theta\|^2}{2\sigma^2} - \mathbb{E}_{V_1,Y} \left[\log \left(1 + \frac{Y^2\|\theta\|^2 V_1^2}{2\sigma^4} \right)
    \right].
\end{align*}
From here, the inequality $\log(1+x)\geq x - \frac{x^2}{2}$ for all $x \geq 0$ leads to
\begin{align}
    \tilde{\mathcal{L}}(\theta) - \tilde{\mathcal{L}}(\theta^*) 
    &\leq \frac{\|\theta\|^2}{2\sigma^2} - \mathbb{E}_{V_1,Y} \left[ \frac{Y^2\|\theta\|^2V_1^2}{2\sigma^4} - \frac{Y^4\|\theta\|^4V_1^4}{8\sigma^8}
    \right] \nonumber \\
    &= \frac{\|\theta\|^2}{2\sigma^2} -  \frac{\mathbb{E}_{Y} [Y^2]\|\theta\|^2\mathbb{E}_{V_1}[V_1^2]}{2\sigma^4} + \frac{\mathbb{E}_{Y}[Y^4]\|\theta\|^4 \mathbb{E}_{V_1}[V_1^4]}{8\sigma^8} \nonumber \\
    & = \frac{9}{8 \sigma^4} \|\theta\|^4. \label{eq:key_inequality_mixed_linear_first}
\end{align}
Now, we establish an lower bound for $\|\nabla \tilde{\mathcal{L}}(\theta)\|$. Indeed, direct calculation shows that
\begin{align*}
    \nabla \tilde{\mathcal{L}}(\theta) = \frac{1}{\sigma^2} \parenth{\theta - \Exs_{X, Y} \brackets{YX\tanh(\frac{Y\theta^{\top}X}{\sigma^2})}}
\end{align*}
Therefore, we find that $\|\nabla \tilde{\mathcal{L}}(\theta)\| \geq \frac{1}{\sigma^2} \parenth{\|\theta\| - \bigg\|\Exs_{X, Y} \brackets{YX\tanh(\frac{Y\theta^{\top}X}{\sigma^2})}\bigg\|}$. Using the earlier change of coordinates, we have
\begin{align*}
    \bigg\|\Exs_{X, Y} \brackets{YX\tanh(\frac{Y\theta^{\top}X}{\sigma^2})}\bigg\| = \Exs_{V_{1}, Y} \brackets{Y V_{1} \tanh(\frac{YV_{1}\|\theta\|}{\sigma^2})}.
\end{align*}
As we have the inequality $x \tanh(x) \leq x^2 - \frac{x^4}{3} + \frac{2x^6}{15}$ for all $x \in \mathbb{R}$, we obtain
\begin{align*}
    \bigg\|\Exs_{X, Y} \brackets{YX\tanh(\frac{Y\theta^{\top}X}{\sigma^2})}\bigg\| & \leq \frac{\sigma^2}{\|\theta\|} \Exs_{V_{1}, Y} \brackets{\frac{Y^2V_{1}^2\|\theta\|^2}{\sigma^4} - \frac{Y^4V_{1}^4\|\theta\|^4}{3 \sigma^{8}} + \frac{2 Y^6V_{1}^6\|\theta\|^6}{15 \sigma^{12}}} \\
    & \leq \|\theta\| - \frac{3}{\sigma^2}\|\theta\|^3 + \frac{30}{\sigma^4} \|\theta\|^5 \leq \|\theta\| - \frac{3}{2\sigma^2}\|\theta\|^3,
\end{align*}
as long as $\|\theta\| \leq \frac{\sigma}{\sqrt{20}}$. Putting the above results together, we find that
\begin{align}
    \|\nabla \tilde{\mathcal{L}}(\theta)\| \geq \frac{3}{2\sigma^4}\|\theta\|^3. \label{eq:key_inequality_mixed_linear_second}
\end{align}
A combination of the results from equation~\eqref{eq:key_inequality_mixed_linear_first} and~\eqref{eq:key_inequality_mixed_linear_second} indicate that
\begin{align*}
     \|\nabla \tilde{\mathcal{L}}(\theta)\| \geq c_{2} \parenth{\tilde{\mathcal{L}}(\theta) - \tilde{\mathcal{L}}(\theta^*)}^{3/4},
\end{align*}
for all $\|\theta\| \leq \frac{\sigma}{\sqrt{20}}$ where $c_{2}$ is some universal constant. As a consequence, we obtain the conclusion of claim~\eqref{eq:geometry_mixed_linear_low_signal}.
\subsubsection{Proof of claim~\eqref{eq:smooth_mixed_linear_low_signal}}
\label{subsec:proof:eq:smooth_mixed_linear_low_signal}
Similar to the proof of claim~\eqref{eq:Lipschitz_mixture_model_low_signal}, we have
\begin{align*}
    \nabla^2 \tilde{\mathcal{L}}(\theta) & = \frac{1}{\sigma^2} \parenth{I_{d} - \frac{1}{\sigma^2} \Exs_{X, Y} \brackets{Y^2XX^{\top} \sech^2 (\frac{Y\theta^{\top}X}{\sigma^2})}} \\
    & = \frac{1}{\sigma^2} \parenth{I_{d} - \frac{1}{\sigma^2} \Exs_{Y, V} \brackets{Y^2 V V^{\top} \sech^2(\frac{Y V_{1}\|\theta\|}{\sigma^2})}},
\end{align*}
where the second equality is from the change of coordinates $R = V X$ and $R$ is an orthogonal matrix such that $R \theta = \|\theta\| e_{1}$. Here, the outer expectation is taken with respect to $Y \sim \mathcal{N}(0, \sigma^2)$ and $V = (V_{1}, \ldots, V_{d}) \sim \mathcal{N}(0, I_{d})$. 

The matrix $B = \frac{1}{\sigma^2} \Exs_{Y, V} \brackets{Y^2 V V^{\top} \sech^2(\frac{Y V_{1}\|\theta\|}{\sigma^2})}$ is a diagonal matrix such that $B_{11} = \Exs_{Y, V_{1}} \brackets{Y^2 V_{1}^2 \sech^2(\frac{Y V_{1}\|\theta\|}{\sigma^2})}$ and $B_{ii} = \Exs_{Y, V_{1}} \brackets{Y^2 \sech^2(\frac{Y V_{1}\|\theta\|}{\sigma^2})}$ for all $i \neq 1$. Using the standard inequality $\sech^2(x) \geq 1 - x^2$ for all $x \in \mathbb{R}$ yields
\begin{align*}
    B_{11} & \geq \Exs_{Y, V_{1}} \brackets{Y^2V_{1}^2 - \frac{Y^4 V_{1}^4 \|\theta\|^2}{\sigma^4}} = \sigma^2 - 9 \|\theta\|^2, \\
    B_{ii} & \geq \Exs_{Y, V_{1}} \brackets{Y^2 - \frac{Y^4 V_{1}^2 \|\theta\|^2}{\sigma^4}} = \sigma^2 - 3 \|\theta\|^2,
\end{align*}
for all $i \neq 1$. Collecting the above results, we obtain
\begin{align*}
    \lambda_{\max}(\nabla^2 \tilde{\mathcal{L}}(\theta)) \leq \frac{9}{\sigma^2} \|\theta\|^2.
\end{align*}
Hence, we obtain the conclusion of claim~\eqref{eq:smooth_mixed_linear_low_signal}.
\section{Auxiliary results}
\label{sec:auxiliary_result}
\begin{lemma}\label{lem:obj_to_params}
If Assumption \ref{assump:nonPL} holds, then for all $\theta\in\mathbb{B}(\theta^*, \rho)$, we have that
\begin{align*}
    \|\theta - \theta^*\| \leq \frac{\alpha + 2}{c_2}(f(\theta) - f(\theta^*))^{\frac{1}{\alpha + 2}}.
\end{align*}
Furthermore, we have
\begin{align*}
    \|\nabla f(\theta)\| \geq c_2 \left(\frac{c_2}{\alpha + 2} \|\theta - \theta^*\|\right)^{\alpha + 1}.
\end{align*}
\end{lemma}
\begin{proof}
The proof idea originates from the proof of Theorem 27 in \citep{bolte2017error}. We start from the gradient flow:
\begin{align*}
    \frac{d\theta(t)}{dt} = - \nabla f(\theta(t)).
\end{align*}
By the convexity, we have that
\begin{align*}
    \frac{d\|\theta(t) - \theta^*\|_2^2}{dt} = 2 \left\langle \theta(t) - \theta^*, \frac{d\theta(t)}{dt}\right\rangle = -2\langle \theta(t) - \theta^*, \nabla f(\theta(t))\rangle\leq 0,
\end{align*}
which means if $\theta(0) \in \mathbb{B}(\theta^*, \rho)$, $\theta(t) \in \mathbb{B}(\theta^*, \rho), \forall t \geq 0$. Meanwhile, $\theta(t) \to \theta^*$ when $t\to\infty$. We then conclude the proof by
\begin{align*}
    \left(f(\theta(0)) - f(\theta^*)\right)^{\frac{1}{\alpha + 2}} = & \int_{\infty}^{0}d(f(\theta(t)) - f(\theta^*))^{\frac{1}{\alpha + 2}} \\
    = & \int_{0}^{\infty} \frac{f(\theta(t)) - f(\theta^*))^{\frac{1}{\alpha + 2} - 1}}{\alpha + 2} \|\nabla f(\theta(t))\|^2 dt\\
    \geq & \int_{0}^{\infty} \frac{c_2}{\alpha + 2}\|\nabla f(\theta(t))\| dt\\
    = & \int_{0}^\infty \frac{c_2}{\alpha + 2} \left\|\frac{d\theta(t)}{dt}\right\| dt\\
    = & \frac{c_2}{\alpha + 2}\|\theta(0) - \theta^*\|.
\end{align*}
The second argument can be directly obtained via Assumption \ref{assump:nonPL}, which concludes our proof.
\end{proof}
\begin{lemma}
\label{lem:convergence_gd}
Under Assumptions~\ref{assump:smoothness} and~\ref{assump:nonPL}, there exists a universal constant $c_0 > 0$ depending on the constants of these assumptions such that
\begin{align*}
    \|\theta_{\text{GD}}^t - \theta^*\|\leq \frac{c_0}{(\eta t)^{1/\alpha}},
\end{align*}
where $\theta_{\text{GD}}^{t + 1} = \theta_{\text{GD}}^{t} - \eta \nabla f(\theta_{\text{GD}}^{t})$ are the fixed-step size gradient descent iterates for minimizing the population loss function $f$. Furthermore, this bound is tight, means there are population loss functions $f$ satisfying Assumptions~\ref{assump:smoothness} and~\ref{assump:nonPL} and 
\begin{align*}
    \|\theta_{\text{GD}}^t - \theta^*\| \geq \frac{c_0}{(\eta t)^{1/\alpha}}.
\end{align*}
\end{lemma}
\begin{proof}
Our proof idea originates from \citep{mei2021leveraging} and we include it for completeness. We start from the following lemma.
\begin{lemma}[Lemma 3.5 in \citep{bubeck2015convex}]
\label{lem:smooth_update}
If $f$ is $\beta$-smooth, then $\forall \theta_1, \theta_2\in\mathbb{R}^d$, we have that
\begin{align*}
    f(\theta_1) - f(\theta_2) \leq \langle\nabla f(\theta_1), \theta_1-\theta_2\rangle - \frac{1}{2\beta}\|\nabla f(\theta_1) - \nabla f(\theta_2)\|^2.
\end{align*}
\end{lemma}
\begin{corollary} 
\label{cor:gradient_norm_bound}
If $f$ is $\beta$-smooth, then $\forall \theta_1, \theta_2\in\mathbb{R}^d$, we have that
\begin{align*}
    \frac{1}{\beta}\|\nabla f(\theta_1) - \nabla f(\theta_2)\|^2 \leq \langle \nabla f(\theta_1) - \nabla f(\theta_2), x-y\rangle.
\end{align*}
\end{corollary}
Notice that, if $\theta_1, \theta_2 \in \mathbb{B}(\theta^*, r)$, then $\|\nabla f(\theta_1) - \nabla f(\theta_2)\| \leq c_1 r^{\alpha}\|\theta_1 - \theta_2\|$, which means $f$ is $c_1 r^\alpha$-smooth in $\mathbb{B}(\theta^*, r)$. We assume the step-size satisfies $0 < \eta < \frac{2}{c_1 r^{\alpha}}$, and define the "effective step-size" $\frac{1}{\beta} := \eta(2 - c_1 r^\alpha \eta) > 0$ where $\beta > c_1 r^\alpha$. If $\theta_{\text{GD}}^t \in \mathbb{B}(\theta^*, r)$, we have that
\begin{align*}
    \|\theta^{t+1} - \theta^*\|^2 - \|\theta^t - \theta^*\|^2
    = & \eta^2 \|\nabla f(\theta^t)\|^2 - 2\eta \langle \nabla f(\theta^t), \theta^t - \theta^*\rangle \\
    \leq & -\frac{1}{\beta}\langle \nabla f(\theta^t), \theta^t - \theta^*\rangle \leq 0,
\end{align*}
where the last inequality is due to Corollary \ref{cor:gradient_norm_bound}. Hence, $\theta_{\text{GD}}^{t+1}\in \mathbb{B}(\theta^*, r)$. Furthermore, from the generalized smoothness property of the function $f$ in Assumption~\ref{assump:smoothness} we have
\begin{align*}
    f(\theta^{t+1}) - f(\theta^t) \leq & \nabla f(\theta^t)^\top (\theta^{t+1} - \theta^t) + \frac{c_1 r^{\alpha}}{2}\|\theta^{t+1} - \theta^t\|^2\\
    = & -\frac{1}{2\beta}\|\nabla f(\theta^t)\|^2\\
    \leq & -\frac{c_2^2}{2\beta} (f(\theta^t) - f(\theta^*))^{2 - \frac{2}{\alpha + 2}} \leq 0.
\end{align*}
\begin{lemma}
Given $\alpha > 0$, $\forall x\in [0, 1]$,
\begin{align*}
    \frac{1}{\alpha}(1-x^{\alpha}) \geq x^{\alpha}(1-x).
\end{align*}
\end{lemma}
\begin{proof}
Consider the mapping $g:x\mapsto \frac{1}{\alpha}(x^{\alpha} - 1) - x^{\alpha}(1-x). $. We can see $g(0) = \frac{1}{\alpha}$ and $g(1) = 0$. Moreover,
\begin{align*}
    \nabla g(x) = - (\alpha + 1)(x^{\alpha-1} - x^{\alpha}) \leq 0,
\end{align*}
which concludes the proof.
\end{proof}
Define $\delta(\theta^t):= f(\theta^t) - f(\theta^*)$, we have that
\begin{align*}
    \frac{1}{\delta(\theta^t)^{\frac{\alpha}{\alpha + 2}}} = & \frac{1}{\delta(\theta^1)^{\frac{\alpha}{\alpha + 2}}} + \sum_{s=1}^{t-1} \left(\frac{1}{\delta(\theta^s)^{\frac{\alpha}{\alpha + 2}}} - \frac{1}{\delta(\theta^{s+1})^{\frac{\alpha}{\alpha + 2}}}\right)\\
    = & \frac{1}{\delta(\theta^1)^{\frac{\alpha}{\alpha + 2}}} + \sum_{s=1}^{t-1} \frac{\frac{\alpha}{\alpha + 2} }{\delta(\theta^{s+1})^{\frac{\alpha}{\alpha + 2}}} \cdot \frac{\alpha + 2}{\alpha} \cdot \left(1-\left(\frac{\delta(\theta^{s+1})}{\delta(\theta^s)}\right)^{\frac{\alpha}{\alpha + 2}}\right)\\
    \geq & \frac{1}{\delta(\theta^1)^{\frac{\alpha}{\alpha + 2}}} + \sum_{s=1}^{t-1} \frac{\frac{\alpha}{\alpha + 2} }{\delta(\theta^{s+1})^{\frac{\alpha}{\alpha + 2}}} \cdot \left(\frac{\delta(\theta^{s+1})}{\delta(\theta^s)}\right)^{\frac{\alpha}{\alpha + 2}} \left( 1 - \frac{\delta(\theta^{s+1})}{\delta(\theta^s)}\right)\\
    = & \frac{1}{\delta(\theta^1)^{\frac{\alpha}{\alpha + 2}}} + \sum_{s=1}^{t-1} \frac{\frac{\alpha}{\alpha + 2} }{\delta(\theta^{s})^{2 - \frac{2}{\alpha + 2}}} \cdot (\delta(\theta^{s}) - \delta(\theta^{s+1}))\\
    \geq & \frac{1}{\delta(\theta^1)^{\frac{\alpha}{\alpha + 2}}} + \sum_{s=1}^{t-1} \frac{\frac{\alpha}{\alpha + 2} }{\delta(\theta^{s})^{2 - \frac{2}{\alpha + 2}}} \cdot \frac{c_2^2}{2\beta} (\delta(\theta^t))^{2 - \frac{2}{\alpha + 2}}\\
    = & \frac{1}{\delta(\theta^1)^{\frac{\alpha}{\alpha + 2}}} + \sum_{s=1}^{t-1} \frac{c_2^2\left(\frac{\alpha}{\alpha + 2}\right)}{2\beta}\\
    = & \frac{1}{\delta(\theta^1)^{\frac{\alpha}{\alpha + 2}}} + \frac{c_2^2\left(\frac{\alpha}{\alpha + 2}\right)}{2\beta}\cdot(t-1).
\end{align*}
We can conclude that
\begin{align*}
    f(\theta^t) - f(\theta^*) \leq \left[\frac{1}{\left(f(\theta^1) - f(\theta^*)\right)^{\frac{\alpha}{\alpha + 2}}} + \frac{c_2^2\left(\frac{\alpha}{\alpha + 2}\right)}{2\beta}\cdot(t-1)\right]^{- \frac{\alpha + 2}{\alpha}}\leq C (\eta \cdot t)^{-\frac{\alpha + 2}{\alpha}},
\end{align*}
where $C$ is some universal constant. Combined Lemma~\ref{lem:obj_to_params} with the upper bound of $f(\theta^t) - f(\theta^*)$, we obtain that
$\|\theta^t - \theta^*\| \leq c_{0} (\eta \cdot t)^{-1/\alpha}$ where $c_{0}$ is some universal constant. As a consequence, we reach the upper bound stated in Lemma~\ref{lem:convergence_gd}.

For the tightness, consider the function $f:\mathbb{R}\to \mathbb{R}$, $f(\theta) = \frac{|\theta|^{\alpha + 2}}{\alpha + 2}$, which satisfies Assumptions~\ref{assump:smoothness} and~\ref{assump:nonPL}. Consider the continuous limit of the fixed-step size gradient descent (i.e., the limit $\eta \to 0$) starting from $\theta^0 = 1$, which corresponds to the following ODE:
\begin{align*}
    \frac{d \theta}{dt} = - |\theta|^{\alpha + 1}, \quad \theta(0) = 1.
\end{align*}
The solution of the ODE can be written as:
\begin{align*}
    \theta(t) = (t + 1)^{-1/\alpha}.
\end{align*}
Notice that the $t$ in the solution of ODE is equivalent to $\eta t$ in the gradient descent dynamics, which concludes the proof.
\end{proof}

\section{Beyond homogeneous assumptions}
\label{sec:inhomogenous}
In this Appendix, we provide a brief discussion on the behaviors of the Polyak step size gradient descent iterates when the constants in Assumptions~\ref{assump:smoothness} and~\ref{assump:nonPL} are different. In particular, we consider the following two-dimensional population loss function $f(\theta) = \theta_{1}^2 + \theta_{2}^4$ for all $\theta = (\theta_{1}, \theta_{2}) \in \mathbb{R}^2$. Under this case, the optima is $(0, 0)$, and the updates of the Polyak step size gradient descent algorithm are given by:
\begin{align*}
    \theta_{1}^{t + 1} & = \theta_{1}^{t} - \frac{(\theta_{1}^{t})^3 + \theta_{1}^{t} (\theta_{2}^{t})^4}{2(\theta_{1}^{t})^2 + 8 (\theta_{2}^{t})^6} = \frac{(\theta_1^t)^3 + \theta_1^t (\theta_2^t)^4(8(\theta_2^t)^2 - 1)}{2(\theta_{1}^{t})^2 + 8 (\theta_{2}^{t})^6}, \\
    \theta_{2}^{t + 1} & = \theta_{2}^{t} - \frac{ (\theta_{1}^{t})^2 (\theta_{2}^{t})^3 +  (\theta_{2}^{t})^7}{(\theta_{1}^{t})^2 + 4 (\theta_{2}^{t})^6} = \frac{3(\theta_2^t)^7 + (\theta_1^t)^2 \theta_{2}^{t} (1 - (\theta_2^t)^2)}{(\theta_{1}^{t})^2 + 4 (\theta_{2}^{t})^6}.
\end{align*}
Consider the local convergence in $\mathbb{B}(0, \rho)$ for some sufficiently small radius $\rho$, such that $(\theta_2^t)^2 \ll 1/8$, which corresponds to the approximate update:
\begin{align*}
    \theta_1^{t+1} \approx & ~\theta_1^t \cdot \frac{(\theta_1^{t})^2 - (\theta_2^{t})^4}{2(\theta_{1}^{t})^2 + 8 (\theta_{2}^{t})^6},\\
    \theta_2^{t+1} \approx & ~\theta_2^t \cdot \frac{(\theta_1^t)^2 + 3(\theta_2^t)^6}{(\theta_{1}^{t})^2 + 4 (\theta_{2}^{t})^6}.
\end{align*}
For $\theta_1$, the update is only stable when $\theta_1^t \geq C (\theta_2^t)^2$ where $C$ is some universal constant. However, in this regime, $\theta_2$ can converge slowly, as
\begin{align*}
    \frac{(\theta_1^t)^2 + 3(\theta_2^t)^6}{(\theta_{1}^{t})^2 + 4 (\theta_{2}^{t})^6} = 1 - \mathcal{O}( (\theta_2^t)^2) \to 1\quad (\mathrm{as} \quad \theta_2^t \to 0).
\end{align*}
On the other hand, if we want $\theta_2$ to converge linearly, we need $\theta_1^t = \mathcal{O}((\theta_2^t)^3)$. In this regime, the update of $\theta_1$ can be unstable, as
\begin{align*}
    \frac{(\theta_1^{t})^2 - (\theta_2^{t})^4}{2(\theta_{1}^{t})^2 + 8 (\theta_{2}^{t})^6} \geq C_{1} (\theta_2^t)^{-2}
\end{align*}
where $C_{1}$ is some constant. Hence, it's pretty hard to characterize the behaviour of Polyak step-size gradient descent iterates when the constants in Assumption~\ref{assump:smoothness} and~\ref{assump:nonPL} are different. We leave the understanding of this case as an interesting future direction.
\bibliographystyle{abbrv}
\bibliography{Nhat}
\end{document}